\documentclass[twoside, 11pt]{article}
\usepackage[preprint]{jmlr2e}
\usepackage{blindtext}
\usepackage{multirow}
\usepackage{enumitem}
\usepackage{xcolor}

\usepackage[utf8]{inputenc} 
\usepackage[T1]{fontenc}    
\usepackage{booktabs}       
\usepackage{amsfonts}       
\usepackage{tikz}
\usetikzlibrary{bayesnet}
\usetikzlibrary{fit}
\usetikzlibrary{positioning}
\usepackage{subcaption}
\usepackage{caption}
\usepackage{pgfplots}
\pgfplotsset{compat=1.18}
\usepackage{caption}
\usepackage{subcaption}

\usepackage{url}
\usepackage{amsmath}
\usepackage{algorithmic}
\usepackage{algorithm}
\usepackage{mathrsfs}

\usepackage{lastpage}
\jmlrheading{00}{2025}{1-\pageref{LastPage}}{0/00; Revised 0/00}{0/00}{00-0000}{Nicola Bariletto, Khai Nguyen, and Nhat Ho}

\ShortHeadings{DRO and Economic Decision Theory with Bayesian Advancements}{Bariletto, Nguyen and Ho}
\firstpageno{1}

\begin{document}

\title{Data-Driven DRO and Economic Decision Theory: An Analytical Synthesis With Bayesian Nonparametric Advancements}

\author{\name Nicola Bariletto \email nicola.bariletto@utexas.edu \\
       \addr Department of Statistics and Data Sciences\\
       University Texas at Austin\\
       Austin, TX 78712, USA
       \AND
       \name Khai Nguyen \email khainb@utexas.edu \\
       \addr Department of Statistics and Data Sciences\\
       University Texas at Austin\\
       Austin, TX 78712, USA
       \AND
       \name Nhat Ho \email minhnhat@utexas.edu \\
       \addr Department of Statistics and Data Sciences\\
       University Texas at Austin\\
       Austin, TX 78712, USA}

\editor{My editor}

\maketitle

\begin{abstract}
We develop an analytical synthesis that bridges data-driven Distributionally Robust Optimization (DRO) and Economic Decision Theory under Ambiguity (DTA). By reinterpreting standard regularization and DRO techniques as data-driven counterparts of ambiguity-averse decision models, we provide a unified framework that clarifies their intrinsic connections. Building on this synthesis, we propose a novel DRO approach that leverages a popular DTA model of smooth ambiguity-averse preferences together with tools from Bayesian nonparametric statistics. Our baseline framework employs Dirichlet Process (DP) posteriors, which naturally extend to heterogeneous data sources via Hierarchical Dirichlet Processes (HDPs), and can be further refined to induce outlier robustness through a procedure that selectively filters poorly-fitting observations during training. Theoretical performance guarantees and convergence results, together with extensive simulations and real-data experiments, illustrate the method's favorable performance in terms of prediction accuracy and stability.
\end{abstract}


\section{Introduction}

Optimization lies at the heart of numerous scientific disciplines and practical applications, such as engineering, statistics, machine learning, and economics. This centrality stems from the fact that many key problems in these fields can be reframed as optimizing an objective function encoding domain-specific modeling choices, the goals of the optimization process, and potentially empirical evidence, such as data.

A prominent example is the reformulation of many \textbf{data-driven statistical learning} tasks as the minimization of a risk function:
\begin{equation*}
    \min_{\theta \in \Theta} \mathcal{R}_p(\theta), \quad \textnormal{where } \mathcal{R}_p(\theta) := \mathbb{E}_{\xi \sim p}[h(\theta, \xi)],
\end{equation*}
where $\theta$ denotes a parameter to be optimized over the space $\Theta\subseteq \mathbb R^d$, $\xi$ represents data distributed according to $p$ and belonging to the sample space $\Xi\subseteq \mathbb R^{k}$, and $h$ is an integrable loss function. For instance, a binary classification task can be recast in the above framework as follows: $\xi = (y, x)\in\{-1,1\}\times\mathbb R^{k-1}$ represents a label-feature pair, while the logit loss
\begin{equation*}
    h(\theta, \xi) = \log(1+\exp(-y \cdot r_\theta(x))
\end{equation*}
can be employed to link the label to the features, where $r_\theta:\mathbb R^{k-1}\to \mathbb R$ is a function (e.g., a neural network) parametrized by $\theta$ (e.g., the network weights and biases). As for the data distribution $p$ needed to compute the risk, it would be ideal to equate it to the true data-generating mechanism $p_\star$ and find a theoretical optimizer
\begin{equation*}
    \theta_\star\in\arg\max_{\theta\in\Theta} \mathcal R_{p_\star}(\theta).
\end{equation*}
However, $p_\star$ is rarely known in practice, and instead one has access to a sample $\boldsymbol{\xi}^n = (\xi_1,\dots,\xi_n)\overset{\textnormal{iid}}{\sim} p_\star$. In this situation, the associated \textit{empirical risk minimization} (ERM) problem is often solved in practice:
\begin{equation*}
    \min_{\theta\in\Theta} \mathcal R_{p_{\boldsymbol{\xi}^n}} (\theta), \quad \textnormal{where } p_{\boldsymbol{\xi}^n} := \frac{1}{n}\sum_{i=1}^n \delta_{\xi_i.}
\end{equation*}

In a similar fashion, classical \textbf{economic decision-making} problems under risk are often framed as expected utility maximization:
\begin{equation*}
    \max_{\theta \in \Theta} \mathbb{E}_{\xi \sim p}[u(\theta, \xi)],
\end{equation*}  
where $\theta$ represents an action chosen from a set of available options $\Theta$ (e.g., wealth allocation profiles in a portfolio), $\xi$ models a random outcome in the decision environment (e.g., as determined by fluctuations in financial markets), and $u$ is a utility function defined over actions $\theta \in \Theta$ for each state $\xi \in \Xi$ (e.g., the enjoyment of investment returns from allocation $\theta$ under financial conditions $\xi$).

The aim of this article is to explore connections between (i) data-driven optimization and (ii) economic decision theory, in the context of recent advances within the two fields that developed in parallel and in a largely unrelated fashion: distributionally robust optimization (DRO) and decision theory under ambiguity (DTA). Broadly speaking, DRO departs from the ERM paradigm in the attempt to manage the uncertainty generated by the empirical distribution $p_{\boldsymbol{\xi}^n}$ as a proxy for $p_\star$. To hedge against such uncertainty, DRO often employs a worst-case approach, whereby $p$ is selected as the distribution yielding the highest risk $\mathcal R_p(\theta)$ within some carefully chosen set of probability models $\mathcal Q$. On the other hand, DTA considers the situation in which the decision-maker (DM) does not possess enough information to pick a single model of the environment $p$, but contemplates a whole set $\mathcal Q$ of plausible data-generating processes (a situation termed model uncertainty or \textit{ambiguity}). Different choice axioms, then, lead to different utility criteria that reflect the ambiguity aversion of the DM. 

In light of these analogies, our goals are threefold:
\begin{enumerate}
    \item To provide a heuristic yet comprehensive introduction to the axiomatic foundations of popular DTA models, tailored to the machine learning research community;
    \item To reinterpret a broad class of regularization and DRO methods as data-driven adaptations of DTA models;
    \item To argue that the scope of DRO can be naturally extended by drawing insights from the DTA literature. In particular, we exemplify this thesis with a recently introduced DRO framework that adapts a popular DTA model to the data-driven context via Bayesian nonparametric techniques \citep{bariletto2024bayesian}. While we build upon this paper, previously published at NeurIPS, as our starting point, our aim is to significantly broaden the scope of the original contribution by (i) placing it within a more comprehensive analytical framework that, as described in the previous points, relates DTA to DRO, (ii) advancing the methodology proposed in that work by accounting for more complex data dependence structures, and (iii) proposing a refinement that further robustifies the method in the presence of outliers.
\end{enumerate}

\paragraph{Related Literature.} Before proceeding with our treatment of the subject, we would like to acknowledge that prior works in the DRO literature (see, for instance, discussions in \cite{kuhn2024distributionallyrobustoptimization, ben2010soft, rahimian2022frameworks}) and robust statistics \citep{holmes2016robuststatistics, cerreia2013ambiguity} have noted connections between distributional robustness and and ambiguity aversion in decision-making. However, these connections have primarily been noted in passing, often with a narrow focus on the max-min preferences of \cite{gilboa1989maxmin} (cf. Subsection \ref{sub:ambiguity_averse_preferences} below). In contrast, our aim is to take a more systematic and explicit approach, precisely delineating these connections and demonstrating how they can be leveraged to derive novel DRO criteria. Specifically, by arguing in detail that DRO methods can be viewed as data-driven instantiations of ambiguity-averse decision-theoretic models, we show that these connections naturally lead to DRO frameworks derived from DTA models that have not previously been applied in data-driven settings.
\bigskip

The remainder of the article is organized as follows. In Section~\ref{sec:DTA}, we provide an overview of the axiomatic foundations of DTA and introduce several popular models from the literature. Section~\ref{sec:data_driven_optim} presents a framework for data-driven optimization that emphasizes the correspondence between the DTA models introduced earlier and popular data-driven optimization methods, particularly highlighting the relationship between DRO methods and ambiguity-averse preference models. In Section~\ref{sec:smooth_DRO}, we introduce an additional class of ambiguity-averse preferences that is well-known in the DTA literature, and we propose a method to adapt these preferences to the data-driven DRO context using Bayesian nonparametric techniques. To familiarize the reader with the essential tools from Bayesian nonparametrics, Section~\ref{sec:HDP} reviews the key properties of our primary nonparametric Bayesian tool, the Hierarchical Dirichlet Process (HDP) prior.  Section~\ref{sec:performance_guarantees} presents theoretical performance guarantees for our methodology, while Section~\ref{sec:MC_approx} proposes tractable approximation strategies yielding a surrogate criterion that is efficiently optimized via gradient-based methods, as demonstrated in Section~\ref{sec:gradient_optimization}. Sections~\ref{sec:DP_experiments} and \ref{sec:HDP_experiments} present the results of numerical experiments employing our method, and Section~\ref{sec:DORO} proposes an outlier-robust refinement.  Section~\ref{sec:conclusion} concludes the main text of the article. The supplementary material contains proofs and additional theoretical results (Appendix~\ref{app:theorems}), further background on topics and algorithms used in the paper (Appendix~\ref{app:background}), detailed descriptions of numerical experiments and their results (Appendices~\ref{app:experiments_DP_DRO}, \ref{app:experiments_HDP_DRO}, and \ref{app:experiments_DORO}), and information regarding our computing resources (Appendix~\ref{app:computing_resources}).

\section{The Axiomatic Approach to Decision-Making Under Ambiguity}\label{sec:DTA}

In this section, we provide a heuristic overview of the axiomatic framework in which much of DTA has been developed over the past decades. An (economic) decision-maker (DM) is assumed to operate in the following decision environment \citep{anscombe1963definition}. The DM's uncertainty is modeled by a set of \textit{ambiguous} states of the world $S$, equipped with a sigma-algebra $\mathcal{B}_S$, reflecting the DM's imperfect knowledge about what the true state is. Moreover, there is a convex space of consequences $C$.\footnote{On an interpretive note, in the DTA literature, the convexity of the consequence space $C$ is often understood as representing a space of lotteries over utility-relevant outcomes. While this interpretation could be used to frame our discussion, doing so would introduce an additional layer to the preference-utility representations that is typically absent in the data-driven contexts of primary interest. Therefore, to maintain consistency and simplicity in our exposition, we treat $C$ directly as the space of utility-relevant outcomes or consequences, without interpreting it as a space of probabilities over another underlying outcome space.} The DM selects an act $f: S \to C$ belonging to $\mathcal{F}$, the set of all simple measurable functions from $S$ to $C$. Returning to the portfolio allocation example, the DM can be viewed as choosing among a set of possible allocations, each corresponding to a mapping $f$ from states of the world $s \in S$ (e.g., financial market conditions) to outcomes $c \in C$ (e.g., investment returns). For instance, a portfolio allocation that heavily favors stocks issued by technology companies might correspond to a mapping $f$ that produces high returns $c$ during periods $s$ of tech optimism in the market.

Having established this abstract decision environment, the DM evaluates options based on their preferences, encoded by a binary relation $\succeq$ on $\mathcal{F}$. For all $f, g \in \mathcal{F}$, $f \succeq g$ indicates that the DM prefers act $f$ over act $g$. Additionally, we write $f \succ g$ if $f \succeq g$ but $g \not\succeq f$ (strict preference), and $f \sim g$ if both $f \succeq g$ and $g \succeq f$ (indifference). Moreover, for any $c \in C$, with an abuse of notation we write $c \in \mathcal{F}$ to denote the constant act equal to $c$ for all $s \in S$. By virtue of this embedding of $C$ into $\mathcal F$, the binary relation $\succeq$ can be naturally extended to $C$.

As we will elaborate next, a cornerstone of the decision theory literature has been the derivation of utility representations of the DM's preferences by imposing axioms on $\succeq$. As it will become clear in the sequel, the power of this approach lies in its ability to transform simple, interpretable, and mostly falsifiable preference axioms into a utility criterion, facilitating the direct comparison of alternative acts and enabling optimal choices based on real-number representations of their relative ranking.

\subsection{The Expected Utility Paradigm}\label{sub:expected_utility}

In the above framework, the classical expected utility theory criterion \citep{neumann1947theory, savage1954foundations, anscombe1963definition} can be axiomatically obtained as follows. If the preference relation $\succeq$ on $\mathcal F$ satisfies the following axioms:\footnote{Note that, while the first three axioms have a substantive interpretation, the last two, although plausible, only appear for technical purposes}
\begin{enumerate}[label=(A\arabic*)]
    \item \label{ax:weak_order} $\succeq$ is a \textit{weak order}, i.e., it is complete and transitive;
    \item \label{ax:monotonicity} $\succeq$ is \textit{monotonic}, i.e., for all $f,g\in \mathcal F$, if $f(s)\succeq g(s)$ for all $s\in S$, then $f\succeq g$;
    \item \label{ax:independence} $\succeq$ satisfies \textit{independence}, i.e., for all $f,g,h\in\mathcal F$ and $\delta \in (0,1)$,
    \begin{equation*}
        f \succ g \implies \delta f + (1-\delta) h \succ \delta g + (1-\delta) h;
    \end{equation*}
    \item \label{ax:archimedean} $\succeq$ satisfies the \textit{Archimedean property}, i.e., for all $f,g,h\in\mathcal F$ such that $f\succ g\succ h$, there exist $\delta, \varepsilon\in(0,1)$ such that
    \begin{equation*}
        \delta f + (1-\delta) h \succ g \succ \varepsilon f + (1-\varepsilon) h.
    \end{equation*}
    \item \label{ax:nondegeneracy} $\succeq$ is \textit{non-degenerate}, i.e., $f\succ g$ for some $f,g\in\mathcal F$;
\end{enumerate}
then\footnote{All of the utility representations that we report in this Section were originally formulated in a stronger way as ``if-and-only-if'' characterizations. However, we report them here in a looser form to be able to omit details that are not essential to our treatment.} there exists a non-constant utility function $u : C \to \mathbb R$ and a unique probability measure $Q$ on $\mathcal B_S$ such that, for all $f,g\in\mathcal F$,
\begin{equation*}
    f\succeq g \iff \int_S u(f(s))\, Q(\mathrm ds) \geq \int_S u(g(s))\, Q(\mathrm ds).
\end{equation*}
In other words, the elementary binary preference $\succeq$ can be represented by the utility criterion
\begin{equation*}
    V(f) := \int_S u(f(s))\, Q(\mathrm ds),
\end{equation*}
which the DM maximizes to choose optimally in $\mathcal F$. This representation suggests that, based on their subjective preference $\succeq$, the DM acts as though: (i) they are capable of forming a precise and unique probability assessment $Q$ over the states of the world $s$, and (ii) they make choices by maximizing a utility function $u$ defined over consequences $f(s)$, averaged using the probability assessment $Q$. Therefore, even though the DM faces a condition of ambiguity, they exhibit strong confidence in specifying a complete and unambiguous probabilistic assessment. They resolve all uncertainty by treating this assessment as if it were an objective probability---such as a known and well-specified data-generating process or model of the world---and compute expected utility values accordingly. In this sense, the preference relation can be interpreted as being \textit{ambiguity-neutral}.

\subsection{Ambiguity-Averse Preferences}\label{sub:ambiguity_averse_preferences}

The neutral attitude towards ambiguity displayed by expected utility preferences may not align with common-sense expectations or empirical observations of human preferences.\footnote{See, for example, the celebrated paradox of \cite{ellsberg1961risk}: when faced with a choice between two urns, each containing 100 balls that are either green or red, most people prefer to bet on an urn with a known composition of 50 green and 50 red balls, rather than on an urn with an unknown proportion of green and red balls.} Indeed, it is reasonable to expect a DM to exhibit a negative attitude towards ambiguity, favoring acts that are relatively “less ambiguous” (i.e., those associated with less variable rewards across uncertain states).

The key axiom on $\succeq$ that leads to this undesirable neutrality is the independence axiom (axiom~\ref{ax:independence}). Specifically, mixing an act $h$ with two different acts $f$ and $g$ can produce new acts $f' \equiv \delta f + (1-\delta) h$ and $g' \equiv \delta g + (1-\delta) h$ such that $f'$ implies a higher degree of ambiguity than $g'$, even if $f \succ g$. This could potentially reverse the original preference relation if the DM is ambiguity-averse. Furthermore, ambiguity aversion may prompt the DM to prefer hedging against ambiguity by mixing across uncertain acts.

As a result, an extensive body of literature has explored various relaxations of Axiom~\ref{ax:independence} that lead to ambiguity-averse utility criteria. \cite{cerreia2011uncertainty} synthesized these contributions by identifying that the key modification to Axiom~\ref{ax:independence} is the introduction of the following condition:

\begin{enumerate}[label=(A\arabic*)]
    \setcounter{enumi}{5}
    \item \label{ax:amb_aversion} The preference relation $\succeq$ exhibits \textit{ambiguity aversion}, meaning that for all $f,g \in \mathcal{F}$ and $\delta \in (0,1)$,
    \begin{equation*}
        f \sim g \implies \delta f + (1-\delta) g \succeq f.
    \end{equation*}
\end{enumerate}
Intuitively, Axiom~\ref{ax:amb_aversion} reflects the decision maker's (DM's) preference for diversification across uncertain acts: if they are indifferent between $f$ and $g$, then hedging against uncertainty by mixing them should be weakly preferred. Combined with additional axioms---for instance introducing weaker forms of independence compared to Axiom~\ref{ax:independence}---Axiom~\ref{ax:amb_aversion} underpins several widely used DTA models, which we review below.

\paragraph{Max-Min Expected Utility Preferences.} \cite{gilboa1989maxmin} made one of the first fundamental steps in leveraging Axiom~\ref{ax:amb_aversion} to obtain meaningful utility representations that incorporate ambiguity aversion. Specifically, imposing Axiom~\ref{ax:amb_aversion} and a weaker form of Axiom~\ref{ax:independence}, they showed that there exist a non-constant utility function $u:C \to \mathbb R$ and a convex compact set $\mathcal Q$ of probability measures on $\mathcal B_S$ such that the preference relation is represented by the criterion
\begin{align}\label{eq:gilboaschmeidler_criterion}
    V_1(f) & :=  \min_{Q\in \mathcal Q} \int_S u(f(s))\, Q(\mathrm ds).
\end{align}
Intuitively, the DM deems plausible a whole set $\mathcal Q$ of distributions over states $s$, and chooses the optimal act $f\in\mathcal F$ according to a worst-case calculation: whatever $f$ the DM selects, they make their decision according to the lowest possible expected utility level induced by that act within the ambiguity set $\mathcal Q$. This can also be understood through the metaphor of a game against a malignant Nature, which chooses the state of the world adversarially after the DM has selected an act $f$.

It is also interesting to note that the size of $\mathcal Q$ encodes information both on the level of ambiguity faced by the DM ---the larger $Q$, the less unambiguous information the DM has---as well as on their degree of ambiguity aversion---the larger $\mathcal Q$, the more cautiously the DM chooses among ambiguous acts \citep{ghirardato2004differentiating}.

%

%

\paragraph{Variational Preferences.} A more general class of ambiguity-averse preferences was axiomatized by \cite{maccheroni2006ambiguity} by further weakening Axiom~\ref{ax:independence}. In particular, Axiom~\ref{ax:amb_aversion}, together with this weaker form of independence, leads to a preference relation $\succeq$ that is represented by the following utility criterion:
%
%
\begin{equation*}
    V_2(f) := \min_{Q\in\mathcal P_S} \left\{\int_S u(f(s))\, Q(\mathrm ds) + c(Q) \right\},
\end{equation*}
where $\mathcal P_S$ is the space of probability measures on $\mathcal B_S$ and $c:\mathcal P_S\to\mathbb [0,\infty]$ is a convex and lower-semicontinuous function. Note that, if
\begin{equation*}
    c(Q) = \delta(Q;\mathcal Q) := \begin{cases}
        0 & \textnormal{if } Q\in \mathcal Q,\\
        \infty & \textnormal{if } Q\not\in \mathcal Q,
    \end{cases}
\end{equation*}
we recover the max-min expected utility criterion $V_1$.

These preferences, known as \textit{variational preferences}, can be interpreted as follows: the DM maximizes a worst-case penalized expected utility criterion, where the function $c$ imposes a penalty based on the distribution $Q$ used to compute expectations. A sharp penalization, such as $c = \delta(\cdot; \mathcal{Q})$, leads to the max-mincriterion proposed by \cite{gilboa1989maxmin}. However, smoother penalties are also accommodated, such as $c(Q) = \mathcal{D}(Q, P)$, where $P$ is a reference distribution and $\mathcal{D}$ is a probability metric; see, e.g., \cite{hansen2001robust}.

\section{Data-Driven Risk Minimization}\label{sec:data_driven_optim}

Similar to the DTA framework, data-driven learning typically involves optimizing a numeric criterion. However, rather than maximizing a utility function, one minimizes a risk function $\mathcal{R}_p(\theta) = \mathbb E_{\xi\sim p}[h(\theta, \xi)]$, where $p$ represents an uncertain distribution used to average the loss $h(\theta, \xi)$ over data $\xi\in\Xi$. To continue the parallel with DTA, $p$ can be interpreted as analogous to an uncertain state $s$, while the parameter $\theta$ (the decision variable) defines a mapping that associates distributions with risk values, $p \mapsto \mathcal{R}_p(\theta)$, similar to how an act $f$ induces a mapping from states to utility values, $s\mapsto u(f(s))$. In what follows, we show that, in light of this interpretation, the three DTA models explored in the previous Section correspond to popular classes of data-driven optimization.
\subsection{Regularized ERM as Ambiguity-Neutral Expected Utility Optimization}

Just as the expected utility criterion, described in Subsection~\ref{sub:expected_utility}, averages utility profiles $u(f(s))$ using a unique probability measure $Q$ on $\mathcal B_S$, one can define a probability measure $Q_{\boldsymbol{\xi}^n}$ on $\mathcal B_{\mathcal P_\Xi}$\footnote{$\mathcal P_\Xi$ denotes the set of probability measures on $\mathcal B_\Xi$, endowed with the Borel sigma-algebra $\mathcal B_{\mathcal P_\Xi}$ generated by the topology of weak convergence.} and average $\mathcal R_p(\theta)$ accordingly---we explicitly highlight the dependence of the measure $Q_{\boldsymbol{\xi}^n}$ on the available data sample $\boldsymbol{\xi}^n=(\xi_1, \dots, \xi_n)$, making the optimization problem data-driven. This leads to the following formulation:
\begin{equation*}
    \min_{\theta\in\Theta} \mathbb E_{p\sim Q_{\boldsymbol{\xi}^n}}[\mathcal R_p(\theta)].
\end{equation*}
In this framework, recalling that $p_{\boldsymbol{\xi}^n}$ denotes the empirical distribution based on $\boldsymbol{\xi}^n$, the ERM approach can be interpreted as choosing $Q_{\boldsymbol{\xi}^n}(\mathrm dp) = \delta_{p_{\boldsymbol{\xi}^n}}(\mathrm dp)$, leading to the standard data-driven optimization problem
\begin{equation*}
    \min_{\theta\in\Theta} \mathcal R_{p_{\boldsymbol{\xi}^n}}(\theta), \quad \textnormal{where } \mathcal R_{p_{\boldsymbol{\xi}^n}}(\theta) \equiv \frac{1}{n} \sum_{i=1}^n h(\theta, \xi_i).
\end{equation*}

A more general and insightful choice of $Q_{\boldsymbol{\xi}^n}$ \citep{lyddon2018nonparametric, wang2022distributional, bariletto2024bayesian} is the Dirichlet Process (DP) posterior \citep{ferguson1973bayesian} with a prior centering measure $p_0\in\mathcal P_\Xi$ and a prior concentration parameter $\alpha\in (0,\infty)$. We briefly recall that the DP with parameters $(\alpha, p_0)$ is the distribution of a random probability measure $p$ on $\mathcal B_\Xi$ (denoted as $p\sim\textnormal{DP}(\alpha, p_0)$), satisfying $(p(A_1),\dots,p(A_k))\sim\textnormal{Dirichlet}(\alpha p_0(A_1), \dots, \alpha p_0(A_k))$ for any finite measurable partition $\{A_1,\dots,A_k\}$ of $\Xi$. Furthermore, given the observed sample $\boldsymbol{\xi}^n$, the conditional distribution of $p$ remains a DP with updated parameters $(\alpha + n, \frac{n}{\alpha + n} p_{\boldsymbol{\xi}^n} + \frac{\alpha}{\alpha + n}p_0)$.

\cite{blackwell1973ferguson} provided a fundamental characterization of the DP predictive distribution through the following infinite P\'olya urn scheme:
\begin{equation*}
    \int_{\mathcal P_\Xi} p(\mathrm d\xi)\, Q_{\boldsymbol{\xi}^n}(\mathrm dp) = \frac{n}{\alpha + n} p_{\boldsymbol{\xi}^n}(\mathrm d\xi) + \frac{\alpha}{\alpha + n}p_0(\mathrm d\xi).
\end{equation*}
This shows that, after observing the sample $\boldsymbol{\xi}^n$, the expected value of the random probability measure $p$ is a convex combination of the empirical distribution and the prior centering distribution. Equivalently, the next observation $\xi_{n+1}$ is predicted to match one of the observed data points $\xi_i$ with probability $1/(\alpha + n)$ for $i=1,\dots,n$, or to be a new draw from $p_0$ with probability $\alpha/(\alpha + n)$. Consequently, averaging the risk with respect to $Q_{\boldsymbol{\xi}^n}$ leads to the following criterion:
\begin{equation}\label{eq:ambiguity_neutral_criterion}
    \mathbb E_{p\sim Q_{\boldsymbol{\xi}^n}}[\mathcal R_p(\theta)] = \frac{n}{\alpha + n} \mathcal R_{p_{\boldsymbol{\xi}^n}}(\theta) + \frac{\alpha}{\alpha + n}\mathcal R_{p_0}(\theta).
\end{equation}
This expression reveals that, under this choice of $Q_{\boldsymbol{\xi}^n}$, the expected utility framework results in \textit{regularized ERM}, where a balance is achieved between the empirical risk $\mathcal R_{p_{\boldsymbol{\xi}^n}}(\theta)$ and the prior-induced regularization term $\mathcal R_{p_0}(\theta)$. The extent of regularization is governed by the DP prior concentration parameter $\alpha$ and the sample size $n$: as $n$ increases (respectively, as $\alpha$ increases), greater weight is placed on the empirical risk term (respectively, the regularization term). Notably, the standard ERM criterion is recovered in the limit as $\alpha \to 0$.

As demonstrated in the following Proposition, an interesting special case arises in linear regression, where a specific choice of $p_0$ leads to an optimization procedure equivalent to \emph{Ridge regression} \citep{hoerl1970ridge}. While simple, this result is remarkable as it provides an alternative Bayesian interpretation of Ridge regression, distinct from the traditional one: instead of placing a parametric normal prior on the Gaussian regression coefficients $\theta$ and finding the posterior mode, Ridge regression is obtained by placing a nonparametric DP prior on the data-generating process within a risk minimization framework, and averaging the risk according to the corresponding posterior.

\begin{proposition}[\cite{bariletto2024bayesian}]\label{pro:equivalence_regularization}
    Assume each data point $\xi=(y,x)\in\mathbb R^{d+1}$ consists of a response $y\in\mathbb R$ and a vector of inputs $x\in\mathbb R^d$, with $h(\theta, \xi) = (y-x^\top \theta)^2$ for $\theta\in\mathbb R^d$. Choosing $Q_{\boldsymbol{\xi}^n}$ as a DP with concentration parameter $\alpha + n$ and centering distribution $\frac{n}{\alpha + n} p_{\boldsymbol{\xi}^n} + \frac{\alpha}{\alpha + n} \mathcal N(0, I_{d+1})$, Equation~\eqref{eq:ambiguity_neutral_criterion} simplifies to Ridge regression:
    \begin{equation*}
        \min_{\theta\in\mathbb R^d} \left\{\frac{1}{n}\sum_{i=1}^n (y_i - x_i^\top \theta)^2 +\frac{\alpha}{n}\Vert \theta\Vert_2^2\right\}.
    \end{equation*}
\end{proposition}

\subsection{DRO as Ambiguity-Averse Utility Optimization}

Just as adapting the ambiguity-neutral expected utility paradigm to a data-driven setting leads to regularized ERM, we are now ready to demonstrate that the ambiguity-averse models naturally corresponds to widely used DRO methods. Intuitively, the DM's uncertainty about the true model of the world parallels the empirical researcher's uncertainty in selecting a single data-driven probability distribution as a proxy for the data-generating process when computing the risk. While expected utility maximization and regularized ERM address this uncertainty by averaging over utility or risk profiles, ambiguity-averse criteria and DRO frameworks acknowledge this uncertainty explicitly and incorporate aversion to it within the optimization process. In what follows, we illustrate how popular DRO frameworks can be interpreted as data-driven instantiations of the ambiguity-averse decision models discussed in Subsection~\ref{sub:ambiguity_averse_preferences}.

\paragraph{DRO Based On Ambiguity Sets.} Analogous to the max-min utility criterion $V_1$ of \cite{gilboa1989maxmin} presented in Equation~\eqref{eq:gilboaschmeidler_criterion}, the most common formulation of DRO seeks to solve the data-driven optimization problem  
\begin{equation}\label{eq:minmax_dro_criterion}
    \min_{\theta\in\Theta} \max_{p\in\mathcal P(\boldsymbol{\xi}^n)} \mathcal R_p(\theta)
\end{equation}  
for a data-driven ambiguity set $\mathcal P(\boldsymbol{\xi}^n) \subseteq \mathcal P_{\Xi}$.  

Two clarifications are in order. First, since we now face a minimization problem, the $\min$ operator in Equation~\eqref{eq:gilboaschmeidler_criterion} is replaced with a $\max$. Second, to maintain consistency with the previous analogies between DTA and DRO---where we interpreted the function $u$ in the utility criterion as the risk function $\mathcal R$ in the data-driven setting---the problem should be written as  
\begin{equation*}
    \min_{\theta\in\Theta} \max_{Q\in\mathcal Q(\boldsymbol{\xi}^n)} \mathbb{E}_{p\sim Q}[\mathcal{R}_p(\theta)]
\end{equation*}  
for some set $\mathcal Q(\boldsymbol{\xi}^n)$ of probability measures on $\mathcal B_{\mathcal P_{\Xi}}$. However, for any $Q\in\mathcal Q(\boldsymbol{\xi}^n)$, denoting  
\begin{equation*}
    p_Q(\mathrm{d}\xi) := \int_{\mathcal P_{\Xi}} p(\mathrm{d}\xi) \,Q(\mathrm{d}p),
\end{equation*}  
we can always rewrite  
\begin{equation*}
    \mathbb{E}_{p\sim Q}[\mathcal{R}_p(\theta)] = \mathbb{E}_{p\sim Q}[\mathbb{E}_{\xi\sim p}[h(\theta,\xi)]] \equiv \mathbb{E}_{\xi\sim p_Q}[h(\theta,\xi)] = \mathcal{R}_{p_Q}(\theta).
\end{equation*}  
Thus, we opt for the formulation in Equation~\eqref{eq:minmax_dro_criterion} as it leads to leaner and somewhat more general notation.\footnote{Indeed, Equation~\eqref{eq:minmax_dro_criterion} does not restrict the set $\mathcal P(\boldsymbol{\xi}^n)$ to be of the form $\left\{ \int_{\mathcal P_\Xi} q(\mathrm{d}\xi) Q(\mathrm{d}q) : Q\in\mathcal Q(\boldsymbol{\xi}^n) \right\}$.}

A significant portion of the DRO literature has been dedicated to defining suitable versions of the ambiguity set $\mathcal P(\boldsymbol{\xi}^n)$ and analyzing the properties of the corresponding optimization procedures. For example, $\mathcal P(\boldsymbol{\xi}^n)$ can be constructed to include all distributions $p$ that satisfy certain moment conditions \citep{scarf1958minmax, cai2023dro, gallego1993newsboy, elghaoui2003worstcase, natarajan2010robust, li2018closedform, yue2006newsvendor}, or all distributions within an $\varepsilon$-distance from a reference (e.g., the empirical) distribution $\hat p$, with various choices for the distance metric $\mathcal D(p, \hat p)$ \citep{ben-tal2013robust, wang2016likelihood, jiang2018risk, duchi2021learning, pflug2007ambiguity, pflug2012investment, mohajerin2018data, lasserre2021distributionally, staib2019kernel}. For additional references on the theory and applications of these methods, we direct the reader to the comprehensive reviews by \cite{rahimian2022frameworks} and \cite{kuhn2024distributionallyrobustoptimization}.

\paragraph{DRO Based On Soft Constraints.} Although less common, data-driven counterparts of the variational preferences discussed in Subection \ref{sub:ambiguity_averse_preferences} have been explored in the DRO literature. Specifically, if $\mathcal D$ denotes a divergence function on $\mathcal P_{\Xi}$, one may consider solving the problem  
\begin{equation}\label{eq:DRO_soft}
    \min_{\theta \in \Theta} \max_{p \in \mathcal P_{\Xi}} \left\{\mathcal R_p(\theta) - \lambda \mathcal D(p, \hat{p})\right\},
\end{equation}
where $\hat{p}$ is a reference distribution informed by the data (e.g., $\hat{p} = p_{\boldsymbol{\xi}^n}$) and $\lambda>0$. This formulation can be viewed as a softer alternative to pursuing the more commonly used ambiguity-set strategy  
\begin{equation*}
    \min_{\theta \in \Theta} \max_{p \in \mathcal P(\boldsymbol{\xi}^n)} \mathcal R_{p}(\theta), \quad \mathcal P(\boldsymbol{\xi}^n) = \{p \in \mathcal P_{\Xi} : \mathcal D(p, \hat{p}) \leq \varepsilon\},
\end{equation*}
for some $\varepsilon > 0$. Indeed, instead of equally weighting all distributions within an $\varepsilon$-neighborhood of $\hat{p}$ and completely ignoring distributions outside this neighborhood, the criterion in Equation~\eqref{eq:DRO_soft}---akin to variational preferences---allows to penalize distributions $p$ based on their distance from $\hat{p}$, with such penalty varying smoothly with $\mathcal D(p, \hat{p})$ and being moderated by the hyperparameter $\lambda$; see for instance \cite{zhang2024short, ben2010soft}.

\section{Smooth Ambiguity-Averse Preferences Meet Bayesian Nonparametrics}\label{sec:smooth_DRO}

\cite{cerreia2011uncertainty} demonstrate that another class of ambiguity-averse preferences are the smooth preferences of \cite{klibanoff2005smooth}. According to this model, similarly to min-max and variational preferences, the DM considers a set $\mathcal{Q}\subseteq \mathcal P_S$ of possible models of the world. However, instead of selecting the optimal act $f \in \mathcal{F}$ based on a highly pessimistic criterion that corresponds to the expected utility computed according to an adversarially chosen $Q\in\mathcal Q$, the DM (i) forms a probability assessment $\Pi$ over models in $\mathcal{Q}$ and (ii) optimizes according to the criterion
\begin{equation}\label{eq:smooth_preferences_utility}
    V_3(f) = \int_{\mathcal{Q}} \psi\left(\int_S u(f(s))\, Q(\mathrm{d}s)\right) \Pi(\mathrm{d}Q),
\end{equation}
where $\psi: \mathbb{R} \to \mathbb{R}$ is a concave and strictly increasing function. The intuition behind why this criterion induces ambiguity aversion is as follows: the expected utilities in the set
\begin{equation*}
    \mathcal{U}_f := \left\{\int_S u(f(s))\, Q(\mathrm{d}s) : Q \in \mathcal{Q}\right\}
\end{equation*}
are averaged according to $\Pi$ after being transformed through a concave function $\psi$, which induces ambiguity aversion in the same way that a concave utility function $u$ induces risk aversion (that is, interpreting $f(s)$ as a known lottery, a concave $u$ encodes aversion to the objective uncertainty, or risk, associated to $f(s)$; see, e.g., \cite{mas-colell1996microeconomictheory}). The more ``spread out'' the values in $\mathcal{U}_f$---that is, the more ambiguous the act $f$ is---the less likely the DM is to choose act $f$ relative to alternatives in $\mathcal{F}$.

This mechanism can be effectively illustrated with the following simple example: assume the DM chooses between two acts $\mathcal{F} = \{f_1, f_2\}$ and considers two models $\mathcal{Q} = \{Q_1, Q_2\}$ by assigning them equal probability $\Pi = 0.5 \delta_{Q_1} + 0.5 \delta_{Q_2}$. Furthermore, suppose the expected utility values $U_{Q_j}(f_i) := \int_S u(f_i(s))\, Q_j(\mathrm{d}s)$, for $i, j \in \{1, 2\}$, satisfy
\begin{equation*}
    U_\star = \int_{\mathcal{Q}} U_Q(f_1)\, \Pi(\mathrm{d}Q) = \int_{\mathcal{Q}} U_Q(f_2)\, \Pi(\mathrm{d}Q),
\end{equation*}
as depicted on the horizontal axis of Figure \ref{fig:psi_graph} (note that, for any measurable function $G:\mathcal Q\to \mathbb R$, $\int_{\mathcal{Q}} G(Q)\, \Pi(\mathrm{d}Q)$ is simply a formal expression for the average of $G(Q_1)$ and $G(Q_2)$, given the assumed form of $\Pi$). Clearly, $f_1$ and $f_2$ yield the same average reward $U_\star$ under $\Pi$, although the rewards $U_{Q_1}(f_2)$ and $U_{Q_2}(f_2)$ are less dispersed compared to those associated with $f_1$. Consequently, transforming these $Q$-dependent rewards $U_{Q}(f_i)$ through the strictly concave function $\psi$ before averaging them with respect to $\Pi$, results in a strict preference for $f_2$ over $f_1$, because
\begin{align*}
    V_3(f_2) = \int_{\mathcal{Q}} \psi\left( U_Q(f_2) \right) \Pi(\mathrm{d}Q) \,>\, \int_{\mathcal{Q}} \psi\left( U_Q(f_1) \right) \Pi(\mathrm{d}Q) = V_3(f_1),
\end{align*}
as it is apparent in Figure \ref{fig:psi_graph}. This highlights the role of the concave non-linear transformation $\psi$ in inducing ambiguity aversion.

\begin{figure}
    \centering
    \begin{tikzpicture}
      \pgfmathsetmacro{\psiZero}{ln(0+1)/5 + 0.05}
      \pgfmathsetmacro{\psiOne}{ln(1+1)/5 + 0.05}
      \pgfmathsetmacro{\psiThree}{ln(3+1)/5 + 0.05}
      \pgfmathsetmacro{\psiFour}{ln(4+1)/5 + 0.05}
      \pgfmathsetmacro{\yLineA}{(\psiOne + \psiThree)/2}  
      \pgfmathsetmacro{\yLineB}{(\psiZero + \psiFour)/2}   
    
      \begin{axis}[
        width=8cm,
        height=7cm,
        xmin=0, xmax=4.5,
        ymin=0.05, ymax=0.4,
        axis x line=bottom,
        axis y line=left,
        xmajorgrids=false,  
        ymajorgrids=false,  
        xtick={0,1,2,3,4},
        xticklabels={
          $U_{Q_1}(f_1)$,
          $U_{Q_1}(f_2)$,
          $U_\star$,
          $U_{Q_2}(f_2)$,
          $U_{Q_2}(f_1)$
        },
        ytick={\yLineB,\yLineA},
        yticklabels={$V_3(f_1)$, $V_3(f_2)$},
        tick label style={font=\small},
        xticklabel style={align=center, font=\small},
        every axis plot/.append style={black}
      ]
    
      \addplot[brown, thick, domain=0:4.5, samples=200] {ln(x+1)/5 + 0.05}
        node[pos=0.96, below, text=brown, font=\large\sffamily] {$\psi$};
    
      \addplot[only marks, mark=*, mark size=1.2pt, black] coordinates {
        (0, \psiZero)
        (1, \psiOne)
        (3, \psiThree)
        (4, \psiFour)
      };
    
      \addplot[black, line width=0.3pt] coordinates {(0, \psiZero) (4, \psiFour)};
      \addplot[black, line width=0.3pt] coordinates {(1, \psiOne) (3, \psiThree)};
      
      \addplot[black, dashed, line width=0.3pt] coordinates {(0, 0.05) (0, \psiZero)};
      \addplot[black, dashed, line width=0.3pt] coordinates {(1, 0.05) (1, \psiOne)};
      \addplot[black, dashed, line width=0.3pt] coordinates {(3, 0.05) (3, \psiThree)};
      \addplot[black, dashed, line width=0.3pt] coordinates {(4, 0.05) (4, \psiFour)};
      
      \addplot[black, dashed, line width=0.3pt] coordinates {(0, \yLineA) (2, \yLineA)};
      \addplot[black, dashed, line width=0.3pt] coordinates {(0, \yLineB) (2, \yLineB)};
      
      \addplot[only marks, mark=*, mark size=1.2pt, black] coordinates {(2, \yLineA)};
      \addplot[only marks, mark=*, mark size=1.2pt, black] coordinates {(2, \yLineB)};
    
      \end{axis}
    \end{tikzpicture}
    \caption{An illustration of smooth ambiguity-averse preferences. Although $f_1$ and $f_2$ yield the same expected utility $U_\star$ across the ambiguous distributions $Q_1$ and $Q_2$, the ambiguity-averse criterion $V_3(f)$ implies a strict preference for the act $f_2$, which delivers more stable utility levels across these distributions.}
    \label{fig:psi_graph}
\end{figure}

\begin{remark}
    \cite{cerreia2011uncertainty} showed that if one chooses $\psi(t)=-\exp(- t/\beta)$ (with $\beta>0$ and under additional technical assumptions), then maximizing $V_3$ is equivalent to solving
    \begin{equation*}
        \max_{f\in\mathcal{F}}\min_{\tilde \Pi\,:\,\tilde \Pi\ll \Pi} \left\{\int_{\mathcal{Q}}  U_Q(f)\,  \tilde\Pi(\mathrm{d}Q) \, + \, \beta\cdot \textnormal{KL}(\tilde \Pi\Vert \Pi)\right\},
    \end{equation*}
    where $\textnormal{KL}(\cdot\Vert\cdot)$ denotes the Kullback-Leibler divergence, and $\ll$ represents absolute continuity. This result further clarifies the mechanism through which ambiguity aversion is induced by $\psi$: intuitively, rather than directly averaging over $Q\sim \Pi$, one considers a worst-case scenario with respect to the mixing measure, penalizing distributions that deviate further from the reference probability measure $\Pi$.
    
    Moreover, in the limiting case $\beta \to 0$, the max-min ambiguity-averse setup (corresponding to the ambiguity set DRO approach) is recovered, with the ambiguity set
    \begin{equation*}
        \mathcal{Q} = \left\{Q\in \mathcal{P}_{S} : \exists \tilde\Pi\ll \Pi,\, Q = \int_{\mathcal{P}_S}P\, \tilde\Pi(\mathrm dP)\right\}.
    \end{equation*}
    In the other limiting case $\beta \to \infty$ (with the convention $0\cdot\infty=0$), the ambiguity-neutral criterion (corresponding to possibly regularized ERM) is recovered. Thus, smooth ambiguity-averse preferences can be seen as a less extreme version of max-min preferences. Consequently, the data-driven DRO criterion we are about to propose can be interpreted as a smoother variant of the well-established DRO methods based on ambiguity sets.

\end{remark}

In analogy with the data-driven DRO counterparts of the max-min and variational utility criteria discussed earlier, we aim to explore data-driven adaptations of smooth ambiguity-averse preferences. Because these preferences are ambiguity-averse, their adaptations can be interpreted as a form of DRO \citep{bariletto2024bayesian}. Specifically, by recalling the idea of averaging the risk $\mathcal{R}_p(\theta)$ (analogous to $\int_S u(f(s)) Q(\mathrm{d}s)$ in the DTA framework) with respect to the DP posterior $Q_{\boldsymbol{\xi}^n}$ (corresponding to $\Pi$), the utility criterion $V_3$ in Equation~\eqref{eq:smooth_preferences_utility} naturally extends to the data-driven setting as
\begin{equation}\label{eq:our_criterion}
    V_{\boldsymbol{\xi}^n}(\theta) := \int_{\mathcal P_\Xi} \phi\left(\mathcal R_p(\theta)\right) Q_{\boldsymbol{\xi}^n}(\mathrm dp),
\end{equation}
where $\phi$ is a strictly increasing and convex function (convexity replaces concavity because risk minimization takes the place of utility maximization). \cite{bariletto2024bayesian} study the properties of $V_{\boldsymbol{\xi}^n}$ under the assumption that 
\[
\boldsymbol{\xi}^n = (\xi_1, \dots,\xi_n)\overset{\textnormal{iid}}{\sim} p_\star \quad \text{for some unknown } p_\star\in\mathcal P_\Xi.
\]
Specifically, they establish convergence properties of the excess risk associated with 
\[
\theta_n\in\arg\max_{\theta\in\Theta}V_{\boldsymbol{\xi}^n}(\theta),
\]
as well as tractable approximations of the criterion and gradient-based optimization procedures for practical implementation.

Our goal here is to extend the approach of \cite{bariletto2024bayesian} to a more general modeling framework that accommodates a richer dependence structure among observations. In particular, we study the properties of a criterion analogous to that of Equation~\eqref{eq:our_criterion} in the presence of $N=N_1+\dots+N_S$ observations naturally grouped into $S$ distinct samples,
\[
\boldsymbol{\xi}^N = (\boldsymbol{\xi}_1,\dots,\boldsymbol{\xi}_S),
\]
where $\boldsymbol{\xi}_s = (\xi_{s1},\dots,\xi_{s N_s})$ for $s=1,\dots, S$. One can think of $\boldsymbol{\xi}^N$ as data collected from $S$ different sources---for instance, baseline health conditions measured on patients from distinct groups, such as age brackets, with the aim of assessing the impact of these conditions on the incidence of a particular disease within each group.

In this setting, there are two naively extreme approaches to modeling the data. First, one may assumet that the $S$ sources are completely homogeneous in distribution; in this case, one would treat $\boldsymbol{\xi}^N$ as a single large sample and optimize $V_{\boldsymbol{\xi}^N}$ as defined in Equation~\eqref{eq:our_criterion}. The alternative approach assumes that the true data-generating processes $p_1^\star, \dots, p_S^\star$ are entirely heterogeneous, leading one to minimize $V_{\boldsymbol{\xi}^{N_1}},\dots,V_{\boldsymbol{\xi}^{N_S}}$ separately. Although both options are viable in practice, they either (a) completely disregard potential heterogeneity among source-specific distributions that should influence the final optimization output, or (b) ignore any possible dependence among the different sources that might be leveraged to improve the optimization outcomes.

The Bayesian nonparametric framework of \cite{bariletto2024bayesian} offers a more principled and less extreme solution. Rather than modeling the vector of distributions as either completely homogeneous or as having (unconditionally) independent components, we consider a more general joint Beysian model specified as follows: for all $s=1,\dots,S$ and $j=1,\dots, N_s$,
\begin{align}\label{eq:partial_exch_model}
    \xi_{sj} \mid (p_1, \dots, p_S) & \stackrel{\text{ind}}{\sim} p_s,\nonumber \\
    (p_1, \dots, p_S) & \,\sim Q.
\end{align}
This model implies that while observations are exchangeable (i.e., distributionally homogeneous) within groups, they may follow different yet dependent laws across groups. Notice that the two extreme cases---identity in law and unconditional independence---are both accommodated by this construction, either by concentrating the mass of $Q$ on the “diagonal” $p_1=\dots=p_S$ or by specifying it as a product distribution (i.e., forcing the $p_s$'s to be unconditionally independent). Given this modeling strategy, it is natural to replace the DRO criterion in Equation~\eqref{eq:our_criterion} with, for each group $s = 1,\dots,S$,
\begin{equation}\label{eq:HDP_criterion}
    V_{\boldsymbol{\xi}^N}^s(\theta) := \int_{\mathcal P_\Xi}\phi\left(\mathcal R_p(\theta)\right) Q_{\boldsymbol{\xi}^N}^s(\mathrm dp),
\end{equation}
where $Q_{\boldsymbol{\xi}^N}^s$ denotes the marginal distribution of $p_s$ conditional on $\boldsymbol{\xi}^N$.

Clearly, if the prior $Q$ in Equation~\eqref{eq:partial_exch_model} is constructed to induce dependence among the coordinates of $(p_1, \dots, p_S)$, then the DRO criterion in Equation~\eqref{eq:HDP_criterion} will yield an optimization output that is both tailored to group $s$ and informed by the entire sample $\boldsymbol{\xi}^N$. The key question, then, is how to design $Q$ so that it induces the desired dependence while preserving tractability.

\section{The Hierarchical Dirichlet Process}\label{sec:HDP}

There has been growing interest in constructing priors for dependent distributions such as $(p_1, \dots, p_S)$, particularly within the realm of Bayesian nonparametrics. Numerous innovative methods have been proposed to build dependent, large-support priors that remain tractable for posterior inference and prediction \citep{maceachern2000dependent, muller2004method, rodriguez2008nested, lijoi2014bayesian, lijoi2014dependent, camerlenghi2019latent}. In light of our goal to address data-driven optimization tasks under minimal distributional assumptions, we draw inspiration from this literature and focus on a seminal model in the field: the hierarchical Dirichlet process (HDP) prior of \cite{teh2004sharing, teh2006hierarchical}. This model is specified hierarchically as follows: for $s=1,\dots,S$, 
\begin{align*}
    p_s\mid p_0 & \overset{\text{iid}}{\sim} \textnormal{DP}(\alpha_s, p_0), \\
    p_0 & \sim \textnormal{DP}(\alpha_0, H),
\end{align*}
where $H$ is a continuous distribution on $(\Xi, \mathscr B(\Xi))$. Intuitively, the HDP prior treats the dependent distributions $p_1, \dots, p_S$ as conditionally independent DPs that share a common centering distribution $p_0$, which is itself assigned a DP prior. This shared base measure $p_0$ is what induces dependence among the $p_s$'s.

Additional insight into the dependence induced by HDP priors is provided by their popular \emph{Chinese restaurant franchise} (CRF) representation \citep{teh2006hierarchical}. In this metaphor, a franchise comprising $S$ restaurants offers an infinite menu of dishes, which are shared across restaurants and generated according to the centering measure $H$. Each restaurant features an infinite number of tables, with each table serving a single dish and potentially accommodating an unlimited number of customers. At the top level, the franchise-wide DP governs the assignment of customers (observations) to dishes (across-group clusters), while at the restaurant level, the individual DPs control the assignment of dishes to tables (within-sample latent clusters). This two-stage clustering mechanism facilitates information sharing among different groups (restaurants) via common atoms (dishes).

Before examining the properties of the HDP-based DRO criterion in Equation~\eqref{eq:HDP_criterion}, we provide further intuition on the role of the HDP in data-driven optimization. Specifically, consider averaging the risk $\mathcal R_p(\theta)$ with respect to $Q_{\boldsymbol{\xi}^N}^s$ directly, which is analogous to the ambiguity-neutral criterion in Equation~\eqref{eq:ambiguity_neutral_criterion} (with the univariate DP posterior replaced by the HDP marginal posterior). In this case, using the CRF construction (see Appendix 2 in the Supplementary Material for a detailed derivation), one deduces that the ambiguity-neutral criterion $\int_{\mathcal P_\Xi} \mathcal R_p(\theta) \,Q_{\boldsymbol{\xi}^N}^s(\mathrm dp)$ reduces to
\begin{align}\label{eq:ambiguity_neutral_criterion_HDP}
    \frac{N_s}{\alpha_s + N_s} & \underbrace{\frac{1}{N_s}\sum_{j=1}^{N_s}h(\theta,\xi_{sj})}_{(\text a)} + \nonumber\\
    & \frac{\alpha_s}{\alpha_s + N_s} \Bigg[ \frac{N}{\alpha_0 + N}\underbrace{\frac{1}{N}\sum_{\ell=1}^S\sum_{j=1}^{N_\ell}h(\theta, \xi_{\ell j})}_{(\text b)} +  \frac{\alpha_0}{\alpha_0 + N} \underbrace{\mathbb E_{\xi\sim H}[h(\theta,\xi)]}_{(\text c)}\Bigg].
\end{align}
This expression can be interpreted as follows. For each sample $s$, the optimization balances the within-group empirical risk (term (a)) against an overall average risk component. The weighting of this balance is determined by the sample size $N_s$ and the sample-specific concentration parameter $\alpha_s$; when $N_s$ is large compared to $\alpha_s$, the within-sample empirical risk dominates. Moreover, the overall average risk itself is a compromise between the across-group empirical risk (term (b)) and the expected risk with respect to the prior centering distribution $H$ (term (c)). Here, the overall sample size $N$ and the top-level concentration parameter $\alpha_0$ control this trade-off. Note that the completely dependent (pooled optimization) scenario is recovered by letting $\alpha_s\to\infty$, while the unconditionally independent (separate optimization) case corresponds to taking the limit $\alpha_0 \to \infty$. Moreover, the single-source DP criterion can be thought of as a special case with $S=1$ (so that $n\equiv N_1$ and $\alpha \equiv \alpha_s$) and $\alpha_0\to\infty$.\footnote{The fact that the DP-based criterion is a special case of the HDP one applies to much of our subsequent analysis of the ambiguity-averse criterion $V^s_{\boldsymbol{\xi}^N}$; hence, unless otherwise specified, results for the HDP also apply to the univariate DP criterion when the parameters are set accordingly.}

Similarly to Proposition~\ref{pro:equivalence_regularization}, the following Example illustrates the ambiguity-neutral HDP criterion in the context of linear regression.

\begin{example}\label{ex:ridge}
    As in Proposition~\ref{pro:equivalence_regularization}, assume that each data point $\xi = (y, \boldsymbol{x}^\top)^\top\in\mathbb R^{d+1}$ consists of the measurement of $d$ features $\boldsymbol{x}^\top = (x_1, \dots, x_d)$ and a response $y$. If interested in linear prediction, one can choose $h(\theta, \xi) = (y - \boldsymbol{x}^\top\theta)^2$ for $\theta\in\mathbb R^d$, i.e., the classic quadratic loss. Moreover, assuming $H=\mathcal N(0, I_{d+1})$, the ambiguity neutral criterion in the previous equation is equivalent to
    \begin{align*}
        \sum_{j=1}^{N_s} (y_{sj} - \boldsymbol{x}_{sj}^\top\theta^s)^2 \, +\, \underbrace{\frac{\alpha_s}{\alpha_0 + N}\, \sum_{\ell=1}^S\sum_{j=1}^{N_\ell}(y_{\ell j}-\boldsymbol{x}_{\ell j}^\top\theta^s)^2}_{\textnormal{Borrowing strength}} \,+\, \underbrace{\frac{\alpha_s\alpha_0}{\alpha_0 + N} \,\Vert \theta^s\Vert_2^2}_{\textnormal{Ridge penalty}}.
    \end{align*}
    In other words, the HDP model with this specific choice of prior centering measure $H$ yields $\boldsymbol{L^2}$\textbf{-regularized least squares with borrowing of information}  \citep{bariletto2024bayesian}: Estimation of $\theta^s$ is guided both by information borrowed from the whole pooled sample $\boldsymbol{\xi}^N$ and by the $L^2$ penalization induced by the prior.
\end{example}

To conclude, the Bayesian approach considered here, when instantiated with the HDP marginal posterior, yields a criterion that both regularizes the empirical risk (through a component computed with respect to the prior centering measure) and facilitates borrowing strength across heterogeneous sources (via an overall empirical risk component). Moreover, as argued at the beginning of this Section, incorporating the nonlinear transformation prior to averaging with respect to the HDP posterior induces distributional robustness. In summary, Figure~\ref{fig:graphical_model} illustrates the key components of the proposed HDP-based DRO criterion $V_{\boldsymbol{\xi}^N}^s$ and highlights the hyperparameters that govern the relative influence of each component.

\begin{figure}
    \centering
    \begin{tikzpicture}[scale=0.8,every node/.style={transform shape, draw,rectangle,rounded corners=0.1cm, inner sep=5pt}]
    \node [fill=gray!10, inner sep=7pt] (distributional) at (0,2) {Distributional Robustness};
    \node [fill=gray!50, inner sep=9pt, below=1.1cm of distributional] (top) {\begin{minipage}{4.5cm}  
    \centering {\Large HDP Criterion $V_{\boldsymbol{\xi}^N}^s$}
    \end{minipage}};
    \node [fill=gray!10, inner sep=7pt, below left=1.8cm and -0.6cm of top] (bottomleft) {{\color{black} Group 1 Local Risk}};
    \node [fill=gray!10, inner sep=7pt, below right=1.8cm and -0.6cm of top] (bottomright) {{\color{black} Group $S$ Local Risk}};
    \node [fill=gray!10, inner sep=7pt, left=-10.9cm of top] (left) {{\color{black} Group $s$ Vanilla Risk}};
    \node [fill=gray!10, inner sep=7pt, right = -11.3cm of top] (topright) {{\color{black} Prior-Regularized Risk}};
    \node[draw = none, below =2.1cm of top](bottom){\Large$\displaystyle\cdots$};
    
    \node [draw, inner sep=18pt, fit=(bottomleft) (bottomright) (bottom), label={[anchor=south,above=-0.7cm]above:\large Borrowing Strength}] (plate) {};
    
    \draw[->] (left) -- (top) node[midway, above, draw=none] {\large$\displaystyle N_s$};
    \draw[->] (topright) -- (top) node[midway, above, draw=none] {\large$\alpha_s, \displaystyle\alpha_0$};
    \draw[->] (plate) -- (top) node[midway, right, draw=none] {\large$\alpha_s, \displaystyle N$};
    
    \draw[->] (distributional) -- (top) node[midway, left, draw=none] {\large$\phi$};
    \end{tikzpicture}
    \vspace{.01in}
    \caption{Key components of the HDP-based DRO method. The criterion for group $s$ is shaped by: (i) the empirical risk within group $s$ (right box); (ii) a borrowing strength component from the risk across groups (bottom box); (iii) a regularization component from the prior centering measure (left box); and (iv) a distributionally robust component governed by the curvature of $\phi$. Hyperparameters near the arrows adjust each component's influence.}
    \label{fig:graphical_model}
\end{figure}

\section{Performance Guarantees}\label{sec:performance_guarantees}

In this Section, we provide statistical guarantees on the performance of our robust optimization method. For each group $s=1,\dots,S$, denote by $\theta^s_N\in\arg\min_{\theta\in\Theta}V^s_{\boldsymbol{\xi}^N}$ a minimizer of the HDP criterion and by $\theta_\star^s \in\arg\min_{\theta\in\Theta}\mathcal R_{p_s^\star}(\theta)$ our parameter of interest at the population level. In this setting, a natural measure of performance is the narrowness of the gap between $\mathcal R_{p_s^\star}(\theta^s_N)$ and $\mathcal R_{p_s^\star}(\theta^s_\star)$, and Lemma \ref{lem:finite_sample_bounds} is a first step towards establishing this type of guarantee.
\begin{lemma}\label{lem:finite_sample_bounds}
    Assume $h$ has range $[0,K]$ for some $K<\infty$. Moreover, let $\phi$ be twice continuously differentiable on $(0,K)$, with $F_\phi:= \sup_{t\in(0,K)}\phi'(t)$ and $S_\phi:=\sup_{t\in(0,K)}\phi''(t)$. Then
    \begin{align*}
        \sup_{\theta \in \Theta} \vert V_{\boldsymbol \xi^N}^s(\theta)  - \phi(\mathcal R_{p_s^\star}(\theta))\vert 
                    & \leq \frac{N_s}{\alpha_s + N_s}F_\phi \sup_{\theta \in \Theta}\vert\mathcal R_{p_{\boldsymbol \xi_s}}(\theta) - \mathcal R_{p_s^\star}(\theta)\vert \\
                    & + \frac{\alpha_s}{\alpha_s + N_s}F_\phi\sup_{\theta \in \Theta}\bigg\vert\frac{N}{\alpha_0 + N}\mathcal R_{p_{\boldsymbol \xi^N}}(\theta)  + \frac{\alpha_0}{\alpha_0 + N}\mathcal R_H(\theta) - \mathcal R_{p_s^\star}(\theta)\bigg\vert \\
                    & + \frac{K^2}{2}S_\phi.
    \end{align*}
\end{lemma}

Lemma \ref{lem:finite_sample_bounds} provides insight into the benefits of regularization and borrowing strength as encoded in the criterion $V_{\boldsymbol \xi^N}^s(\theta)$. The result reveals that the sup distance between the robust and the true criterion for sample $s$ can be bounded by a weighted sum of (i) the maximum distance between the naive empirical risk criterion and the true one, (ii) the sup distance between a mix of the regularized and empirical across-group criteria from the true criterion, and (iii) a term depending on the curvature of $\phi$. In particular, this characterization implies that, if group $s$ is similar in distribution to the other groups, the worst-case distance between our criterion and the ground truth can be improved through the borrowing of strength enabled by the proposed methodology. The same holds for the regularization term $\mathcal R_H(\theta)$, which can improve worst-case performance if it encodes accurate prior information on the data-generating process $p_s^\star$ (e.g., sparsity, correlation structure, and so on).

The next Theorem is useful because it reveals the possibility of obtaining finite-sample probabilistic performance certificates by establishing analogous guarantees for the naive empirical risk of each group $s$. The latter is a classic topic in modern statistical learning theory, which has produced a variety of techniques to ensure ERM convergence by imposing restrictions on the complexity of the function class $\mathscr H := \{h(\theta, \cdot) : \theta \in\Theta\}$, for instance by controlling its VC dimension, metric entropy, etc. We refer the reader to \cite{wainwright2019high} and \cite{vershynin2018high} for an exhaustive treatment of the topic. We also highlight that such a straightforward transfer from classical theory to our methodology is a key dividend of the smoothness and tractability of the proposed criterion.

\begin{theorem}\label{thm:finite_sample_bounds}
    Under the assumptions of Lemma~\ref{lem:finite_sample_bounds},
    \begin{align*}
    & \mathbb P[ \phi(\mathcal R_{p_s^\star}(\theta_N^s)) - \phi(\mathcal R_{p_s^\star}(\theta_\star^s))\leq \delta] \\
    & \geq \mathbb P\Bigg[ \sup_{\theta \in \Theta}\vert\mathcal R_{p_{\boldsymbol \xi_s}}(\theta) - \mathcal R_{p_s^\star}(\theta)\vert \,\leq\, \frac{\alpha_s + N_s}{N_s}\left(\frac{\delta}{2F_\phi} - \frac{\alpha_s}{\alpha_s + N_s}K- \frac{K^2}{2}\frac{S_\phi}{F_\phi}\right) \Bigg]
\end{align*}
    for all $\delta>0$.
\end{theorem}

Before looking at asymptotic results, we note that while the finite-sample probability statements in Lemma~\ref{lem:finite_sample_bounds} and Theorem~\ref{thm:finite_sample_bounds} hold for each group $s$ individually, they can be extended to uniform results across all $S$ groups via a simple application of the union bound.

Turning to asymptotic convergence results as the sample size $N_s$ increases, finite-sample guarantees on 
\[
\sup_{\theta \in \Theta}\vert\mathcal R_{p_{\boldsymbol \xi_s}}(\theta) - \mathcal R_{p_s^\star}(\theta)\vert
\]
are typically of the form
\begin{equation*}
    \mathbb P\Big[\sup_{\theta \in \Theta}\vert\mathcal R_{p_{\boldsymbol \xi_s}}(\theta) - \mathcal R_{p_s^\star}(\theta)\vert\leq\delta\Big]\geq 1-\eta_n,
\end{equation*}
with $\sum_{n=1}^\infty\eta_n<\infty$. By the first Borel-Cantelli lemma, this implies that
\begin{equation*}
    \lim_{N_s\to\infty}\sup_{\theta \in \Theta}\vert\mathcal R_{p_{\boldsymbol \xi_s}}(\theta) - \mathcal R_{p_s^\star}(\theta)\vert = 0
\end{equation*}
almost surely. We adopt this convergence as an assumption in the next Proposition, which establishes the convergence of optimal values to the true target. Moreover, we explicitly introduce a dependence of $\phi$ on the sample size $n$, and denote $\phi\equiv\phi_n$ accordingly.\footnote{Note that the assumptions imposed on $\phi_n$ intuitively (and desirably) require that, as the sample size grows, the ambiguity aversion of the criterion vanishes (i.e., $\phi_{N_s}$ converges smoothly to the identity function). These assumptions are satisfied, for example, by choosing $\phi_n(t) = \beta_n\exp(t/\beta_n) - \beta_n$, with $\lim_{n\to\infty}\beta_n =\infty$, which from now on we silently assume to be our choice of $\phi$.}

\begin{theorem}
\label{thm:uniform_convergence_criterion}
    Retain the assumptions of Lemma \ref{lem:finite_sample_bounds} and, for all $s=1,\dots,S$, assume
    \begin{equation*}
        \lim_{N_s\to\infty}\sup_{\theta \in \Theta}\big\vert \mathcal R_{p_{\boldsymbol\xi_s}}(\theta) - \mathcal R_{p_s^\star}(\theta)\big\vert = 0
    \end{equation*}
    almost surely. Moreover, assume that $\phi_n$ satisfies (1) $\lim_{n\rightarrow \infty}S_{\phi_n} = 0$, (2) $\sup_{n\geq 1}F_{\phi_n} < \infty$, and (3) $\lim_{n\rightarrow\infty}\sup_{t\in[0,K]}\vert\phi_n(t)-t\vert = 0$. Then the next two limits hold almost surely for every group $s=1,\dots,S$:
    \begin{align*}
        \lim_{N_s\rightarrow\infty}\mathcal R_{p_s^\star}(\theta_N^s) = \mathcal R_{p_s^\star}(\theta_\star^s), \quad \lim_{N_s\rightarrow\infty}V_{\boldsymbol \xi^N}(\theta_N^s) = \mathcal R_{p_s^\star}(\theta_\star^s).
    \end{align*}
\end{theorem}

Finally, in the next Theorem, we prove the convergence of the robust criterion optimizers to the target parameter depending on the true unknown data-generating mechanism.

\begin{theorem}\label{thm:optimizer_convergence}
    Let $\theta\mapsto h(\theta, \xi)$ be continuous for all $\xi\in\Xi$ and $\lim_{N_s\rightarrow\infty}\mathcal R_{p_s^\star}(\theta_N^s) = \mathcal R_{p_s^\star}(\theta_\star^s)$ almost surely. Then, almost surely and for all samples $s=1,\dots,S$, $\lim_{N_s\to\infty}\theta_N^s = \bar\theta$ implies $\mathcal R_{p_s^\star}(\bar\theta)=\mathcal R_{p_s^\star}(\theta_\star^s)$.
\end{theorem}

\section{Monte Carlo Approximation}\label{sec:MC_approx}

Due to the infinite dimensionality of the HDP marginal posterior
$Q^s_N(\mathrm dp)$ and the non-linearity of the convex transformation $\phi$, the integral defining $V_{\boldsymbol{\xi}^N}^s(\theta)$ in Equation \eqref{eq:HDP_criterion} is analytically intractable. Hence, for practical implementation, we need to resort to suitable approximation schemes. For the purpose of such approximations, the hierarchical nature of the HDP significantly complicates our analysis compared to the homogeneous data, DP-based criterion, so we opt to treat the two separately (also, treating the DP case first will help to build intuition for the HDP-based criterion).

\subsection{Single-Source Criterion Approximation}

To approximate the DP-based DRO criterion
\begin{equation*}
    V_{\boldsymbol{\xi}^n}(\theta) = \int_{\mathcal P_\Xi} \phi\left(\mathcal R_p(\theta)\right)\, Q_{\boldsymbol{\xi}^n}(\mathrm dp),
\end{equation*}
we borrow tools from a rich literature on DP representations. Specifically, \cite{sethuraman1994constructive} proved that the DP enjoys the following ‘‘stick-breaking'' (SB) representation
\begin{equation*}
    p\sim\textnormal{DP}(\alpha,P)\implies p\overset{\textnormal d}{=} \sum_{j=1}^\infty p_j \delta_{x_j},
\end{equation*}
where
\begin{align*}
    x_j & \overset{\textnormal{iid}}{\sim} P, \quad j = 1,2,\dots, \\
    p_1 & = B_1,\quad p_j = B_j\prod_{i=1}^{j-1} B_i, \quad j = 2,3,\dots, \\
    B_j & \overset{\textnormal{iid}}{\sim} \textnormal{Beta}(1,\alpha), \quad j = 1,2,\dots
\end{align*}
The name of the procedure comes from the analogy with breaking a stick of length 1 into two pieces of length $B_1$ and $1-B_1$, then the second piece into two sub-pieces of length $(1-B_1)B_2$ and $(1-B_1)(1-B_2)$, and so on.

Therefore, because $Q_{\boldsymbol{\xi}^n}$ is itself a DP with updated parameters \citep{ferguson1973bayesian}, one can draw $M\in\mathbb N$ (approximate) iid samples
\begin{equation*}
    p^m = \sum_{j=0}^T p_j^m \delta_{\xi_{mj}}
\end{equation*}
from $Q_{\boldsymbol{\xi}^n}$ by repeating the above SB procedure (appropriately truncated at some finite number $T$ of steps, see Algorithm~\ref{alg:stickbreaking} in Appendix~\ref{app:background}) and optimizing the criterion
\begin{equation}\label{eq:SB_DP_criterion}
    \hat V_{\boldsymbol{\xi}^n}(\theta, T, M) := \frac{1}{M} \sum_{m=1}^M \phi\left(\sum_{j=0}^T p_j^m h(\theta, \xi_{mj})\right).
\end{equation}
See Algorithm~\ref{alg:approx_criterion} in Appendix~\ref{app:background} for more details on this approximation scheme (as usual, the DP criterion case corresponds to setting $S=1$ and $\alpha_0 = \infty$), as well as for another approximation method based on the Multinomial-Dirichlet (MD) construction\footnote{In all of our experiments, we use this alternative approximation scheme as it generally yields more balanced weights $p_1^m,\dots, p_T^m$ for moderate values of $T$ and $n$, thus stabilizing optimization.} of the DP (Algorithm~\ref{alg:mult_dir}). Also, in Theorems~\ref{thm:uniform_convergence_SB} and \ref{thm:optimizer_convergence_SB} in Appendix~\ref{app:theorems}, we develop convergence guarantees for the optimization procedure based on the approximate criterion $\hat V_{\boldsymbol{\xi}^n}(\theta, T, M)$ to the infinite-dimensional target $V_{\boldsymbol{\xi}^n}(\theta)$.

\subsection{Multi-Source Criterion Optimization}

The multi-source, HDP-based criterion poses some additional challenges in terms of approximation, due to the complex hierarchical structure induced by the HDP posterior. Luckily, we are able to exploit some recent distributional results from the literature on Bayesian nonparametric models for dependent distributions.

\begin{proposition}[\cite{camerlenghi2019distribution}, Thms. 9 and 10, Ex. 5]\label{pro:HDP_post_characterization} Assume that $(p_1,\dots,p_S)$ is modeled as a HDP and that the true laws $(p_1^\star, \dots,p_S^\star)$ are diffuse. Then, for all $s= 1,\dots,S$,
\begin{align*}
    p_0 \mid \boldsymbol \xi_{1:S}& \sim \textnormal{DP} \Bigg(\alpha_0  + N,
    \frac{N}{\alpha_0 + N}\frac{1}{N}\sum_{s=1}^S\sum_{j=1}^{N_s}\delta_{\xi_{sj}} + \frac{\alpha_0}{\alpha_0 + N}H \Bigg), \\
    p_s \mid p_0, \boldsymbol\xi_{1:S} & \overset{\textnormal{id}}{\sim} \textnormal{DP}\Bigg(\alpha_s  + N_s, \,
    \frac{N_s}{\alpha_s + N_s}\frac{1}{N_s}\sum_{j=1}^{N_s}\delta_{\xi_{sj}} + \frac{\alpha_s}{\alpha_s + N_s} p_0 \Bigg).
\end{align*}
\end{proposition}
Proposition \ref{pro:HDP_post_characterization} allows to set up a two-stage procedure to obtain an approximate Monte Carlo sample from the marginal posterior $Q_N^s$ of $p_s$: First, simulate $p_0\mid \boldsymbol \xi_1, \dots, \boldsymbol \xi_S$, then simulate $p_s \mid p_0, \boldsymbol\xi_1, \dots, \boldsymbol \xi_S$. Finally, the sampled measure $p_s$ (which, as a DP realization, is discrete) can be used to compute $\phi(\mathcal R_{p_s}(\theta))$, and a Monte Carlo average over many such samples approximates $V_N^s(\theta) = \int \phi(\mathcal R_{p_s}(\theta)) Q_N^s(\mathrm dp_s)$. This two-step algorithm simplifies the initial task of drawing samples from the HDP posterior by only requiring to approximately simulate from univariate DP distributions, which is easy by appealing to the same SB and MD constructions of the DP used to approximate the single-source criterion in Equation~\eqref{eq:SB_DP_criterion}.

\begin{remark}
    An important question is how to select the hyperparameters $\alpha_0$ and $\alpha_s$. Throughout our experiments, we employ cross-validation to optimize out-of-sample prediction. While this is a widely-applicable solution, data scenarios involving a high degree of domain expertise (e.g., about the level of dependence across samples) may make user-specified values an attractive choice.
\end{remark}

\section{Gradient-Based Optimization}\label{sec:gradient_optimization}

Given the approximation strategies we proposed for the HDP criterion, in practice one ends up minimizing a function of the form
\begin{equation*}
    \hat V_s(\theta) = \frac{1}{M}\sum_{m=1}^M \phi\left(\sum_{j=1}^T p_j^m h(\theta,\xi_j^m)\right)
\end{equation*}
for each sample $s$. Under mild regularity assumptions on the loss function $h(\theta,\xi)$, $\hat V_s$ is easily optimized via first-order methods. For instance, assuming we have access to the  gradient $\nabla_\theta h(\theta,\xi)$ for all $\xi\in\Xi$, a simple yet scalable solution is to adopt a stochastic gradient descent (SGD) algorithm of the following type: At each iteration $t$, select a (possibly random) index $m_t\in\{1,\dots,M\}$, then perform a gradient-based update of the form
\begin{equation}\label{eq:sgd_step}
    \theta^{t} = \theta^{t-1} - \eta_t \underbrace{\phi'\left(\sum_{j=1}^T p_j^{m_t} h(\theta^{t-1},\xi_j^{m_t})\right) \sum_{j=1}^T p_j^{m_t} \nabla_\theta h(\theta^{t-1},\xi_j^{m_t})}_{(\star)},
\end{equation}
where $\eta_t$ is a pre-specified step size and $(\star)$ is an unbiased Monte Carlo estimate of the gradient of $\hat V_s$ evaluated at $\theta^{t-1}$ (see Algorithm~\ref{alg:SGD_modified} in Appendix~\ref{app:background} for full details).

This procedure highlights several appealing features of our method. First, the smoothness of the criterion opens the door for simple, off-the-shelf optimization procedures. Second, the convexity of $h$ is easily seen to be inherited by $\hat V_s$, so that standard SGD convergence results for convex objectives hold \citep{garrigos2023handbook}. Third, the form of the gradient allows to choose the truncation step $T$ and the number of Monte Carlo samples $M$ by interpreting them as the SGD minibatch size and the number of passes over the data, respectively. In fact, assume that the number of SGD iterations is chosen as a multiple of $M$ and that $m_t$ is chosen deterministically as follows: $m_0 =1$ and $m_t = m_{t-1}+1$. Then, the algorithm requires $T$ gradient evaluations at each step, and it iterates $M$ times over the whole (augmented) data.\footnote{This interpretation of $T$ and $M$ also clarifies why larger values of these parameters are necessary when dealing with larger sample sizes, to ensure that the approximated criterion accurately represents the available data. For practical purposes, especially with large datasets, we recommend keeping $T$ relatively small to maintain manageable per-iteration costs, while increasing $M$ accordingly.} Finally, because $\phi$ is convex, $\phi'$ is increasing, so that $(\star)$ can be interpreted as a form of \emph{robustly weighted gradient}: The worse the current parameter value $\theta^{t-1}$ performs on the selected minibatch $m_t$, the more the procedure weights the corresponding gradient step.

\section{Experiments on Single-Source DP-Based Criterion}\label{sec:DP_experiments}

In this Section, we present the results of numerical experiments conducted on single-source data using our DP-based DRO method.

\paragraph{Simulation Experiments.} We test our method on two different learning tasks featuring a high degree of distributional uncertainty in the data-generating process, and compare its performance with the corresponding ambiguity-neutral (i.e., simply regularized) and unregularized procedures. First, we perform a high-dimensional sparse linear regression simulation experiment. We simulate 200 independent samples of size $n=100$ from a linear model with $d=90$ features (moderately correlated with each other), only the first $s=5$ of which have unitary positive marginal effect on the scalar response $y$; see Figure~\ref{fig:DP_lin_reg_simulations} for a summary of the results. Second, we perform a simulation experiment on high-dimensional sparse logistic regression for binary classification. We set up a data-generating mechanism similar to the linear regression experiment, where a small subset of features linearly influence the log odds ratio.

Appendix~\ref{app:experiments_DP_DRO} collects further details on the above experiments as well as plots summarizing the results of the second experiment (see Figure~\ref{fig:DP_logit_simulations}). All three experiments reveal the ability of our robust method to improve out-of-sample performance and estimation accuracy in two ways, i.e., (i) by yielding good results on average and (ii) by reducing performance variability. The latter is a key robustness property that our method is designed to achieve.

\paragraph{Real Data Applications.} We test our method on three diverse real-world datasets. In the first study, we apply our method to predict diabetes development based on a host of features, as collected in the popular and public Pima Indian Diabetes dataset. Because the outcome is binary, we use logistic regression as implemented (i) with our robust method, (ii) with $L^1$ regularization, and (iii) in its plain, unregularized version. We select hyperparameters via cross-validation and test the out-of-sample performance of the three methods on disjoint batches of training observations to assess performance variability. As reported in Appendix~\ref{app:experiments_DP_DRO}, our robust method outperforms both alternatives on average and does significantly better in reducing variability.

\begin{figure*}[t]
\begin{center}
\centerline{\includegraphics[width=0.9\textwidth]{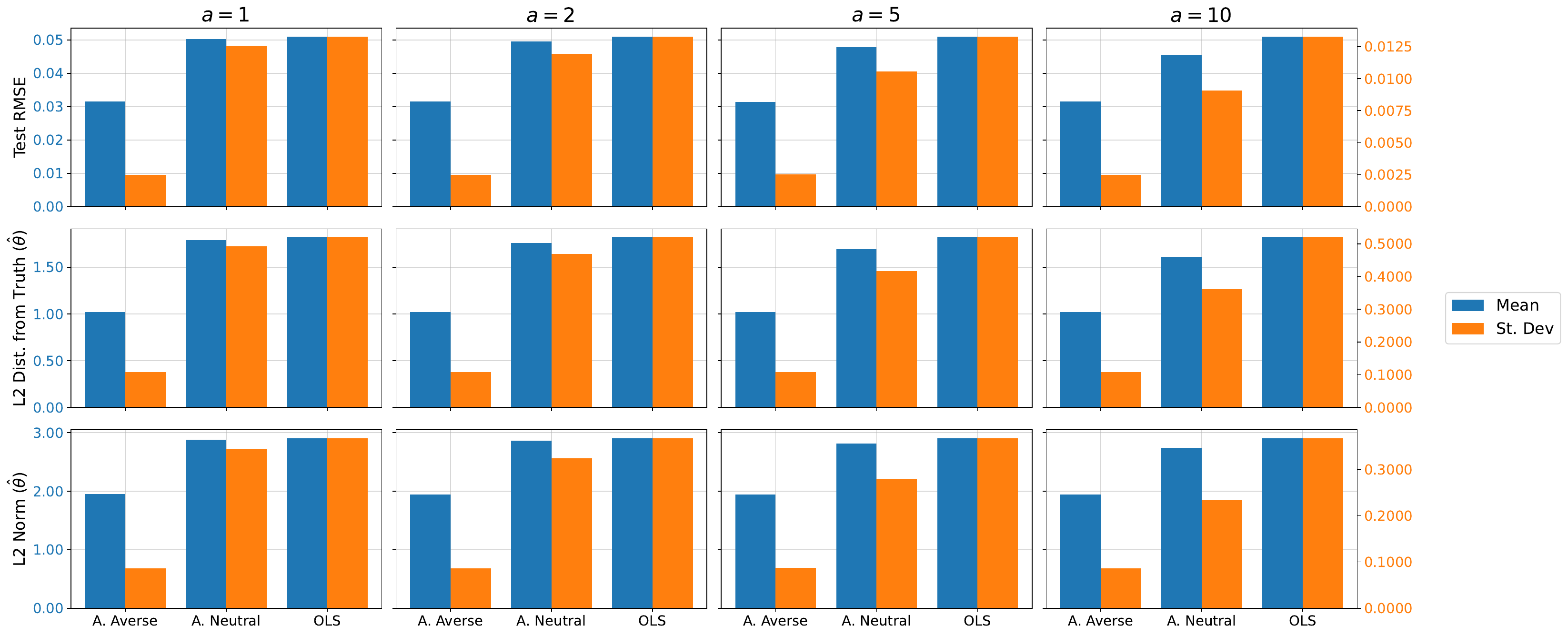}}
\caption{Simulation results for the high-dimensional sparse linear regression experiment. Bars report the mean and standard deviation (across 200 sample simulations) of the test RMSE, $L^2$ distance of the estimated coefficient vector $\hat\theta$ from the data-generating one, and the $L^2$ norm of $\hat\theta$. Results are shown for the ambiguity-averse, ambiguity-neutral, and OLS procedures. Note: The left (blue) axis refers to mean values, the right (orange) axis to standard deviation values.}
\label{fig:DP_lin_reg_simulations}
\end{center}
\end{figure*}

We conduct two further studies on linear regression using two popular UCI Machine Learning Repository datasets: the Wine Quality dataset \citep{misc_wine_quality_186} and the Liver Disorders dataset \citep{misc_liver_disorders_60}. Similarly to the first study, we compare the performance of our method to OLS (unregularized) estimation and $L^1$-penalized (LASSO) regression. After cross-validation for parameter selection, we train the models multiple times on separate batches of data and compute out-of-sample performance on a large held-out set of observations. As the results in Appendix~\ref{app:experiments_DP_DRO} show, in these settings our robust DP-based method performs better than the alternatives both on average and, notably, in terms of lower variability.

\section{Experiments on Multi-Source HDP-Based Criterion}\label{sec:HDP_experiments}

This Section reports the results of numerical experiments performed on multi-source data with our HDP-based DRO method.

\paragraph{Simulation Experiments.} In the first experiment, we test the performance of the HDP robust criterion in a two-sample high-dimensional linear regression task, comparing it to OLS and robust DP estimation (both pooling samples and keeping them separate). In each simulation, we generate two size-100 samples of 95 features and a response, where the response is linearly influenced by only 5 features. Moreover, the 5 non-zero coefficients are simulated at each iteration with positive dependence across samples, and we explore various degrees of dependence. Figure~\ref{fig:lin_reg_main} shows the results of the study with 100 simulations, revealing that the robust HDP method outperforms all of the other methods in terms of both out-of-sample risk and estimation accuracy. Importantly, in addition to performing better on average, HDP-robust estimation displays less variable performance. Figure 3 in Appendix~3 shows similar results when the degree of dependence among group laws is either reduced or increased.

\begin{figure}[t]
    \centering
    \begin{subfigure}[b]{\textwidth}
        \centering
        \includegraphics[width=0.9\textwidth]{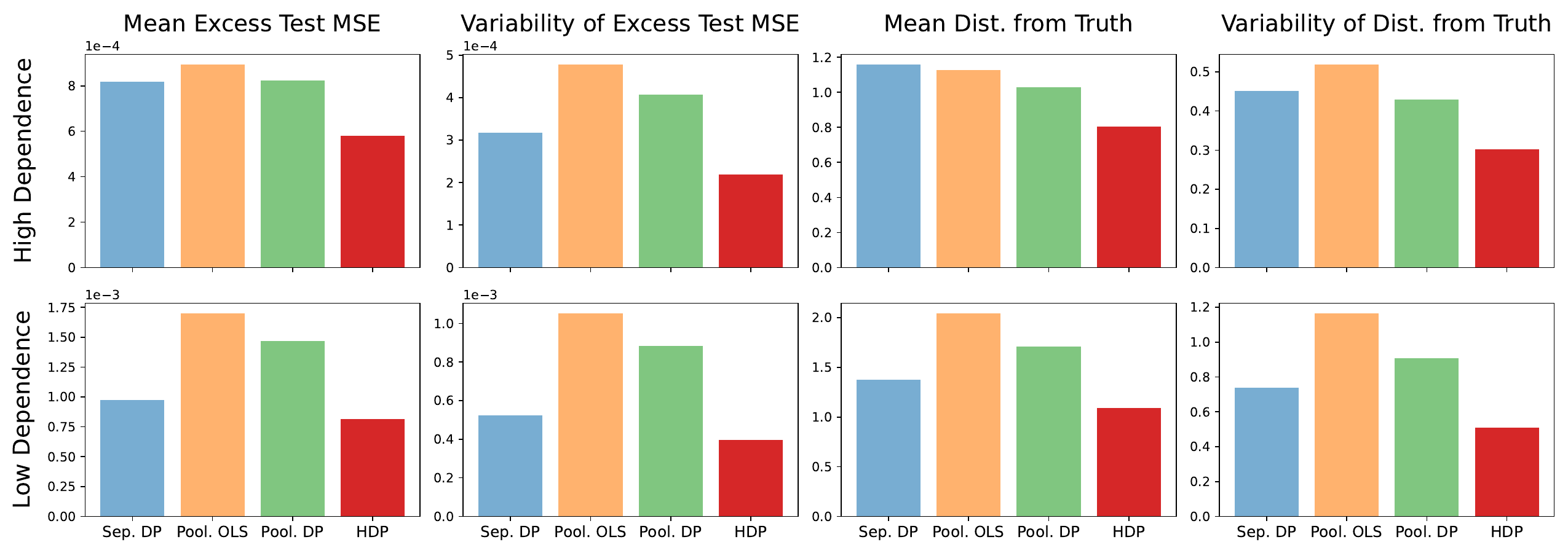}
        \caption{Mean regression.}
        \label{fig:lin_reg_main}
    \end{subfigure}
    
    \vspace{0.3cm}
    \begin{subfigure}[b]{\textwidth}
        \centering
        \includegraphics[width=0.9\textwidth]{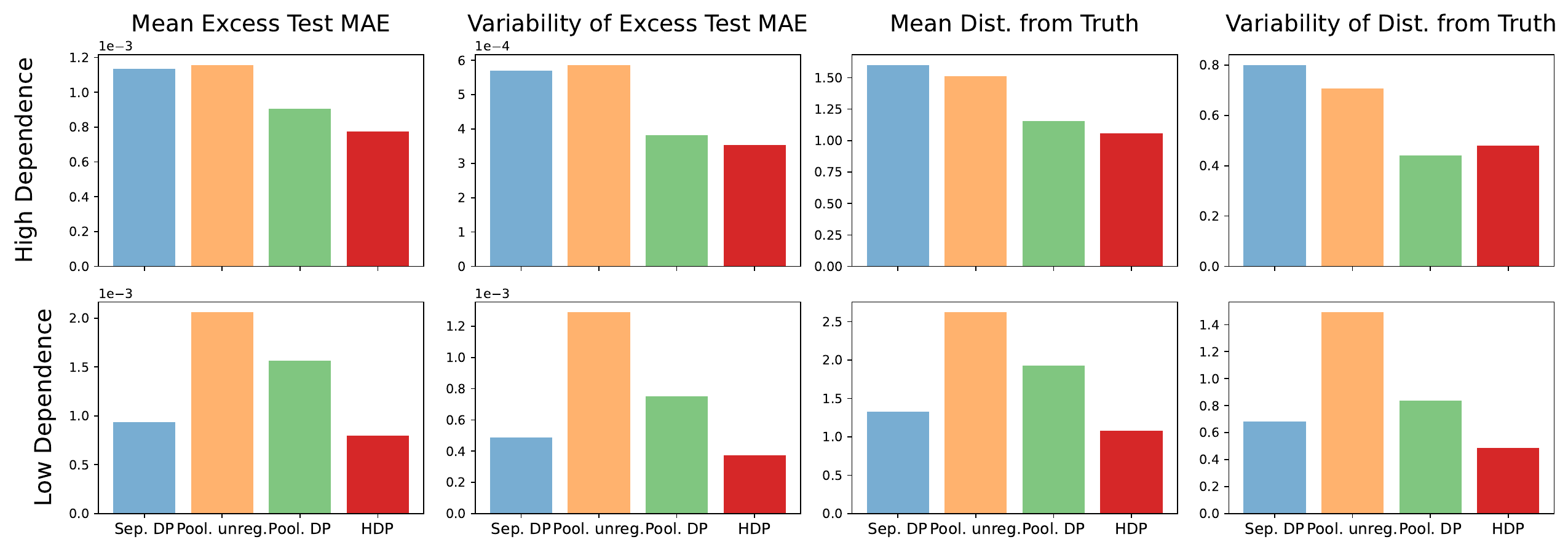}
        \caption{Median regression.}
        \label{fig:med_reg_HDP_main}
    \end{subfigure}
    \vspace{.01in}
    \caption{Comparison of out-of-sample performance and estimation accuracy of different methods in high-dimensional regression experiments. The robust HDP method (in bright red) outperforms the others both in terms of average performance and variability. Note: Distance from truth is measured as the squared $L^2$ distance between the estimated coefficients and the data-generating ones.}
    \label{fig:combined}
\end{figure}

In the second simulation experiment, we test the performance of the HDP robust criterion in a two-sample high-dimensional median linear regression task, comparing it to the same baselines as above. Instead of recovering the conditional mean structure of the data-generating process, this method aims to reconstruct the conditional median of the response variable as a linear function of the features, which makes estimation more robust to outlier data points \citep{koenker2005quantile}. Using a data-generating process analogous to the first experiment, we test the ability of our HDP robust model to improve and stabilize performance when varying degrees of dependence are induced across heterogeneous groups. Figure~\ref{fig:med_reg_HDP_main} (and Figure 4 in Appendix~3) shows results in line with those of the linear regression experiment, with our robust HDP method generally outperforming the baselines on average and in terms of variability.

\paragraph{Real Data Applications.} We further validate the applicability of our method with three real-world learning tasks. In the first two experiments, we apply our method, as well as ERM optimization, to a classification task aimed at predicting diabetes (encoded as binary) from baseline health conditions measured on patients from different age cohorts. In these experiments, we select the loss function to correspond to either logistic regression or support vector machine classification, and each experiment is performed across two age cohorts ($S=2$, differing across experiments). In the third experiment, we apply our method and ERM to a support vector regression task, predicting wine quality from baseline characteristics across red and white wines ($S=2$).

In all three experiments, our HDP method shows favorable out-of-sample predictive performance compared to pooled- and separate-samples baselines, both on average and in terms of variability (see Tables \ref{tab:diabetes_log_reg}, \ref{tab:diabetes_SVM}, and \ref{tab:winequality_SVR} in Appendix~\ref{app:experiments_HDP_DRO}), highlighting the benefit of partial borrowing of information in the context of data-driven DRO.

\section{Incorporating Outlier Robustness}\label{sec:DORO}

In the context of DRO based on ambiguity sets, \cite{zhai2021doro} empirically demonstrated that unadjusted DRO methods become unstable and exhibit suboptimal performance in the presence of outliers. The intuition behind this phenomenon is straightforward: since these methods aim to minimize risk under the worst-case distribution, the presence of outliers makes this distribution highly unstable and potentially far from the true underlying process. To address this issue, \cite{zhai2021doro} proposed a modified optimization procedure called Distributionally and Outlier Robust Optimization (DORO), which dynamically filters out the worst-fitting training instances at each gradient-based parameter update.

While our smooth DRO method does not directly perform worst-case optimization, its ambiguity aversion still implies a tendency to overemphasize outliers during training. Indeed, recall the SGD update in Equation~\eqref{eq:sgd_step}, which assigns greater weight to batches with higher risk under the current parameter value. Consequently, batches containing more outliers may receive disproportionately higher weights, and a persistent presence of outliers may destabilize the optimization process and degrade its performance.

\begin{figure}
    \centering
    \includegraphics[width=0.8\linewidth]{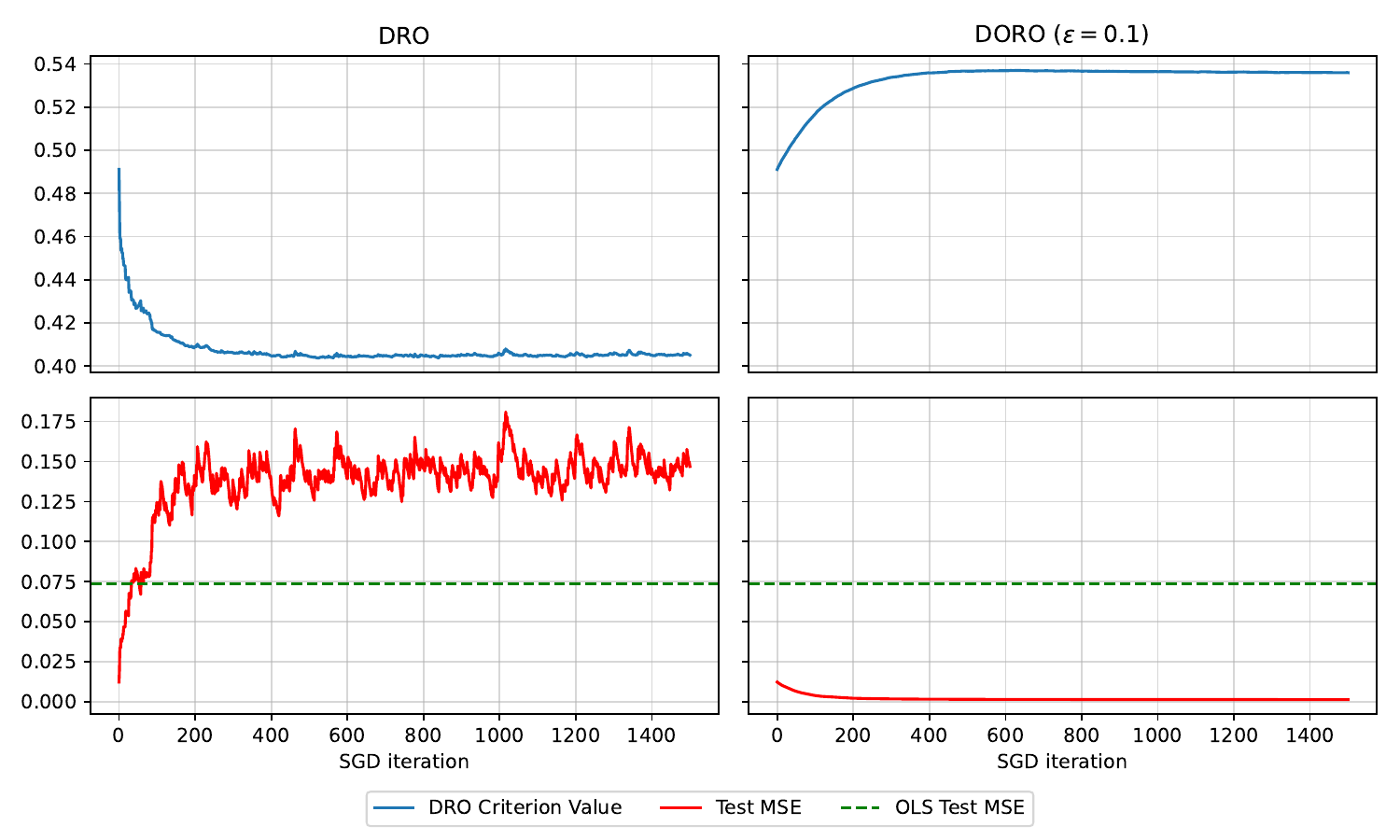}
    \caption{Loss dynamics of DP-based DRO and DORO routines.}
    \label{fig:DRO_DORO_losses}
\end{figure}

To empirically illustrate this issue, we conduct a simple regression experiment using a high-dimensional Gaussian linear model identical to our previous experiments, but with 5\% of observations replaced by samples from a significantly different distribution---effectively injecting outliers. We then apply our DP-based DRO procedure to this contaminated dataset.\footnote{For clarity, this section focuses on the single-source, DP-based method, though the discussion readily extends to the multi-source, HDP-based case.} The first column of Figure~\ref{fig:DRO_DORO_losses} confirms our expectation: while the training loss decreases and stabilizes with more gradient steps (indicating convergence of the SGD algorithm), the out-of-sample performance of the learned parameters degrades and becomes highly unstable. In fact, it deteriorates to the point of performing worse than OLS, which, as seen in our previous experiments, typically struggles in high-dimensional settings. This motivates us to develop a refined procedure that incorporates outlier robustness into our DRO framework.

Formally, given a training sample $\boldsymbol\xi^{\textnormal{tr}}\subset \Xi$ of size $n\in \mathbb{N}$, we define the set
\begin{equation*}
    A^{\textnormal{tr}}_\varepsilon := \left\{\boldsymbol{\xi} \subset \Xi : \exists \boldsymbol\xi' \subset \Xi \textnormal{ such that } p_{\boldsymbol{\xi^{\textnormal{tr}}}} = (1-\varepsilon)p_{\boldsymbol\xi} + \varepsilon p_{\boldsymbol\xi'} \right\}.
\end{equation*}
Using this notation, we consider the following modified optimization problem:
\begin{equation}\label{eq:theoretical_DORO_problem}
    \min_{\theta\in\Theta}V_{\boldsymbol \xi^{\textnormal{tr}}}^{\varepsilon}(\theta), \quad \textnormal{where }V_{\boldsymbol \xi^{\textnormal{tr}}}^{\varepsilon}(\theta) := \inf_{\boldsymbol{\xi}\in A^{\textnormal{tr}}_\varepsilon}V_{\boldsymbol{\xi}}(\theta),
\end{equation}
where $\varepsilon\in[0,1/2)$ represents the level of outlier contamination.\footnote{For simplicity, we use $\varepsilon$ as shorthand for $\lceil n\varepsilon\rceil/n$ when necessary.} Intuitively, this formulation introduces an additional ``optimistic'' layer to the optimization process, filtering out the worst-performing $(\varepsilon \times 100)\%$ of observations in terms of the standard DRO criterion $V_{\boldsymbol{\xi}^n}$. This is particularly beneficial when we suspect that the training data originates from a Huber-type contaminated distribution $\tilde p = (1-\varepsilon_\star)p_\star + \varepsilon_\star p_c$, where $p_\star$ is the target distribution and $p_c$ is a contaminating distribution. In what follows, we assume that this contamination model holds and that we can choose $\varepsilon \geq \varepsilon_\star$ accordingly.

In the next Theorem, we analyze the convergence behavior of the proposed DORO procedure, similar to Theorem~\ref{thm:uniform_convergence_criterion}. Specifically, we establish that if a modified ERM procedure that filters out outliers is consistent in terms of predictive risk, then our method is consistent as well. This result suggests that, as long as the destabilizing effect of outliers can be asymptotically mitigated by filtering them out, our procedure successfully does so while preserving distributional robustness in finite samples.

\begin{theorem}\label{thm:DORO_criterion_convergence}
    Assume that $h$ is bounded between $0$ and $K < \infty$, and that the function $\phi_n(t) = \beta_n \exp(t/\beta_n) - \beta_n$ satisfies $\lim_{n\to\infty} \beta_n = \infty$. Moreover, define 
    \[
    \hat\theta_\varepsilon \in\arg\min_{\theta\in\Theta}V_{\boldsymbol \xi^{\textnormal{tr}}}^{\varepsilon}(\theta), \quad \theta_\star\in\arg\min_{\theta\in\Theta}\mathcal{R}_{p_\star}(\theta).
    \]
    Then, the following holds:
    \begin{equation*}
        \phi_n(\mathcal{R}_{p_\star}(\hat\theta_\varepsilon)) - \phi_n(\mathcal{R}_{p_\star}(\theta_\star)) \leq o(1) + O(1) \cdot \sup_{\theta\in\Theta} \Big\lvert \inf_{\boldsymbol{\xi}\in A^{\textnormal{tr}}_\varepsilon} \mathcal{R}_{p_{\boldsymbol{\xi}}}(\theta) - \mathcal{R}_{p_\star}(\theta) \Big\rvert,
    \end{equation*}
    where the order symbols are with respect to the sample size $n$.
\end{theorem}

As in the standard DRO framework, the optimization problem in Equation~\eqref{eq:theoretical_DORO_problem} is inherently infinite-dimensional and cannot be solved directly. To make it tractable, we adopt a similar approximation strategy to that used for the DRO criterion and instead optimize:
\begin{equation*}
    \min_{\theta\in\Theta} \hat V_{\boldsymbol \xi^{\textnormal{tr}}}^{\varepsilon}(\theta, T, M), \quad \textnormal{where } \hat V_{\boldsymbol \xi^{\textnormal{tr}}}^{\varepsilon}(\theta, T, M) := \inf_{\boldsymbol{\xi}\in A^{\textnormal{tr}}_\varepsilon}\hat V_{\boldsymbol{\xi}}(\theta, T, M),
\end{equation*}
with $\hat V_{\boldsymbol{\xi}}(\theta, T, M)$ defined in Equation~\eqref{eq:SB_DP_criterion} using one of the previously discussed approximation schemes for the DP posterior. Algorithm~\ref{alg:DORO_SGD} in Appendix~\ref{app:background} details an implementable SGD algorithm for solving this problem, filtering out the worst-fitting $(\varepsilon \times 100)\%$ of observations---most likely outliers---at each gradient step.

The second column of Figure~\ref{fig:DRO_DORO_losses} illustrates the performance of the proposed DORO procedure. As expected, while it does not reduce the in-sample DRO criterion value, it successfully mitigates the sensitivity of DRO to outliers by improving and stabilizing out-of-sample generalization, bringing it back to a level better than OLS. Additionally, Figure~\ref{fig:DORO_lin_reg_histograms} in Appendix~\ref{app:experiments_DORO} presents results from repeating the experiment across 200 independent simulations, demonstrating that DORO effectively stabilizes optimization in the presence of outliers by improving generalization performance on average and reducing variability.

\section{Conclusion}\label{sec:conclusion}

This work explores some deep connections between data-driven Distributionally Robust Optimization (DRO) and Economic Decision Theory under Ambiguity (DTA). By demonstrating that many DRO formulations can be reinterpreted as data-driven analogues of ambiguity-averse decision models, we provide a theoretical synthesis that clarifies their shared foundations. This perspective not only unifies existing approaches but also suggests new methodological pathways for designing robust optimization frameworks grounded in decision theory.  

Building on this synthesis, we introduce a novel DRO approach inspired by smooth ambiguity-averse preferences in DTA, leveraging Bayesian nonparametric techniques to model the data-generating process. Our baseline framework is built on Dirichlet Process (DP) posteriors, which naturally accommodate distributional uncertainty and extend seamlessly to heterogeneous data through Hierarchical Dirichlet Processes (HDPs). In addition, we propose an outlier robust refinement that selectively filters problematic observations during optimization, improving both stability and predictive accuracy. These methodological contributions are supported by theoretical performance guarantees and extensive empirical validation.  

Beyond methodological innovations, our results emphasize the broader potential of integrating ideas from decision theory into data-driven optimization. The conceptual link between DRO and DTA highlights how ambiguity-aware principles can inform the design of robust learning algorithms, moving beyond worst-case formulations toward more flexible, structured approaches to uncertainty. By synthesizing ideas from DRO and DTA, this work contributes both a deeper theoretical understanding of robustness in optimization and practical tools for enhancing reliability in data-driven decision-making. We hope that these insights will inspire further research at the intersection of optimization, statistical learning, and decision theory under uncertainty.




\acks{The authors are thankful to Igor Pr\"unster and Alessandro Rinaldo for helpful comments. NB gratefully acknowledges funding from the ``Giorgio Mortara'' scholarship by the Bank of Italy. }



\appendix
\section{Technical Proofs and Further Theoretical Results}
\label{app:theorems}

In this Appendix, we report the mathematical proofs of the theoretical results reported in the main text, as well as further theoretical results with the corresponding proofs.

\subsection{Proofs of Results Appearing in the Main Text}

\paragraph{Proof of Proposition~\ref{pro:equivalence_regularization}.} Because $p_0 = \mathcal N(0,I_{d+1})$ and given the form of the criterion in Equation~\eqref{eq:ambiguity_neutral_criterion}, we are left to establish an expression for $\mathcal R_{p_0}(\theta) \equiv \mathbb E_{(y, x)\sim \mathcal N(0, I_{d+1})}[(y-\theta^\top x)^2]$. Notice that, for any $\theta = (\theta_1,\dots,\theta_d)$, $(y,x)\sim\mathcal N(0, I_{d+1})$ implies $-\theta_j x_j\overset{\textnormal{id}}{\sim}\mathcal N(0,\theta_j^2)$ independendently of $y\sim \mathcal N(0,1)$, so that $y-\theta^\top x \sim \mathcal N(0, 1 + \Vert\theta\Vert_2^2)$. Therefore, $\mathcal R_{p_0}(\theta) = 1 + \Vert\theta\Vert_2^2$. Therefore
\begin{equation*}
    \mathbb E_{p\sim Q_{\boldsymbol{\xi}^n}}[\mathcal R_p(\theta)] = \frac{n}{\alpha +n} \frac{1}{n}\sum_{i=1}^n (y_i - \theta^\top x_i)^2 + \frac{\alpha}{\alpha+n}(1 + \Vert\theta\Vert_2^2),
\end{equation*}
Because $r\mapsto \frac{\alpha + n}{n}r -\frac{\alpha}{n}$ is a strictly increasing transformation on $\mathbb R$, the desired equivalence is obtained.

\paragraph{Proof of Lemma~\ref{lem:finite_sample_bounds}.} First note that, by the stated assumptions, it follows from Taylor's theorem that
\begin{equation*}
    \phi(\mathcal R_p(\theta)) = \phi(\mathcal R_{p_s^\star}(\theta)) + \phi'(\mathcal R_{p_s^\star}(\theta))[\mathcal R_{p}(\theta)-\mathcal R_{p_s^\star}(\theta)] +\frac{\phi''(c_{p,\theta})}{2}[\mathcal R_{p}(\theta)-\mathcal R_{p_s^\star}(\theta)]^2
\end{equation*}
for all $p\in\mathscr P_\Xi$ and $\theta\in\Theta$ and for some $c_{p,\theta}\in[0,K]$. Then
\begin{align*}
    \sup_{\theta \in \Theta} \vert V_{\boldsymbol \xi^N}^s(\theta) - \phi(\mathcal R_{p_s^\star}(\theta))\vert & = \sup_{\theta \in \Theta} \bigg\vert \phi'(\mathcal R_{p_s^\star}(\theta))\int_{\mathscr P_\Xi}\mathcal [\mathcal R_{p}(\theta)-\mathcal R_{p_s^\star}(\theta)]Q_{\boldsymbol\xi^N}^s(\mathrm dp) \\
                                  & + \int_{\mathscr P_\Xi} \frac{\phi''(c_{p,\theta})}{2}[\mathcal R_{p}(\theta)-\mathcal R_{p_s^\star}(\theta)]^2 Q_{\boldsymbol\xi^N}^s(\mathrm dp)\bigg\vert \\
                                  & \leq \frac{N_s}{\alpha_s + N_s}F_\phi \sup_{\theta \in \Theta}\vert\mathcal R_{p_{\boldsymbol \xi_s}}(\theta) - \mathcal R_{p_s^\star}(\theta)\vert \\
                                  & + \frac{\alpha_s}{\alpha_s + N_s}F_\phi\sup_{\theta \in \Theta}\left\vert\frac{N}{\alpha_0 + N}\mathcal R_{p_{\boldsymbol \xi^N}}(\theta) + \frac{\alpha_0}{\alpha_0 + N}\mathcal R_H(\theta) - \mathcal R_{p_s^\star}(\theta)\right\vert\\
                                  & + \frac{K^2}{2}S_\phi.
\end{align*}

\paragraph{Proof of Theorem~\ref{thm:finite_sample_bounds}.} Notice the following decomposition:
\begin{align}\label{eq:excess_risk_decomposed}
    & \underbrace{\phi(\mathcal R_{p_s^\star}(\theta_N^s)) - \phi(\mathcal R_{p_s^\star}(\theta^s_\star))}_{\geq 0} \\ 
                                                          & = \phi(\mathcal R_{p_s^\star}(\theta_N^s)) - V_{\boldsymbol \xi^N}^s(\theta_N^s) + \underbrace{V_{\boldsymbol \xi^N}^s(\theta_N^s) - V_{\boldsymbol \xi^N}^s(\theta^s_\star)}_{\leq 0} + V_{\boldsymbol \xi^N}^s(\theta^s_\star) - \phi(\mathcal R_{p_s^\star}(\theta^s_\star))\nonumber \\
                                                          & \leq 2\sup_{\theta \in \Theta} \vert V_{\boldsymbol \xi^N}^s(\theta) - \phi(\mathcal R_{p_s^\star}(\theta))\vert. \nonumber
\end{align}
Then, Lemma~\ref{lem:finite_sample_bounds} and the boundedness of $h$ implies that, for all $\delta > 0$,
\begin{align*}
    & \mathbb P[\mathcal \phi(\mathcal R_{p_s^\star}(\theta_N^s)) - \phi(\mathcal R_{p_s^\star}(\theta^s_\star))\leq \delta] \\
    & \geq \mathbb P\left[\sup_{\theta \in \Theta} \vert V_{\boldsymbol \xi^N}^s(\theta) - \phi(\mathcal R_{p_s^\star}(\theta))\vert \leq \delta/2\right] \\
    & = \mathbb P\Bigg[ \sup_{\theta \in \Theta}\vert\mathcal R_{p_{\boldsymbol \xi_s}}(\theta) - \mathcal R_{p_s^\star}(\theta)\vert \leq \frac{\alpha_s + N_s}{N_s}\left(\frac{\delta}{2F_\phi} - \frac{\alpha_s}{\alpha_s + N_s}K- \frac{K^2}{2}\frac{S_\phi}{F_\phi}\right) \Bigg].
\end{align*}

\paragraph{Proof of Theorem~\ref{thm:uniform_convergence_criterion}.} Since
\begin{equation*}
    \lim_{N_s\rightarrow \infty}\sup_{\theta \in \Theta}\vert\mathcal R_{p_{\boldsymbol\xi_s}}(\theta) - \mathcal R_{p_s^\star}(\theta)\vert = 0 
\end{equation*}
almost surely and given assumptions (1) and (2) on $\phi_n$ in the statement of Theorem~\ref{thm:uniform_convergence_criterion}, by Lemma~\ref{lem:finite_sample_bounds} we obtain
\begin{equation*}
    \lim_{N_s\rightarrow \infty}\sup_{\theta \in \Theta}\vert V_{\boldsymbol \xi^N}(\theta) - \mathcal \phi_{N_s}(\mathcal R_{p_s^\star}(\theta))\vert = 0
\end{equation*}
almost surely. Then, by decomposition (\ref{eq:excess_risk_decomposed}),
\begin{equation*}
    \lim_{N_s\rightarrow \infty} \left[\phi_{N_s}(\mathcal R_{p_s^\star}(\theta_N^s)) - \phi_{N_s}(\mathcal R_{p_s^\star}(\theta^s_\star)) \right] = 0
\end{equation*}
almost surely and
\begin{equation*}
    \lim_{N_s\rightarrow \infty} \left[V_{\boldsymbol \xi^N}(\theta_N^s) - V_{\boldsymbol \xi^N}(\theta^s_\star) \right] =0
\end{equation*}
almost surely. As a consequence,
\begin{align*}
    \lim_{N_s\rightarrow\infty}\vert V_{\boldsymbol \xi^N}(\theta_N^s) & - \mathcal \phi_{N_s}(\mathcal R_{p_s^\star}(\theta^s_\star))\vert \\
                                                           & \leq \lim_{N_s\rightarrow\infty} \left[\vert V_{\boldsymbol \xi^N}(\theta_N^s) - V_{\boldsymbol \xi^N}(\theta^s_\star) \vert + \vert V_{\boldsymbol \xi^N}(\theta^s_\star) - \phi_{N_s}(\mathcal R_{p_s^\star}(\theta^s_\star)) \vert \right] \\
                                                           & \leq \lim_{N_s\rightarrow\infty} \left[\vert V_{\boldsymbol \xi^N}(\theta_N^s) - V_{\boldsymbol \xi^N}(\theta^s_\star) \vert + \sup_{\theta \in \Theta}\vert V_{\boldsymbol \xi^N}(\theta) - \phi_{N_s}(\mathcal R_{p_s^\star}(\theta)) \vert \right] \\
                                                           & = 0
\end{align*}
almost surely. Now recall assumption (3) in the statement of Theorem~\ref{thm:uniform_convergence_criterion}, i.e., the sequence $(\phi_n)_{n\geq 1}$ converges uniformly to the identity map. Then, in light of the previous observations and by noticing that
\begin{align*}
    \vert\mathcal R_{p_s^\star}(\theta_N^s) - \mathcal R_{p_s^\star}(\theta^s_\star)\vert & \leq \vert \mathcal R_{p_s^\star}(\theta_N^s) - \phi_{N_s}(\mathcal R_{p_s^\star}(\theta_N^s))\vert + \vert \phi_{N_s}(\mathcal R_{p_s^\star}(\theta_N^s)) - \phi_{N_s}(\mathcal R_{p_s^\star}(\theta^s_\star)) \vert \\
                                                                      & + \vert \phi_{N_s}(\mathcal R_{p_s^\star}(\theta^s_\star)) - \mathcal R_{p_s^\star}(\theta^s_\star)\vert
\end{align*}
and
\begin{equation*}
    \vert V_{\boldsymbol \xi^N}(\theta_N^s) - \mathcal R_{p_s^\star}(\theta^s_\star) \vert \leq \vert V_{\boldsymbol \xi^N}(\theta_N^s) - \phi_{N_s}(\mathcal R_{p_s^\star}(\theta^s_\star)) \vert + \vert \phi_{N_s}(\mathcal R_{p_s^\star}(\theta^s_\star)) - \mathcal R_{p_s^\star}(\theta^s_\star) \vert,
\end{equation*}
the two desired almost sure limits follow:
\begin{align*}
    \lim_{N_s\rightarrow\infty}\mathcal R_{p_s^\star}(\theta_N^s) = \mathcal R_{p_s^\star}(\theta^s_\star), \qquad
    \lim_{N_s\rightarrow\infty}V_{\boldsymbol \xi^N}(\theta_N^s) = \mathcal R_{p_s^\star}(\theta^s_\star).
\end{align*}

\paragraph{Proof of Theorem~\ref{thm:optimizer_convergence}.}
We have
\begin{align*}
    \mathcal R_{p_s^\star}(\theta^s_\star) & \leq \mathcal R_{p_s^\star}(\bar\theta) = \mathbb E_{\xi\sim p_s^\star}\lim_{N_s\rightarrow\infty} h(\theta_N^s,\xi) = \lim_{N_s\rightarrow\infty} \mathcal R_{p_s^\star}(\theta_N^s) =\mathcal R_{p_s^\star}(\theta^s_\star)
\end{align*}
almost surely, where the first equality follows from the continuity of $\theta\mapsto h(\theta,\xi)$ and the second one from the Dominated Convergence Theorem. Then, $\mathcal R_{p_s^\star}(\bar\theta)=\mathcal R_{p_s^\star}(\theta^s_\star)$ almost surely, proving the result.

\paragraph{Proof of Theorem~\ref{thm:DORO_criterion_convergence}.} First, we note that any sample $\xi\in A^{\textnormal{tr}}_\varepsilon$ has cardinality $n_\varepsilon:=(1-\varepsilon)n$. We have the following inequalities
\begin{align*}
    \phi_n(\mathcal R_{p_\star}(\hat\theta_\varepsilon)) - \phi_n(\mathcal R_{p_\star}(\theta_\star)) & = \phi_n(\mathcal R_{p_\star}(\hat\theta_\varepsilon)) - V_{\boldsymbol \xi^{\textnormal{tr}}}^{\varepsilon}(\hat\theta_\varepsilon) + V_{\boldsymbol \xi^{\textnormal{tr}}}^{\varepsilon}(\hat\theta_\varepsilon) - \phi_n(\mathcal R_{p_\star}(\theta_\star)) \\
      & \leq \phi_n(\mathcal R_{p_\star}(\hat\theta_\varepsilon)) - V_{\boldsymbol \xi^{\textnormal{tr}}}^{\varepsilon}(\hat\theta_\varepsilon) + V_{\boldsymbol \xi^{\textnormal{tr}}}^{\varepsilon}(\theta_\star) - \phi_n(\mathcal R_{p_\star}(\theta_\star)) \\
      & \leq 2\sup_{\theta\in\Theta}\Big\lvert V_{\boldsymbol \xi^{\textnormal{tr}}}^{\varepsilon}(\theta) - \phi_n(\mathcal R_{p_\star}(\theta)) \Big\rvert.
\end{align*}
Then, using a second order Taylor expansion of $\phi_n(\mathcal R_p(\theta))$ around $\phi_n(\mathcal R_{p_\star}(\theta))$, we get
\begin{align*}
    & \sup_{\theta\in\Theta}\Big\lvert V_{\boldsymbol \xi^{\textnormal{tr}}}^{\varepsilon}(\theta) - \phi_n(\mathcal R_{p_\star}(\theta)) \Big\rvert \\
    & = \sup_{\theta\in\Theta}\Big\lvert \inf_{\boldsymbol{\xi}\in A^{\textnormal{tr}}_\varepsilon}\int\phi_n(\mathcal R_p(\theta)) Q_{\boldsymbol{\xi}}(\mathrm dp) - \phi_n(\mathcal R_{p_\star}(\theta))\Big\rvert \\
    & = \sup_{\theta\in\Theta}\Big\lvert \inf_{\boldsymbol{\xi}\in A^{\textnormal{tr}}_\varepsilon} \Big\{\phi_n'(\mathcal R_{p_\star}(\theta)) \int[\mathcal R_p(\theta) - \mathcal R_{p_\star}(\theta)]Q_{\boldsymbol{\xi}}(\mathrm dp) + \int\frac{\phi_n''(\tilde r_p(\theta))}{2}[\mathcal R_p(\theta) - \tilde r_p(\theta)]^2Q_{\boldsymbol{\xi}}(\mathrm dp)\Big\} \Big\rvert \\
    & \leq \sup_{\theta\in\Theta}\Big\lvert \inf_{\boldsymbol{\xi}\in A^{\textnormal{tr}}_\varepsilon}\Big\{\max_{t\in[0,K]}\phi_n'(t)\Big[\frac{\alpha}{\alpha + n_\varepsilon}(\mathcal R_{p_0}(\theta) - \mathcal R_{p_\star}(\theta)) + \frac{n_\varepsilon}{\alpha + n_\varepsilon}(\mathcal R_{p_{\boldsymbol{\xi}}}(\theta) - \mathcal R_{p_\star}(\theta))\Big] \\
    & + \frac{K}{2}\max_{t\in[0,K]}\phi_n''(t)\Big\}\Big\rvert \\
    & = o(1) + O(1)\cdot\sup_{\theta\in\Theta}\Big\lvert \inf_{\boldsymbol{\xi}\in A^{\textnormal{tr}}_\varepsilon} \mathcal R_{p_{\boldsymbol{\xi}}}(\theta) - \mathcal R_{p_\star}(\theta)\Big\rvert.
\end{align*}

\subsection{Further Results and Proofs}

With the following results, we ensure finite-sample and asymptotic guarantees on the closeness of optimization procedures based on the SB approximation $\hat V_{\boldsymbol\xi^n}(\theta, T, M)$ of the DP-based criterion (as defined in Equation~\eqref{eq:SB_DP_criterion} and corresponding to the case $S=1$ and $\alpha_0 = \infty$ in the HDP framework) versus the infinite-dimensional target $V_{\boldsymbol\xi^n}$. We recall that
\begin{equation*}
    \hat V_{\boldsymbol\xi^n}(\theta, T, M) = \frac{1}{M} \sum_{m=1}^M \phi\left(\sum_{j=0}^T p_j^m h(\theta, \xi_{mj})\right),
\end{equation*}
where the weights $p_1^m,\dots, p_T^m, p_0^m$ are generated according to the truncated SB Algorithm~\ref{alg:stickbreaking}.

\begin{theorem}\label{thm:SB_bounds}
    Assume $\Theta$ is a bounded subset of $\mathbb R^d$ and, for all $\xi\in \Xi$,  $\theta\mapsto h(\theta,\xi)$ is $c(\xi)$-Lipschitz continuous. Then, for all $T,M\in\mathbb N$ and $\varepsilon >0$,
    \begin{align*}
        \sup_{\theta \in \Theta} \big\vert \hat V_{\boldsymbol\xi^n}(\theta, T, M) & - V_{\boldsymbol\xi^n}(\theta)] \big\vert \leq M_\phi K\times\Bigg[\frac{\alpha + n}{\alpha + n + 1} \Bigg]^T + \varepsilon
    \end{align*}
    with probability at least
    \begin{align*}
        & 1 -2\left(\frac{32 M_\phi C_T\textnormal{diam}(\Theta)\sqrt{d}}{\varepsilon}\right)^d  \times \left[ \exp\left\{-\frac{3M\varepsilon^2}{4\phi(K)(6\phi(K)+\varepsilon)}\right\} + \exp\left\{-\frac{3M\varepsilon}{40\phi(K)}\right\} \right]
    \end{align*}
    for some constant $C_T>0$.
\end{theorem}
We divide the proof of Theorem~\ref{thm:SB_bounds} in two simpler Lemmas as follows.

\begin{lemma}\label{lem:sublemma_1}
    For all $T, M\in\mathbb N$,
    \begin{equation} \label{error_bound_prop_equation}
        \sup_{\theta\in\Theta} \vert \hat V_{\boldsymbol \xi^n}(\theta, T, M) - V_{\boldsymbol\xi^n}(\theta)\vert \leq M_\phi K\Bigg(\frac{\alpha + n}{\alpha + n + 1} \Bigg)^T+ \sup_{\theta \in \Theta} \Big\vert \hat V_{\boldsymbol\xi^n}(\theta, T, M) - \mathbb E[\hat V_{\boldsymbol\xi^n}(\theta, T, 1)] \Big\vert.
    \end{equation}
\end{lemma}

\begin{proof}
    We have
    \begin{align}\label{error_bound_eq}
        \sup_{\theta\in\Theta} & \vert \hat V_{\boldsymbol \xi^n}(\theta, T, M) - V_{\boldsymbol\xi^n}(\theta)\vert \nonumber \\
                               & \leq \sup_{\theta \in \Theta} \Big\vert \hat V_{\boldsymbol\xi^n}(\theta, T, M) - \mathbb E[\hat V_{\boldsymbol\xi^n}(\theta, T, 1)] \Big\vert + \sup_{\theta \in \Theta} \Big\vert \mathbb E[\hat V_{\boldsymbol\xi^n}(\theta, T, 1)] - V_{\boldsymbol\xi^n}(\theta)\Big\vert.
    \end{align}
    Note that
    \begin{align*}
        \mathbb E[\hat V_{\boldsymbol\xi^n}(\theta, T, 1)] & = \mathbb E_{p_1,\xi_1,\dots,p_T,\xi_T}\Bigg[\phi\Bigg( \sum_{j=1}^{T}p_{j} h(\theta,\xi_{j}) + p_{0}h(\theta,\xi_{0})\Bigg)\Bigg] \\
                                                      & = \mathbb E_{\sum_{j\geq 1} p_j\delta_{\xi_j}\sim Q_{\boldsymbol\xi^n}}\Bigg[\phi\Bigg( \sum_{j=1}^{T}p_{j} h(\theta,\xi_{j}) + p_{0}h(\theta,\xi_{0})\Bigg)\Bigg] \\
                                                      & = \mathbb E_{\sum_{j\geq 1} p_j\delta_{\xi_j}\sim Q_{\boldsymbol\xi^n}}\Bigg[\phi\Bigg( \sum_{j=1}^{\infty}p_{j} h(\theta,\xi_{j}) + p_{0}h(\theta,\xi_{0}) - \sum_{j=T+1}^{\infty}p_{j} h(\theta,\xi_{j})\Bigg)\Bigg] \\
                                                      & = V_{\boldsymbol\xi^n}(\theta) + \mathbb E_{\sum_{j\geq 1} p_j\delta_{\xi_j}\sim Q_{\boldsymbol\xi^n}}\Bigg[\phi'\big(c_{\theta,\sum_{j\geq 1} p_j\delta_{\xi_j}}\big)p_0\Bigg\{ h(\theta,\xi_{0}) - \sum_{j=T+1}^{\infty}\frac{p_{j}}{p_0} h(\theta,\xi_{j})\Bigg\} \Bigg],
    \end{align*}
    where the last equality follows from the mean value theorem applied to endpoints $\sum_{j=1}^{\infty}p_{j} h(\theta,\xi_{j})$ and $\sum_{j=1}^{\infty}p_{j} h(\theta,\xi_{j}) + p_{0}h(\theta,\xi_{0}) - \sum_{j=T+1}^{\infty}p_{j} h(\theta,\xi_{j})$. Then the second term in (\ref{error_bound_eq}) is bounded by
    \begin{equation*}
        M_\phi K  \mathbb E_{\sum_{j\geq 1} p_j\delta_{\xi_j}\sim Q_{\boldsymbol\xi^n}}[p_0]=M_\phi K  \mathbb E\Bigg[\prod_{k=1}^T (1-B_k)\Bigg]=M_\phi K\Bigg(\frac{\alpha + n}{\alpha + n + 1} \Bigg)^T.
    \end{equation*}
\end{proof}
The second term on the left-hand side of Equation~\eqref{error_bound_prop_equation} is instead of the form
\begin{equation}\label{eq:sup_distance_stickbreaking}
    \sup_{g\in\mathscr G}\bigg\vert \frac{1}{M}\sum_{i=1}^M g(X_i) - \mathbb E[g(X_1)]\bigg \vert,
\end{equation}
where $X_i=\sum_{j=0}^{T}p_{ij} h(\theta,\xi_{ij})$ are iid random variables whose distribution is determined by the truncated stick-breaking procedure.

The aim of the next Lemma is to provide sufficient conditions for finite sample bounds and asymptotic convergence to 0 of the term in (\ref{eq:sup_distance_stickbreaking}). Specifically, we impose complexity constraints on the function class $\mathscr H :=\{\xi\mapsto h(\theta, \xi):\theta\in\Theta\}$ which allow us to obtain appropriate conditions on the derived class
\begin{equation*}
    \mathscr F:= \Bigg\{(p_{j},\xi_{j})_{j=0}^{T} \mapsto \phi\Bigg( \sum_{j=1}^{T}p_{j} h(\theta,\xi_{j}) + p_{0}h(\theta,\xi_{0})\Bigg) : \theta\in\Theta \Bigg\},
\end{equation*}
ensuring the non-asymptotic results we seek.

\begin{lemma}\label{lem:approx_second_term}
    Assume $\Theta$ is a bounded subset of $\mathbb R^d$ and $\theta\mapsto h(\theta,\xi)$ is $c(\xi)$-Lipschitz continuous for all $\xi\in \Xi$. Then, for all $T,M\geq 1$ and $\varepsilon >0$,
    \begin{equation*}
        \sup_{\theta \in \Theta} \Big\vert \hat V_{\boldsymbol\xi^n}(\theta, T, M) - \mathbb E[\hat V_{\boldsymbol\xi^n}(\theta, T, 1)] \Big\vert \leq \varepsilon
    \end{equation*}
    with probability at least
    \begin{equation*}
        1-2\left(\frac{32 M_\phi C_T\textnormal{diam}(\Theta)\sqrt{d}}{\varepsilon}\right)^d\left[ \exp\left\{-\frac{3M\varepsilon^2}{4\phi(K)(6\phi(K)+\varepsilon)}\right\} + \exp\left\{-\frac{3M\varepsilon}{40\phi(K)}\right\} \right].
    \end{equation*}
    for some constant $C_T>0$.
\end{lemma}

\begin{proof}
    By the Lipschitz continuity assumption on $h(\theta,\xi)$, we obtain that, for all $(p_{j},\xi_{j})_{j=0}^{T}$, $\theta\mapsto \phi\left( \sum_{j=1}^{T}p_{j} h(\theta,\xi_{j}) + p_{0}h(\theta,\xi_{0})\right)$ is $M_\phi \tilde c\big((p_{j},\xi_{j})_{j=0}^{T}\big)$-Lipschitz continuous, with
    \begin{equation*}
        \tilde c\big((p_{j},\xi_{j})_{j=0}^{T}\big) := \sum_{j=0}^T p_jc(\xi_j).
    \end{equation*}
    Indeed, for all $\theta_1,\theta_2\in\Theta$,
    \begin{align*}
        \Bigg\vert & \phi \Bigg( \sum_{j=1}^{T}p_{j} h(\theta_1,\xi_{j}) + p_{0}h(\theta_1,\xi_{0})\Bigg) - \phi\Bigg( \sum_{j=1}^{T}p_{j} h(\theta_2,\xi_{j}) + p_{0}h(\theta_2,\xi_{0})\Bigg)\Bigg\vert \\
                   & \leq M_\phi \sum_{j=1}^{T}p_{j} \vert h(\theta_1,\xi_{j}) - h(\theta_2,\xi_{j})\vert + p_{0}\vert h(\theta_1,\xi_{0}) - h(\theta_2,\xi_{0})\vert \\
                   & \leq M_\phi \tilde c\big((p_{j},\xi_{j})_{j=0}^{T}\big)\Vert\theta_1 - \theta_2\Vert.
    \end{align*}
    Therefore, denoting by $P$ the law of the vector $(p_{j},\xi_{j})_{j=0}^{T}$ and by $N_{[]}(\varepsilon, \mathscr F, \mathcal L^1(P))$ the associated $\varepsilon$-bracketing number of the class $\mathscr F$, by Lemma 7.88 in \cite{wasserman2010concentration} we obtain
    \begin{equation*}
        N_{[]}(\varepsilon, \mathscr F, \mathcal L^1(P)) \leq \left(\frac{4 M_\phi C\textnormal{diam}(\Theta)\sqrt{d}}{\varepsilon}\right)^d,
    \end{equation*}
    with $C_T:=\int \tilde c\big((p_{j},\xi_{j})_{j=0}^{T}\big)\mathrm dP$. Then the result follows by Theorem 7.86 in \cite{wasserman2010concentration} after noticing that $\sup_{f\in\mathscr F}\Vert f\Vert_{\mathcal L^1(P)} \leq\sup_{f\in\mathscr F}\Vert f\Vert_\infty \leq \phi(K)$. This ends the proof of the Lemma as well as of Theorem~\ref{thm:SB_bounds}.
\end{proof}
Heuristically, the bound in Theorem \ref{thm:SB_bounds} is obtained by decomposing the left-hand side of the inequality into a first term depending on the truncation error induced by the threshold $T$, and a second term reflecting the Monte Carlo error related to $M$. Moreover, analogously to Lemma \ref{lem:finite_sample_bounds}, Lemma \ref{thm:SB_bounds} easily implies finite-sample bounds on the excess ``robust risk'' $V_{\boldsymbol\xi^n}(\hat \theta_n(T, M)) - V_{\boldsymbol\xi^n}(\theta_n)$, where $\hat \theta_n(T, M)\in\arg\min_{\theta\in\Theta} \hat V_{\boldsymbol\xi^n}(\theta, T, M)$. Another consequence is the following asymptotic convergence Theorem, whose proof is analogous to that of Theorem \ref{thm:uniform_convergence_criterion}.
\begin{theorem}\label{thm:uniform_convergence_SB}
    Under the same assumptions of Lemma \ref{thm:SB_bounds}, and if $\sup_{T\geq 1}C_T<\infty$,
    \begin{equation*}
        \lim_{T,M\to\infty}\sup_{\theta \in \Theta} \big\vert \hat V_{\boldsymbol\xi^n}(\theta, T, M) - V_{\boldsymbol\xi^n}(\theta)] \big\vert = 0
    \end{equation*}
    almost surely. Also, almost surely
    \begin{align*}
        \lim_{T,M\to\infty} \hat V_{\boldsymbol\xi^n}(\hat\theta_n(T,M), T, M) = V_{\boldsymbol\xi^n}(\theta_n), \qquad \lim_{T,M\to\infty} V_{\boldsymbol\xi^n}(\hat \theta_n(T, M)) = V_{\boldsymbol\xi^n}(\theta_n).
    \end{align*}
\end{theorem}
In words, Theorem \ref{thm:uniform_convergence_SB} ensures that, as the truncation and MC approximation errors vanish, the optimal approximate criterion value converges to the optimal exact one, and that the exact criterion value at any approximate optimizer converges to the exact optimal value. Also note that Theorems \ref{thm:uniform_convergence_criterion} and \ref{thm:uniform_convergence_SB}, when combined, provide guarantees on the convergence of $\hat V_{\boldsymbol{\xi}^n}(\hat\theta_n(T,M), T, M)$ (the empirical criterion one has optimized in practice) to $\mathcal R_{p_\star}(\theta_\star)$ (the theoretical optimal target) as the sample size increases and the DP approximation improves.

Finally, as a byproduct of Theorem \ref{thm:uniform_convergence_SB}, convergence of any approximate robust optimizer to an exact one is established as follows.

\begin{theorem}\label{thm:optimizer_convergence_SB}
    Let $\theta\mapsto h(\theta,\xi)$ be continuous for all $\xi\in\Xi$. Moreover, assume
    \begin{equation*}
        \lim_{T,M\to\infty}V_{\boldsymbol\xi^n}(\hat \theta_n(T, M)) = V_{\boldsymbol\xi^n}(\theta_n)
    \end{equation*}
    almost surely (e.g., as ensured above). Then, almost surely, $\lim_{T,M\to\infty}\hat \theta_n(T, M) = \bar\theta_n$ implies $V_{\boldsymbol\xi^n}(\bar\theta_n)=V_{\boldsymbol\xi^n}(\theta_n)$.
\end{theorem}

\begin{proof}
We have
\begin{align*}
    V_{\boldsymbol\xi^n}(\theta_n) & \leq V_{\boldsymbol\xi^n}(\bar\theta_n) \\
                           & =\mathbb E_{p\sim Q_{\boldsymbol\xi^n}}\left[\lim_{M\to\infty}\lim_{T\to\infty}\phi(\mathcal R_p(\hat \theta_n(T, M)))\right] \\
                           & = \lim_{M\to\infty}\lim_{T\to\infty}V_{\boldsymbol\xi^n}(\hat \theta_n(T, M)) \\
                           & = V_{\boldsymbol\xi^n}(\theta_n)
\end{align*}
almost surely, where the first two equalities follow from the continuity of $\theta\mapsto h(\theta,\xi)$ and $\phi$ as well as from an iterated application of the Dominated Convergence Theorem (recall that $h(\theta,\xi)\in[0,K]$ by assumption for all $\theta$ and $\xi$). This implies $V_{\boldsymbol\xi^n}(\theta_n) = V_{\boldsymbol\xi^n}(\bar\theta_n)$ almost surely.
\end{proof}



\section{Further background and Algorithms}\label{app:background}

In this Appendix, we give further theoretical background on some topics from Bayesian nonparametrics that we touch upon throughout the paper, and we formally spell out the Algorithms heuristically proposed in the main text.

\subsection{The HDP Prior And Its CRF Construction}

The hierarchical Dirichlet process \citep{teh2004sharing, teh2006hierarchical} serves as a prior on a vector of dependent probability measures $(p_1, \dots, p_S)$, and is specified as follows:
\begin{align*}
    \xi_{sj} \mid (p_1, \dots, p_S) & \stackrel{\text{ind}}{\sim} p_s, \qquad s=1,\dots,S, \, j=1,\dots,N_j \\
    p_s \mid p_0     & \overset{\text{iid}}{\sim} \text{DP}(\alpha_s, p_0), \qquad s=1,\dots,S, \\
    p_0              & \sim \text{DP}(\alpha_0, H).
\end{align*}
where $H$ is a continuous distribution on $(\Xi, \mathscr B(\Xi))$ and $\text{DP}(\alpha, P)$ denotes the distribution of a Dirichlet process (DP) with concentration parameter $\alpha>0$ and centering distribution $P$. This construction implies that the observations $\xi_{sj}$ are \emph{partially exchangeable}: exchangeability (i.e., distributional invariance under finite index permutations) holds within each group $s$, but not necessarily across different groups $s\neq s'$, thus allowing for (partial) heterogeneity.

In order to derive Equation~\eqref{eq:ambiguity_neutral_criterion_HDP} in the main text, we need a characterization of the predictive distribution for group $s$, that is, $\mathbb E[p_s\mid \boldsymbol\xi_1, \dots,\boldsymbol\xi_S]$. This is possible by leveraging the Chinese restaurant franchise construction of \cite{teh2004sharing}. The metaphor goes as follows: A franchise of $S$ restaurants shares dishes (unique values) drawn from a franchise-wide menu $p_0$, which is a weighted collection of dishes drawn from a DP with base measure $H$ (the latter can be thought of as an infinitely rich source of recipes). Each restaurant has infinite capacity, meaning that it contains an infinite number of tables, each able to host an infinite number of customers---the only restriction is that customers seating at the same table will share the same dish. Now fix a restaurant $s$ and assume we are given the configuration of the $N_s$ customers $(\xi_{s1},\dots,\xi_{s N_s})$ into $T_s$ tables: Table $\theta_{s1}$ seats $t_{s1}$ customers, table $\theta_{s2}$ seats $t_{s2}$ customers, etc., with the obvious constraint $\sum_{j=1}^{T_s} t_{sj} = N_s$. Each table corresponds to an iid draw $\theta_{sj}$ from the franchise-level menu $p_0$, which we hold fixed for now. Then, because the HDP model places a $\text{DP}(\alpha_s, p_0)$ at the level of restaurant $s$, the Chinese restaurant construction of the DP \citep{blackwell1973ferguson} implies that the next customer (observation) of restaurant $s$ will be seated to table $\theta_{sj}$ with probability proportional to $t_{sj}$, or to a yet unoccupied table with probability proportional to $\alpha_s$. In formulas,\footnote{To keep the notation parsimonious, we identify the customer label $\phi$ with the table $\theta$ at which they sit.}
\begin{equation*}
    \xi_{s N_{s}+1} \mid \boldsymbol{t}_s, \boldsymbol{\theta}_s, p_0 \sim \sum_{j=1}^{T_s}\frac{t_{sj}}{N_{s} + \alpha_s} \delta_{\theta_{sj}} + \frac{\alpha_s}{N_s + \alpha_s}\delta_{\theta_{\textnormal{new}}},
\end{equation*}
with $\theta_{\textnormal{new}}\sim p_0 \mid \boldsymbol{t}_s, \boldsymbol{\theta}_s$. This procedure takes care of partitioning customers into tables within each restaurant. Notice that, because different tables can be assigned the same dish ($\Xi$ value), the table configuration is only latent and instrumental to describe the predictive structure of the HDP.

Given the customer-table configurations of all restaurants, assume there are $K$ distinct dishes being served in the whole franchise. That is, the tables $\theta_{sj}$ only feature $K$ unique values $\xi_1^\star, \dots,\xi_K^\star$, with $m_k$ tables serving dish $\xi_k^\star$ and, clearly, $\sum_{k=1}^K m_k = \sum_{s=1}^S T_s$. Then, because the HDP model places a $\text{DP}(\alpha_0, H)$ prior on $p_0$, the Chinese restaurant process predictive construction of the DP implies\footnote{Again for parsimony of notation, we identify each table with the dish served at it.}
\begin{equation*}
    \theta_{\textnormal{new}} \sim \sum_{k=1}^K \frac{m_k}{\sum_{\ell=1}^K m_\ell + \alpha_0} \delta_{\xi_k^\star} + \frac{\alpha_0}{\sum_{\ell=1}^K m_\ell + \alpha_0} H.
\end{equation*}

See Figure \ref{fig:CRF} for a graphical illustration of this construction. Now recall our assumptions that the data is generated from a continuous distribution, implying that there are no ties among observations: $\xi_{sj}\neq \xi_{s'j'}$ for all $s,s'=1,\dots,S$, $j=1,\dots,N_s$, and $j'=1,\dots,N_{s'}$. This immediately implies that the only consistent table configuration in the Chinese restaurant franchise metaphor is the one in which all customers seat at a different table, each eating a different dish. In turn, this implies the predictive
\begin{equation*}
    \mathbb E[p_s\mid \boldsymbol\xi_1, \dots,\boldsymbol\xi_S] = \frac{N_s}{\alpha_s + N_s}\frac{1}{N_s}\sum_{j=1}^{N_s}\delta_{\xi_{sj}}+  \frac{\alpha_s}{\alpha_s + N_s}\Bigg[ \frac{N}{\alpha_0 + N}\frac{1}{N}\sum_{\ell=1}^S\sum_{j=1}^{N_\ell}\delta_{\xi_{\ell j}} + \frac{\alpha_0}{\alpha_0 + N} H\Bigg],
\end{equation*}
yielding Equation~\eqref{eq:ambiguity_neutral_criterion_HDP} in the main text. This observation on the simplification of the table configuration in our continuous setting also yields the posterior characterization of Proposition 4.1. In fact, \cite{camerlenghi2019distribution} provided a two-stage posterior characterization of the HDP that relies on the same type of latent table configuration appearing in the Chinese restaurant franchise. However, the no-ties assumptions in our setting makes it possible to simplify the characterization as in Proposition~\ref{pro:HDP_post_characterization}.

\begin{figure}[th]
    \centering
    \begin{tikzpicture}
  \node[circle, draw, fill=gray!20, minimum size=3em, inner sep=0pt] (xi11) at (0,0) {{ $\theta_{11}=\xi^\star_{1}$}};
  \node[circle, draw, fill=gray!20, minimum size=3em, inner sep=0pt, right=1.5cm of xi11] (xi12) {{$\theta_{12}=\xi^\star_{2}$}};
  \node[circle, draw, fill=gray!20, minimum size=3em, inner sep=0pt, right=1.5cm of xi12] (xi13) {{$\theta_{13}=\xi^\star_{1}$}};
  \node[circle, draw, fill=gray!20, minimum size=3em, inner sep=0pt, right=1.75cm of xi13] (xi14) {};
  \node[right=0.5cm of xi14] (dots1) {{\Large$\cdots$}};

  \node[above left=0cm and 0cm of xi11] (phi11) {$\xi_{11}$};
  \node[above right=0cm and 0cm of xi11] (phi13) {$\xi_{13}$};
  \node[below left=0cm and 0cm of xi11] (phi14) {$\xi_{14}$};
  \node[below right=0cm and 0cm of xi11] (phi18) {$\xi_{18}$};

  \node[above left=0cm and 0cm of xi12] (phi12) {$\xi_{12}$};
  \node[above right=0cm and 0cm of xi12] (phi15) {$\xi_{15}$};
  \node[below left=0cm and 0cm of xi12] (phi16) {$\xi_{16}$};

  \node[above left=0cm and 0cm of xi13] (phi17) {$\xi_{17}$};

  \node[circle, draw, fill=gray!20, minimum size=3em, inner sep=0pt, below=2cm of xi11] (xi21) {{ $\theta_{21}=\xi^\star_{1}$}};
  \node[circle, draw, fill=gray!20, minimum size=3em, inner sep=0pt, right=1.5cm of xi21] (xi22) {{$\theta_{22}=\xi^\star_{3}$}};
  \node[circle, draw, fill=gray!20, minimum size=3em, inner sep=0pt, right=1.5cm of xi22] (xi23) {{$\theta_{23}=\xi^\star_{2}$}};
  \node[circle, draw, fill=gray!20, minimum size=3em, inner sep=0pt, right=1.5cm of xi23] (xi24) {{$\theta_{24}=\xi^\star_{2}$}};
  \node[right=0.5cm of xi24] (dots2) {{\Large $\cdots$}};

  \node[above left=0cm and 0cm of xi21] (phi21) {$\xi_{21}$};
  \node[below left=0cm and 0cm of xi21] (phi22) {$\xi_{22}$};

  \node[above left=0cm and 0cm of xi22] (phi23) {$\xi_{23}$};
  \node[below left=0cm and 0cm of xi22] (phi24) {$\xi_{24}$};
  \node[above right=0cm and 0cm of xi22] (phi26) {$\xi_{26}$};
  
  \node[above left=0cm and 0cm of xi23] (phi25) {$\xi_{25}$};

  \node[above left=0cm and 0cm of xi24] (phi27) {$\xi_{27}$};
  \node[below left=0cm and 0cm of xi24] (phi28) {$\xi_{28}$};

  \node[circle, draw, fill=gray!20, minimum size=3em, inner sep=0pt, below=2cm of xi21] (xi31) {{$\theta_{31} =\xi^\star_{2}$}};
  \node[circle, draw, fill=gray!20, minimum size=3em, inner sep=0pt, right=1.5cm of xi31] (xi32) {{$\theta_{32}=\xi^\star_{4}$}};
  \node[circle, draw, fill=gray!20, minimum size=3em, inner sep=0pt, right=1.5cm of xi32] (xi33) {};
  \node[circle, draw, fill=gray!20, minimum size=3em, inner sep=0pt, right=2.05cm of xi33] (xi34) {};
  \node[right=0.5cm of xi34] (dots3) {{\Large $\cdots$}};

  \node[above left=0cm and 0cm of xi31] (phi31) {$\xi_{31}$};
  \node[below left=0cm and 0cm of xi31] (phi32) {$\xi_{32}$};
  \node[above right=0cm and 0cm of xi31] (phi35) {$\xi_{35}$};
  \node[below right=0cm and 0cm of xi31] (phi36) {$\xi_{36}$};

  \node[above left=0cm and 0cm of xi32] (phi33) {$\xi_{33}$};
  \node[below left=0cm and 0cm of xi32] (phi34) {$\xi_{34}$};

  \begin{scope}
    \node[draw=black, rounded corners, inner sep=8mm, fit=(xi11) (dots1)] {};
    \node[draw=black, rounded corners, inner sep=8mm, fit=(xi21) (dots2)] {};
    \node[draw=black, rounded corners, inner sep=8mm, fit=(xi31) (dots3)] {};
  \end{scope}

\end{tikzpicture}
\caption{Illustration of the Chinese restaurant franchise construction of the HDP prior. In this example, there are $S=3$ restaurants (represented by the rectangles) each hosting, at the current stage of the generative process, respectively 8, 8, and 6 customers. The restaurants seat their customers at 3, 4, and 2 tables, respectively, and a total number of $K=4$ dishes $\xi^\star_{1}, \dots,\xi^\star_{4}$ is served in the whole franchise.}
\label{fig:CRF}
\end{figure}

\paragraph{Monte Carlo Approximation for the HDP Robust Criterion.} Algorithm \ref{alg:approx_criterion} summarizes the simulation strategy for the HDP robust criterion outlined in Section 4 of the main text: Given (i) the posterior characterization of the HDP in the case with no ties among observations (see Proposition Proposition~\ref{pro:HDP_post_characterization}) and (ii) a method to simulate from the DP (see the next two paragraphs for examples of such methods), one can repeatedly (a) simulate from the posterior of the top-level distribution $p_0$, and (b) given the realization of $p_0$, simulate from the posterior of the group-level distributions $p_s$. Finally, each group-specific criterion is approximated as a Monte Carlo average of the risks computed with respect to the simulated group-level distributions. We note that the algorithm scales linearly both in the number of Monte Carlo samples $M$, and is fully parallelizable along this dimension. It is also parallelizable across groups $s=1,\dots,S$.

\begin{algorithm}[ht]
   \caption{Monte Carlo Approximate HDP Criterion for Group $s$ ($\text{HDP-MC}_s$)}
   \label{alg:approx_criterion}
\begin{algorithmic}
   \STATE {\bfseries Input:} Data $\boldsymbol\xi_1, \dots, \boldsymbol\xi_d$, loss function $h$, function $\phi$, concentration parameters $\alpha_s, \alpha_0$, top-level centering probability $H$, approximation type $\text{AP} \in\{\text{SB, MD}\}$, truncation criteria $T_s\, (s\in\{1,\dots,S\})$ and $T_0$, number of MC samples $M$
   \FOR{$m=1$ {\bfseries to} $M$}
   \STATE $(p_{0j}^m, \xi_{0j}^m)_{j=0}^{T_0} = \text{AP}\big(\alpha_0 + N, \frac{\alpha_0}{\alpha_0 + N}H + \frac{N}{\alpha_0 + N}\frac{1}{N}\sum_{\ell=1}^S\sum_{j=1}^{N_\ell}\delta_{\xi_{\ell j}}, T_0\big)$
   \STATE $\hat p_0^m = \sum_{j=0}^{T_0}p_{0j}^m \delta_{\xi_{0j}^m}$
   \STATE $(p_{sj}^m, \xi_{sj}^m)_{j=0}^{T_i} = \text{AP}\big(\alpha_s + N_s, \frac{\alpha_s}{\alpha_s + N_s}\hat p_0^m + \frac{N_s}{\alpha_s + N_s}\frac{1}{N_s}\sum_{j=1}^{N_s}\delta_{\xi_{sj}},T_s\big)$
   \ENDFOR
   \STATE {\bfseries Return: $\theta \mapsto M^{-1}\sum_{m=1}^M \phi\big( \sum_{j=0}^{T_s}p_{sj}^m h(\theta,\xi_{sj}^m)\big)$}
\end{algorithmic}
\end{algorithm}

\paragraph{Truncated Stick-Breaking Approximation of the Dirichlet Process.}
 Algorithm \ref{alg:stickbreaking} presents a truncated version of the stick-breaking procedure described in Section~\ref{sec:MC_approx} in the main text, stopping the theoretically infinite procedure at step $T$. The remaining portion of the stick is then allocated to one further atom drawn from the centering measure $P$. We note that the complexity of the algorithm is linear in the truncation threshold $T$.

\begin{algorithm}[ht]
   \caption{Truncated Dirichlet Process Stick-Breaking Algorithm (SB)}
   \label{alg:stickbreaking}
\begin{algorithmic}
   \STATE {\bfseries Input:} Concentration parameter $\alpha$, centering probability $P$, truncation criterion $T\in\mathbb N$
   \STATE Set $\prod_{k=1}^{0}(1-B_k)\equiv 1$
   \FOR{$j=1$ {\bfseries to} $T$}
   \STATE Draw $\xi_{j}\sim \pi$
   \STATE Draw $B_{j} \sim \textnormal{Beta}(1, \alpha)$
   \STATE Set $p_{j} = B_j\prod_{k=1}^{j-1}(1-B_k)$
   \ENDFOR
   \STATE Draw $\xi_{0}\sim \pi$
   \STATE Set $p_{0} = \prod_{k=1}^{T}(1-B_k)$
   \STATE {\bfseries Return: $(p_j, \xi_j)_{j=0}^T$}
\end{algorithmic}
\end{algorithm}

\paragraph{Multinomial-Dirichlet Construction of the Dirichlet Process.} Another finite approximation of $ p\sim\textnormal{DP}(\alpha,P)$ is $p_T = \sum_{j=1}^T p_j\delta_{x_j}$, with $x_j\overset{\textnormal{iid}}{\sim} P$ and
$$(p_1,\dots,p_j)\sim\textnormal{Dirichlet}(T;\alpha/T,\dots,\alpha/T).$$
As $T\to\infty$, $p_T$ approaches $p$; see Theorem 4.19 in \cite{ghosal2017fundamentals}. Hence, one can approximately sample from a DP as in Algorithm~\ref{alg:mult_dir}. For all of our experiments, we choose to simulate from the DPs at both levels of the hierarchy of the HDP via the MD approximation. This is because, even for moderate $T$, this method tends to assign more balanced weights than the stick-breaking constructions, making practical optimization of the HDP robust criterion more stable. We also note that the complexity of the algorithm is linear in the truncation of the threshold $T$.

\begin{algorithm}[ht]
   \caption{Truncated Dirichlet Process Multinomial-Dirichlet Algorithm (MD)}
   \label{alg:mult_dir}
\begin{algorithmic}
   \STATE {\bfseries Input:} Concentration parameter $\alpha$, centering probability $P$, truncation criterion $T\in\mathbb N$
   \STATE Initialize $\boldsymbol p = (p_1, \dots, p_T) \in\mathbb R^T$
   \FOR{$j=1$ {\bfseries to} $T$}
   \STATE Sample $\xi_j \sim P$
   \STATE Update $p_j\sim \textnormal{Gamma}(\alpha/T,1)$
   \ENDFOR
   \STATE Normalize $\boldsymbol p = \frac{\boldsymbol p}{\sum_{j=1}^n p_j}$
   \STATE {\bfseries Return: $(p_j, \xi_j)_{j=1}^T$}
\end{algorithmic}
\end{algorithm}

\paragraph{Stochastic Gradient Descent Algorithm.} Algorithm \ref{alg:SGD_modified} spells out the details of the SGD algorithm used to optimize the HDP criterion in practice. As a standard SGD routine, it enjoys all of the computational properties of such optimization methods. We further note that, while an explicit dependence on the group sample sizes is not present, in order to make the MC draws from the HDP predictive representative of such samples, it is necessary to increase the truncation threshold $T$ and/or the number of MC samples $M$ as the sample sizes increase. Another possibility to enhance the computational efficiency of the method, which we do not explore further, is to avoid random sampling and assign to each batch (MC sample) in the HDP approximation process a separate subset of the data, then add a number of observations from the prior centering measure $H$ by respecting the proportions dictated by the concentration parameters and the sample sizes.

\begin{algorithm}[ht]
   \caption{Stochastic Gradient Descent Algorithm (SGD)}
   \label{alg:SGD_modified}
\begin{algorithmic}
   \STATE {\bfseries Input:} Approximate criterion parameters $\{(p_j^m, \xi_j^m):j=1,\dots,T, \, m=1,\dots,M\}$, loss function $h$, function $\phi$ step size schedule $(\eta_t)_{t\geq 1}$, starting value $\theta^0$, number of iterations $I$
   \FOR{$t=1$ {\bfseries to} $I$}
   \STATE Choose $m_t\in\{1,\dots,M\}$
   \STATE Update $\theta^{t} = \theta^{t-1} - \eta_t \phi'\left(\sum_{j=1}^T p_j^{m_t} h(\theta^{t-1},\xi_j^{m_t})\right) \sum_{j=1}^T p_j^{m_t} \nabla_\theta h(\theta^{t-1},\xi_j^{m_t})$
   \ENDFOR
   \STATE {\bfseries Return: $\theta^{I}$}
\end{algorithmic}
\end{algorithm}

\paragraph{Modified SGD for DORO Routine.} Our DORO procedure consists in performing a modified SGD algorithm on the approximated DP-based criterion
\begin{equation*}
    \hat V_{\boldsymbol{\xi}^n}(\theta, T, M) := \frac{1}{M} \sum_{m=1}^M \phi\left(\sum_{j=0}^T p_j^m h(\theta, \xi_{mj})\right).
\end{equation*}
In particular, while the weights $p^m_j$ are drawn according to either Algorithm~\ref{alg:stickbreaking} or Algorithm~\ref{alg:mult_dir}, the atoms $\xi_{mj}$ are drawn iid from the DP centering probability measure $p_0$ with probability $\propto \alpha$ or drawn from $p_{\boldsymbol \xi^{\textnormal{tr}}}$ with probability $\propto n$. Let this fact be recorded in a set of $M$ binary vectors $\boldsymbol \omega_m = (\omega_{m1}, \dots, \omega_{mT})$, where $\omega_{mj}=1$ if $\xi_{mj}\sim p_{\boldsymbol \xi^{\textnormal{tr}}}$  and $\omega_{mj}=0$ otherwise. Based on this, Algorithm \ref{alg:DORO_SGD} spells out a SGD procedure for practical $\varepsilon$-DORO. In essence, at each iteration, the $\varepsilon\times 100\%$ worst fitting observations sampled from the training set are excluded from gradient computations (and probability weights are re-normalized accordingly).

\begin{algorithm}[ht]
   \caption{DORO Stochastic Gradient Descent algorithm}
   \label{alg:DORO_SGD}
\begin{algorithmic}
   \STATE {\bfseries Input:} Approximate criterion parameters $\{(p^m_j, \xi_{mj}, \omega_{mj}):m=1,\dots,M, j=1,\dots,T\}$, contamination level $\varepsilon$, step size schedule $(\eta_t)_{t\geq 0}$, starting value $\theta^0$, number of passes $P$, iteration tracker $t=0$
   \FOR{$p=1$ {\bfseries to} $P$}
   \STATE Initialize $I=\{1,\dots,M\}$
   \FOR{$m'=1$ {\bfseries to} $N$}
   \STATE Sample uniformly $m\in I$
   \STATE Rank $[h(\theta,\xi_{mj}): \omega_{mj} = 1]$ from lowest to highest
   \STATE Let $\tilde{\boldsymbol{\xi}}_m = [\xi_{mj}:\omega_{mj} = 0 \textnormal{ or } (\omega_{mj} = 1 \textnormal{ and } \textnormal{rank}(h(\theta^t,\xi_{mj})) < \varepsilon\cdot\boldsymbol{\omega}_m^\top\boldsymbol{1}_T)]$, retaining the index ordering
   \STATE Let $\tilde{\boldsymbol{p}}_m = [p^m_j:\omega_{mj} = 0 \textnormal{ or } (\omega_{mj} = 1 \textnormal{ and } \textnormal{rank}(h(\theta^t,\xi_{mj})) < \varepsilon\cdot\boldsymbol{\omega}_m^\top\boldsymbol{1}_T)]$, retaining the index ordering
   \STATE Normalize $\tilde{\boldsymbol{p}}_m = (\tilde{\boldsymbol{p}}_m^\top\boldsymbol{1})^{-1}\tilde{\boldsymbol{p}}_m$
   \STATE Update $\theta^{t+1} = \theta^t - \eta_t \cdot \phi'\left(\tilde{\boldsymbol{p}}_m^{\top}h(\theta^t, \tilde{\boldsymbol{\xi}}_m)\right) \cdot \tilde{\boldsymbol{p}}_m^{\top}\nabla_\theta h(\theta^t, \tilde{\boldsymbol{\xi}}_m)$
   \STATE Update $I=I\setminus\{m\}$
   \STATE Update $t = t + 1$
   \ENDFOR
   \ENDFOR
   \STATE {\bfseries Return: $\theta^{PM + 1}$}
\end{algorithmic}
\end{algorithm}

\section{Experiments on Single-Source DP-Based DRO}\label{app:experiments_DP_DRO}

\subsection{High-Dimensional Linear Regression Experiment}\label{app:DP_lin_reg_experiment}

\paragraph{Setting.} In this experiment, we test the performance of our robust criterion in a high-dimensional sparse linear regression task. The high-dimensional and sparse nature of the data-generating process is expected to induce distributional uncertainty, and our method is meant to address this. In this context, we use the quadratic loss function $(\theta, y, x)\mapsto10^{-3}(y-\theta^\top x)^2$, where the $10^{-3}$ factor serves to stabilize numerical values in the optimization process. Notice that, by the form of the ambiguity-neutral criterion (\ref{eq:ambiguity_neutral_criterion}), the multiplicative factor on the loss function does not change the equivalence with Ridge.

\paragraph{Data-Generating Process.} The data for the experiment are generated iid across simulations (200) and observations ($n=100$ per simulation) as follows. For each observation $i=1,\dots,n$, the $d$-dimensional ($d=90$) covariate vector follows a multivariate normal distribution with mean 0 and such that (i) each covariate has unitary variance, and (ii) any pair of distinct covariates has covariance 0.3:
\begin{equation*}
    x_i = \begin{bmatrix}
        x_{i1} \\
        \vdots \\
        x_{id}
    \end{bmatrix} \sim \mathcal N(0, \Sigma), \quad \Sigma = \begin{bmatrix}
        1 & 0.3 & \cdots & 0.3 \\
        0.3 & 1 & \cdots & 0.3\\
        \vdots & \vdots & \ddots & \vdots\\
        0.3 & 0.3 & \cdots & 1
    \end{bmatrix} \in \mathbb R^{d\times d}.
\end{equation*}
Then, the response has conditional distribution $y_i\mid x_i \sim \mathcal N(a^\top x_i, \sigma^2)$, with
$$a = (1, 1, 1, 1, 1, 0, \cdots, 0)^\top\in\mathbb R^d$$
and $\sigma = 0.5$. That is, out of 90 covariates, only the first 5 have a unitary positive marginal effect on $y_i$, and additive Gaussian noise is added to the resulting linear combination. Together with 100 training samples, at each simulation we generate 5000 test samples on which we compute out-of-sample RMSE for the ambiguity-averse, ambiguity-neutral, and OLS procedures.

\paragraph{Robust Criterion Parameters.} For each simulated sample, we run our robust procedure setting the following parameter values: $\phi(t)=\beta\exp(t/\beta)-\beta$, $\beta \in\{1, \infty\}$, $\alpha=a/n$ for $a\in\{1, 2, 5, 10\}$, and $p_0 = \mathcal N(0,I)$, where the $\beta = \infty$ setting corresponds to Ridge regression with regularization parameter $\alpha$ (see Proposition \ref{pro:equivalence_regularization} in the main text). Finally, we run 300 Monte Carlo simulations to approximate the criterion, and truncate the Multinomial-Dirichlet approximation at $T=50$.

\paragraph{Stochastic Gradient Descent Parameters} We initialize the algorithm at $\theta = (0,\dots,0)$ and set the step size at $\eta_t = 50/(100 + \sqrt{t})$. The number of passes over data is set after visual inspection of convergence of the criterion value. The run time per SGD run is less than 1 second on our infrastructure (see Appendix \ref{app:computing_resources}).

\subsection{High-Dimensional Logistic Regression Experiment}

\paragraph{Setting.} In this experiment, we test the performance of our robust criterion on a high-dimensional sparse classification task using the framework of logistic regression. As in the linear regression experiment, the high-dimensional and sparse nature of the data-generating process is expected to induce distributional uncertainty, and our method is meant to address this. In this setting, the loss function is $h(\xi, \theta) = \log(1+\exp(-y\cdot x^\top\theta))$. As in the previous experiment, we pre-multiply it by a factor of $10^{-3}$ for numerical stability reasons.

\paragraph{Data-Generating Process.} The data for the experiment are generated iid across simulations (200) and observations ($n=100$ per simulation) as follows. For each observation $i=1,\dots,n$, the $d$-dimensional ($d=90$) covariate vector follows a multivariate normal distribution with mean 0 and such that (i) each covariate has unitary variance, and (ii) any pair of distinct covariates has covariance 0.3:
\begin{equation*}
    x_i = \begin{bmatrix}
        x_{i1} \\
        \vdots \\
        x_{id}
    \end{bmatrix} \sim \mathcal N(0, \Sigma), \quad \Sigma = \begin{bmatrix}
        1 & 0.3 & \cdots & 0.3 \\
        0.3 & 1 & \cdots & 0.3\\
        \vdots & \vdots & \ddots & \vdots\\
        0.3 & 0.3 & \cdots & 1
    \end{bmatrix} \in \mathbb R^{d\times d}.
\end{equation*}
Then, the response has conditional distribution $y_i\mid x_i \sim \textnormal{Binary}(\{1,-1\}, p_x)$, with $p_x =1/(1+\exp(-x^\top a))$ and $a = (1, 1, 1, 1, 1, 0, \cdots, 0)^\top\in\mathbb R^d$. That is, out of 90 covariates, only the first 5 have a unitary positive marginal effect on the log-odds. Together with 100 training samples, at each simulation we generate 5000 test samples on which we compute the out-of-sample average loss for the ambiguity-averse, $L^2$-regularized (with regularization parameter $\alpha$, see below), and un-regularized procedures.

\paragraph{Robust Criterion Parameters.} For each simulated sample, we run our robust procedure setting the following parameter values: $\phi(t)=\beta\exp(t/\beta)-\beta$, $\beta =1$, $\alpha=a/n$ for $\alpha\in\{1, 2, 5, 10\}$, and $p_0 = \textnormal{Binary}(\{1,-1\}, 0.5)\times\mathcal N(0,I)$. Finally, we run 200 Monte Carlo simulations to approximate the criterion, and truncate the Multinomial-Dirichlet approximation at $T=50$.

\paragraph{Stochastic Gradient Descent Parameters} We initialize the algorithm at $\theta = (0,\dots,0)$ and set the step size at $\eta_t = 1000/(100 + \sqrt{t})$. The number of passes over data is set after visual inspection of convergence of the criterion value. The run time per SGD run is 3 seconds on our infrastructure (see Appendix \ref{app:computing_resources}).

\paragraph{Results.}
In Figure \ref{fig:DP_logit_simulations}, we present the results of the simulation study. As for the regression experiment, the ambiguity-averse criterion brings improvement, across $\alpha$ values and compared to the $L^2$-regularized and the unregularized procedures, both in terms of average performance and in terms of the latter's variabiliy (see the first row of the Figure). From the second row of Figure \ref{fig:DP_logit_simulations}, it also emerges that, on average, the ambiguity-averse procedure is more accurate and less variable at estimating the true regression coefficient than the two other methods. Also, our method is able to more effectively shrink the norm of the coefficient vector towards 0 (see the third row). Taken together, these results confirm the theoretical expectation that the ambiguity-averse optimization is effective at hedging against the distributional uncertainty arising in high-dimensional classification problems (in this experimental setting, tackled via logistic regression).

\begin{figure*}[t]
\begin{center}
\centerline{\includegraphics[width=0.9\textwidth]{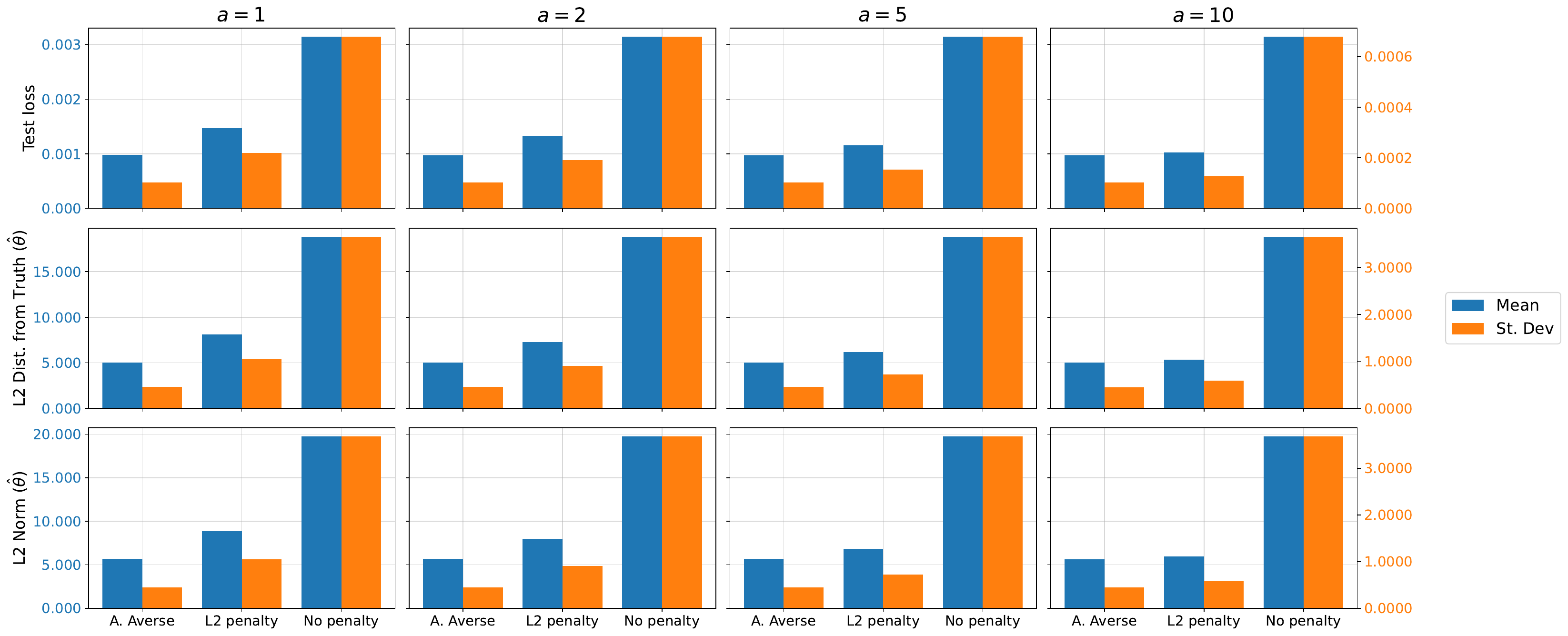}}
\caption{Simulation results for the high-dimensional sparse logistic regression experiment. Bars report the mean and standard deviation (across 200 sample simulations) of the test average loss, $L^2$ distance of estimated coefficient vector $\hat\theta$ from the data-generating one, and the $L^2$ norm of $\hat\theta$. Results are shown for the ambiguity-averse, $L^2$-regularized, and un-regularized procedures. Note: The left (blue) axis refers to mean values, the right (orange) axis to standard deviation values.}
\label{fig:DP_logit_simulations}
\end{center}
\vskip -0.2in
\end{figure*}

\subsection{Pima Indian Diabetes Dataset Experiment}

In this experiment, we use logistic regression for classification on the popular Pima Indians Diabetes dataset,\footnote{Made available by the National Institute of Diabetes and Digestive and Kidney Diseases and downloaded from \url{https://www.kaggle.com/datasets/kandij/diabetes-dataset?resource=download}.} collecting data on 768 women belonging to a Native American group that lives in Mexico and Arizona. The data consists of a binary outcome (whether the subject developed diabetes or not) and 8 features related to her physical condition (these features are standardized before running the analysis).

To test our method, we randomly select 300 training observations and leave out the rest for as a test sample. Then, we randomly split the training data into 15 folds of size 20 and select, via $k$-fold cross validation, the optimal DP concentration parameter $\alpha$ over a wide grid of values. We do the same for the $L^1$-penalty coefficient used to implement regularized logistic regression with the Python library \texttt{scikit-learn} \cite{scikit-learn}. Once the optimal parameters are selected based on out-of-sample risk, we again randomly split the training sample into the same number of folds, and implement our roubust DP method, L1-penalized logistic regression, and unregularized logistic regression on each of the folds.\footnote{All of the implementation details (e.g., parameter values), can be found in our code. This holds for the next two experiments as well.} This splitting procedure allows us (i) to test and compare the performance of our method in a setting with scarce data, where distributional uncertainty is most likely present, and (ii) to asses the sampling variability of the implemented procedures. The run time per SGD run is 19 seconds on our infrastructure (see Appendix \ref{app:computing_resources}).

Table \ref{tab:diabetes} reports the results from the described procedure. Unregularized logistic regression performs quite poorly compared to the other two methods. Instead, the latter yield results in the same orders of magnitude both in terms of average performance and of performance variability, though our DP robust method features almost half of the variability produced by $L^1$-regularized logistic regression.

\begin{table}[ht]
\centering
\begin{tabular}{|l|c|c|c|}
\hline
 & Unregularized & $L^1$ Regularized & \textbf{DP Robust} \\
\hline
 Average & 0.0142 & 0.0007 & 0.0006 \\
\hline
 Standard Deviation & 0.0127 & 6.2253e-05 & 3.9742e-05 \\
\hline
\end{tabular}
\vspace{0.3cm}
\caption{Comparison of average and standard deviation of the out-of-sample performance (out-of-sample expected logistic loss) of the three employed methods for binary classification on the Pima Indian Diabetes dataset.}
\label{tab:diabetes}
\end{table}

\subsection{Wine Quality Dataset Experiment}

In this experiment, we applied linear regression to the popular UCI Machine Learning Repository Wine Quality dataset \citep{misc_wine_quality_186}. Data consists of 4,898 measurements of 11 wines' characteristics and a quality score assigned to each wine. The aim is to predict the latter based on the former (both features and response are standardized before running the analysis). We implement linear regression using our DP-based robust method (with the squared loss function), OLS, and LASSO (the last two methods are implemented using \texttt{scikit-learn} \citep{scikit-learn}).

To test our method, we randomly select 300 training observations and leave out the rest for as a test sample. Then, we randomly split the training data into 10 folds of size 30 and select, via $k$-fold cross validation, the optimal DP concentration parameter $\alpha$ over a wide grid of values. We do the same for the $L^1$-penalty coefficient used to implement LASSO. Once the optimal parameters are selected based on out-of-sample risk, we again randomly split the training sample into the same number of folds, and implement our roubust DP method, LASSO regression, and OLS estimation on each of the folds. This splitting procedure allows us (i) to test and compare the performance of our method in a setting with scarce data, where distributional uncertainty is most likely present, and (ii) to asses the sampling variability of the implemented procedures. The run time per SGD run is 5 seconds on our infrastructure (see Appendix \ref{app:computing_resources}).

Table \ref{tab:wine_quality} reports the results from the described procedure, whose interpretation is very much in line with the results of the previous experiment.

\begin{table}[ht]
\centering
\begin{tabular}{|l|c|c|c|}
\hline
 & Unregularized & $L^1$ Regularized & \textbf{DP Robust} \\
\hline
 Average & 0.0014 & 0.0009 & 0.0009 \\
\hline
 Standard Deviation & 0.0004 & 8.0192e-05 & 6.0076e-05 \\
\hline
\end{tabular}
\vspace{0.3cm}
\caption{Comparison of average and standard deviation of the out-of-sample performance (out-of-sample expected squared loss) of the three employed methods for linear regression on the Wine Quality dataset.}
\label{tab:wine_quality}
\end{table}

\subsection{Liver Disorders Dataset Experiment}

In this experiment, we applied linear regression to the popular UCI Machine Learning Repository Liver Disorders dataset \citep{misc_liver_disorders_60}. Data consists of 345 measurements of 5 blood test results and the number of drinks consumed per day by each subject. The aim is to predict the latter based on the former (both features and response are standardized before running the analysis). We implement linear regression using our DP-based robust method (with the squared loss function), OLS, and LASSO (the last two methods are implemented using \texttt{scikit-learn} \citep{scikit-learn}).

To test our method, we randomly select 200 training observations and leave out the rest for as a test sample. Then, we randomly split the training data into 10 folds of size 20 and select, via $k$-fold cross validation, the optimal DP concentration parameter $\alpha$ over a wide grid of values. We do the same for the $L^1$-penalty coefficient used to implement LASSO. Once the optimal parameters are selected based on out-of-sample risk, we again randomly split the training sample into the same number of folds, and implement our roubust DP method, LASSO regression, and OLS estimation on each of the folds. This splitting procedure allows us (i) to test and compare the performance of our method in a setting with scarce data, where distributional uncertainty is most likely present, and (ii) to asses the sampling variability of the implemented procedures.The run time per SGD run is 15 seconds on our infrastructure (see Appendix \ref{app:computing_resources}).

Table \ref{tab:liver_disorders} reports the results from the described procedure, whose interpretation is very much in line with the results of the previous two experiments.

\begin{table}[ht]
\centering
\begin{tabular}{|l|c|c|c|}
\hline
 & Unregularized & $L^1$ Regularized & \textbf{DP Robust} \\
\hline
 Average & 0.0012 & 0.0009 & 0.0007 \\
\hline
 Standard Deviation & 0.0005 & 0.0001 & 6.6597e-05 \\
\hline
\end{tabular}
\vspace{0.3cm}
\caption{Comparison of average and standard deviation of the out-of-sample performance (out-of-sample expected squared loss) of the three employed methods for linear regression on the Liver Disorders dataset.}
\label{tab:liver_disorders}
\end{table}
\section{Experiments on Multi-Source HDP-Based DRO}\label{app:experiments_HDP_DRO}

\paragraph{High-Dimensional Sparse Linear Regression Simulation Study.} In this experiment, we take $h(\theta,\xi)$ to be the squared loss (as usual in linear regression tasks) and conduct simulations as follows. Denote by $s=1,2$ two distinct yet related samples consisting of $n=100$ observations per sample, where each observation consists of $p=95$ features, collected in a matrix $\boldsymbol X^s\in\mathbb R^{n\times p}$, and a target, collected in a vector $\boldsymbol y^s\in\mathbb R^{n}$. Observations are generated according to the following hierarchical model:
\begin{align*}
    \boldsymbol y^1 \mid \boldsymbol X^1,  \boldsymbol\beta^1 & \sim N(\boldsymbol X^1  \boldsymbol\beta^1, \sigma^2 I),\\
    \boldsymbol y^2 \mid  \boldsymbol X^2, \boldsymbol{\beta}^2 & \sim N(\boldsymbol X^2  \boldsymbol\beta^2, \sigma^2 I), \\
    \boldsymbol{X}^s_{i} & \overset{\text{iid}}{\sim} N(\boldsymbol 0, \boldsymbol\Sigma_{p\times p}), \quad i=1,\dots,n, \, s = 1,2,
\end{align*}
where $\boldsymbol\Sigma_{p\times p}$ takes value 1 on its diagonal and 0.3 off-diagonal, and $\sigma = 0.5$. In order to induce dependence across samples 1 and 2, we generate the coefficient vectors $\boldsymbol{\beta}^1$ and $\boldsymbol{\beta}^2$ as follows. First, to ensure sparsity, we set the last 90 coordinates of both vectors equal to 0. Second, we generate the first 5 coordinates as follows (denote by $\boldsymbol\beta^1_5$ and $\boldsymbol{\beta}^2_5$ the sub-vectors of active coefficients):
\begin{align*}
    (\boldsymbol\beta^1_5, \boldsymbol{\beta}^2_5) \sim N(\boldsymbol{1}_{10}, c\cdot \boldsymbol{V}_{10\times 10}),
\end{align*}
where $\boldsymbol{V}_{10\times 10}$ takes value 1 on its diagonal and 0.3 off-diagonal, while $c\in\{0.1, 0.2, 0.4, 0.5\}$ controls the degree of dependence among coefficients both across and within samples (the smaller $c$, the larger the dependence among coefficients). Finally, we take $H = N(\boldsymbol{0}, I_{n})$ (prior centering distribution), $\phi(t) = \exp(t)-1$ (i.e., $\beta = 1$), $T_s=T_0=100$ (Multinomial-Dirichlet approximation steps at both levels of the hierarchy), $M=300$ (number of MC samples from the HDP posterior), and set the SGD step size $\eta_t = 500/(100 + \sqrt{t})$. We run SGD until visually-inspected convergence, which takes around 2 seconds per run on our infrastructure (see Appendix \ref{app:computing_resources}).

We first conduct 10 simulations through which we select the optimal concentration parameter values for the HDP procedure, the separate-samples DP procedure, and the pooled-samples DP procedure, across a grid of plausible values
\begin{equation*}
    \alpha_0, \alpha_1, \alpha_2 \in \{1, 2, 5, 10, 15, 20, 25, 30, 40, 50, 60, 70, 80, 90, 100\}.
\end{equation*}
The selection is performed by fitting the models on training samples generated as above, then computing the out-of-sample risk on 10,000 additional simulated test observations. Using the optimized parameter values, we run 100 more simulations and, for each of these, compute (1) the excess\footnote{The word ``excess" refers to the risk computed at the true underlying parameter, which, in the context of simulation studies like this, we obviously have access to.} out-of-sample loss (mean squared error) and (2) the squared $L^2$ distance between the estimated and true coefficient vectors. Figure \ref{fig:lin_reg_appendix}, which shows results for $c\in\{0.1, 0.5\}$, confirms the results of Figure 2a in the main body, which shows results for $c\in\{0.2, 0.4\}$: Even in the presence of more extreme dependence structures (i.e., very low dependence or very high dependence, based on the more extreme values of $c$), the HDP method does better both on average and in terms of reduced variability, compared to the alternatives.\footnote{Notice that the Figures do not report results for the separate-samples OLS procedures. This is because, both in terms of average performance and its variability, this method performs worse than the others by one order of magnitude, and including it in the plots would distort relative comparisons among the other methods. Nevertheless, we refer the reader to our code for results on this procedure as well.}

\begin{figure}[ht]
\begin{center}
\centerline{\includegraphics[width=\textwidth]{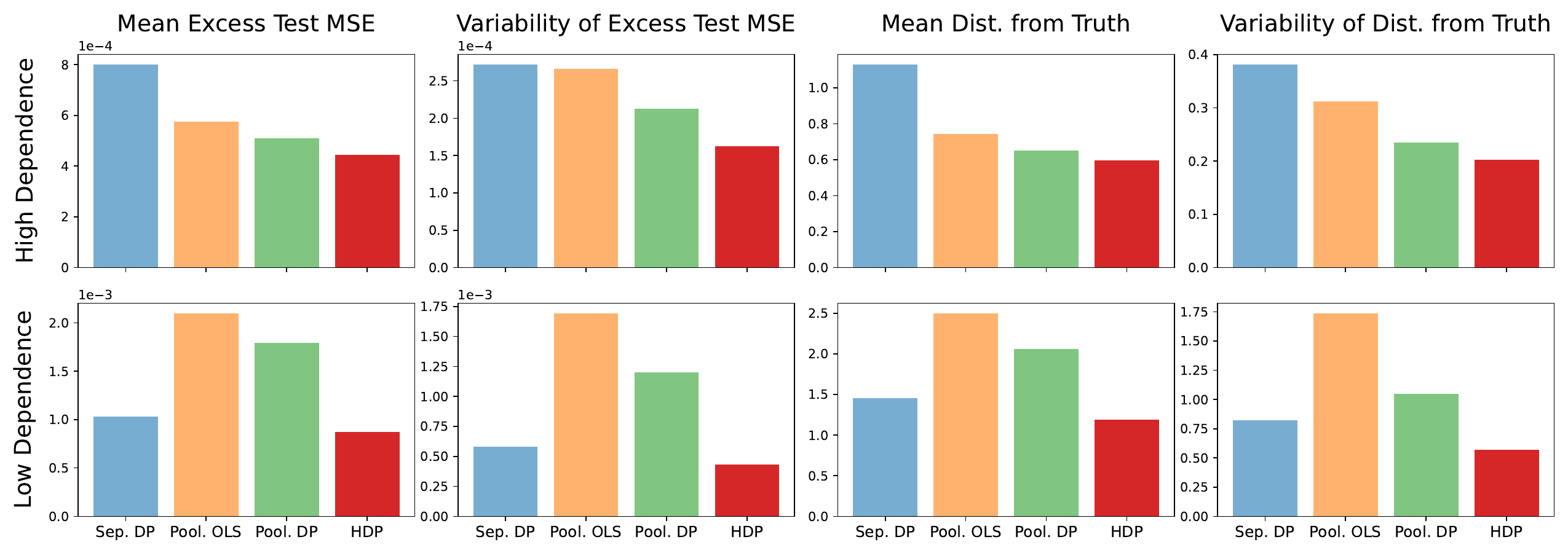}}
\caption{Out-of-sample performance and estimation accuracy of different methods in the high-dimensional linear regression experiment. The HDP method (in bright red) outperforms the others both in terms of the average and the variability of performance. Distance from truth is measured as the square $L^2$ distance between the estimated coefficients and the data-generating ones. Note: OLS estimation was performed using the Python library \texttt{scikit-learn} \citep{scikit-learn}.}
\label{fig:lin_reg_appendix}
\end{center}
\end{figure}

\paragraph{High-Dimensional Sparse Linear Median Regression Simulation Study.} 

In this experiment, we take $h$ to be the pinball loss with quantile parameter 0.5, $h(\theta,\xi)=|y-\theta^\top \boldsymbol x|$, which aims to recover the conditional median of the response variable $y$ given the feature vector $\boldsymbol{x}$. In terms of generative process and simulation setting, we keep everything as in the previous experiment on linear (mean) regression. We also keep the DP and HDP robust criterion parameters as before, and set the SGD step size $\eta_t = 500/(100 + \sqrt{t})$. We run SGD until visually-inspected convergence, which takes around 12 seconds per run on our infrastructure (see Appendix~\ref{app:computing_resources}).

Figure 2b in the main text and Figure \ref{fig:med_reg_HDP_appendix} in this Supplement report the results from 100 simulation after 10 initial ones for parameter selection (analogously to the previous experiment). Figure 2b in the main text reports results for low and high across-groups dependence regimes ($c\in\{0.2, 0.4\}$), and in Figure \ref{fig:med_reg_HDP_appendix} these regimes are taken to even larger extremes. In both cases, the qualitative conclusions highlighted for the linear regression experiment hold as well: The HDP-robust method, compared to the baseline robust DP and naive ERM estimation strategies, is effective at (i) borrowing information across groups to an optimal extent, and (ii) managing distributional uncertainty by reducing performance variability.

\begin{figure}[ht]
\begin{center}
\centerline{\includegraphics[width=\textwidth]{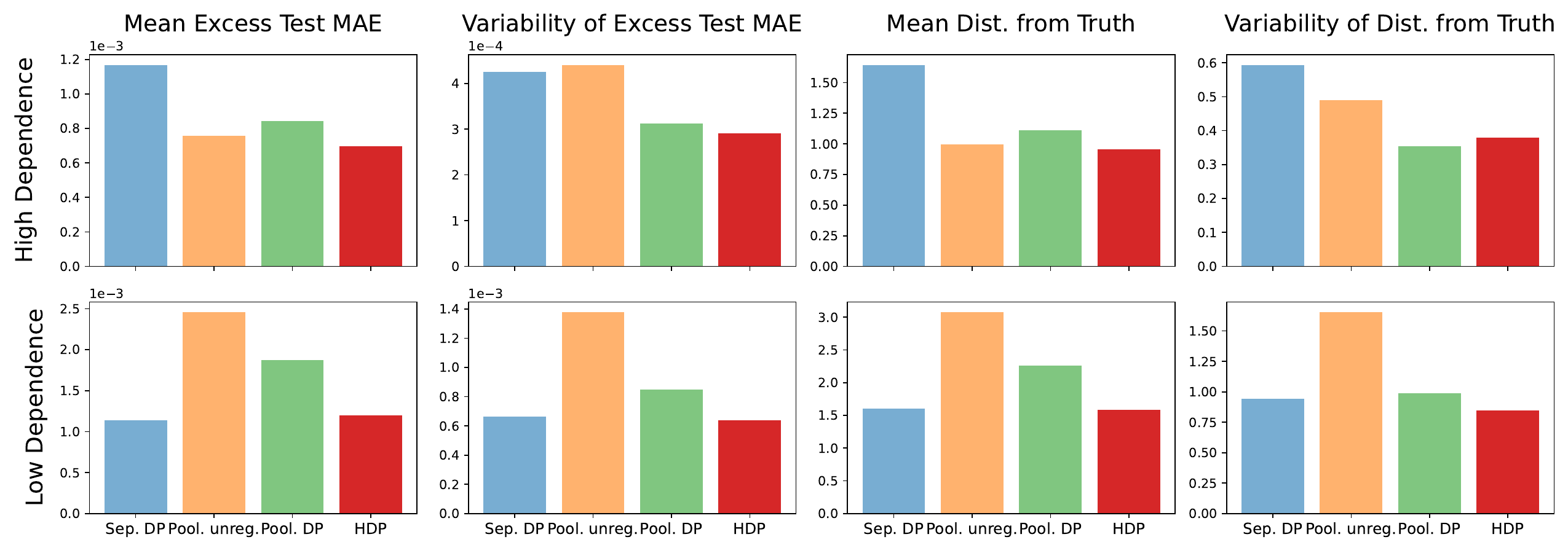}}
\caption{Out-of-sample performance and estimation accuracy of different methods in the high-dimensional linear median regression experiment for dependence parameter values $c\in\{0.1,0.5\}$. The HDP method (in bright red) outperforms the others in terms of performance variability, and does as well on average. Distance from truth is measured as the square $L^2$ distance between the estimated coefficients and the data-generating ones. Note: unregularized estimation was performed using the Python library \texttt{scikit-learn} \citep{scikit-learn}.}
\label{fig:med_reg_HDP_appendix}
\end{center}
\end{figure}

\paragraph{Experiment on Diabetes Prediction With Logistic Regression.} In this experiment, we focus on patient-level diabetes (binary) prediction across age groups based on 20 features, using the publicly available CDC Diabetes Health Indicators Dataset \citep{burrows2017incidence}\footnote{Downloaded from \url{https://www.archive.ics.uci.edu/dataset/891/cdc+diabetes+health+indicators}}. Specifically, the dataset provides baseline health information as well as an indicator of diabetes diagnosis for 253,680 US and Puerto Rico patients belonging to 13 distinct age groups (called groups 1-13 in increasing order of age). Here, we focus on predicting diabetes via logistic regression ($h(\theta, \xi)$ is specialized to the classical log-loss) for age groups 5 (of size 16,157) and 8 (of size 30,832), comparing our HDP-based method to pooled and separate ERM baselines.

To make the task relevant for our interest in distributionally uncertain problems, we focus on a very small training sample size of $50$ patients for each group. The large remaining test samples also allow us to precisely quantify out-of-sample prediction, which is our main quantity of interest. In order to select the optimal hyperparameter values
\begin{align*}
    \alpha_0, \alpha_1, \alpha_2 \in \{ & 15, 16, 17, 18, 19, 20, 22, 24, 26, 28, 30, 35, 40, 45,\\
    & 50, 55, 60, 65, 70, 75, 80, 85, 90, 95, 100, 110, 120\},
\end{align*}
we first perform 10-fold cross validation on a randomly selected sub-sample of 500 patients per group. Then, after parameter selection, we re-sample 500 patients per group and split them into 10 new folds---this helps us quantify the sample variability of the average out-of-sample performance of the different estimators computed on the 10 different folds. Out-of-sample performance is then calculated using every observation not included in the sampled folds.

After standardizing feature variables, we take $H = N(\boldsymbol{0}, I)$ (prior centering distribution), $\phi(t) = \exp(t)-1$ (i.e., $\beta = 1$), $T_s=T_0=100$ (Multinomial-Dirichlet approximation steps at both levels of the hierarchy), $M=10$ (number of MC samples from the HDP posterior), and set the SGD step size $\eta_t = 60,000/(100 + \sqrt{t})$. We run SGD until visually-inspected convergence, which takes around 2 seconds per run on our infrastructure (see Appendix \ref{app:computing_resources}).

Table \ref{tab:diabetes_log_reg} shows the results of the experiment, reporting both the mean and standard deviation of out-of-sample risk (average test loss) computed across the 10 folds. Two facts stand out: First, for both age groups, pooled ERM estimation results in better average performance than separate ERM, suggesting the presence of some degree of across-group homogeneity that is not taken into account when processing the two groups separately. Nevertheless, the borrowing-of-information mechanism implied by our HDP-based method seems in turn to outperform pooled ERM, suggesting that this degree of homogeneity is only partial and is better accounted for by our partially exchangeable modeling. Second, our method displays less variable out-of-sample risk, a likely manifestation of the benefits of distributional robustness and regularization in this data-scarce application.

\begin{table}
\centering

\begin{tabular}{cccccccc}
\toprule
\textbf{Group ($S=2$)} & \textbf{Statistic} & \textbf{Pooled ERM} & \textbf{Separate ERM} & \textbf{HDP} \\
\midrule
\multirow{2}{*}{Age Group 5} & Mean & 0.001058 & 0.005270 & 0.000791 \\
& Standard Deviation & 0.000226 & 0.001734 & 0.000049 \\
\midrule
\multirow{2}{*}{Age Group 8} & Mean & 0.001036 & 0.004285 & 0.000744 \\
& Standard Deviation & 0.000248 & 0.002312 & 0.000033 \\
\bottomrule
\end{tabular}
\caption{Summary of results from logistic regression experiment on the CDC Diabetes Health Indicators Dataset. Mean and standard deviation of out-of-sample risk are computed across 10 training data folds of size 50, after HDP hyperparameter selection performed on a separate validation sample. Note: ERM was performed using the \texttt{scikit-learn} Python library \citep{scikit-learn}, and it includes an $L^2$ regularization parameter of size $0.005$ to stabilize the optimization algorithm.}
\label{tab:diabetes_log_reg}
\end{table}

\paragraph{Experiment on Diabetes Prediction With Support Vector Machine.} This experiment is also based on the CDC Diabetes Health Indicators Dataset and it is analogous to the previous one, except for the following details: First, we perform estimation on age groups 9 (of size 33,244) and 12 (of size 15,980). Second, we perform support vector machine classification \citep{cortes1995support} by choosing $h(\theta,\xi)$ as the smooth hinge loss \citep{rennie2005loss}
\begin{align*}
    h(\theta, (y,\boldsymbol{x})) & = \left(\frac{1}{2}-z(y,\boldsymbol{x})\right) I(z\leq 0) + \frac{1}{2}(1-z)^2 I(0<z(y,\boldsymbol{x})<1),\\
    \textnormal{where } z(y,\boldsymbol{x}) &:= y\cdot \theta^\top\boldsymbol{x},
\end{align*}
for $y\in\{-1,1\}$.

After standardizing feature variables, we take $H = N(\boldsymbol{0}, I)$ (prior centering distribution), $\phi(t) = \exp(t)-1$ (i.e., $\beta = 1$), $T_s=T_0=100$ (Multinomial-Dirichlet approximation steps at both levels of the hierarchy), $M=10$ (number of MC samples from the HDP posterior), and set the SGD step size $\eta_t = 10,000/(100 + \sqrt{t})$. We run SGD until visually-inspected convergence, which takes around 4 seconds per run on our infrastructure (see Appendix \ref{app:computing_resources}).

The results shown in Table \ref{tab:diabetes_SVM} are in line with those of the logistic regression experiment: The borrowing of information mechanism of our HDP method, combined with its distributional robustness, lead to improved test performance both on average and in terms of reduced variability.

\begin{table}[ht]
\centering

\begin{tabular}{cccccccc}
\toprule
\textbf{Group ($S=2$)} & \textbf{Statistic} & \textbf{Pooled ERM} & \textbf{Separate ERM} & \textbf{HDP} \\
\midrule
\multirow{2}{*}{Age Group 9} & Mean & 0.000793 & 0.005173 & 0.000520 \\
& Standard Deviation & 0.000485 & 0.004005 & 0.000042 \\
\midrule
\multirow{2}{*}{Age Group 12} & Mean & 0.000699 & 0.003582 & 0.000523 \\
& Standard Deviation & 0.000246 & 0.004515 & 0.000014 \\
\bottomrule
\end{tabular}
\caption{Summary of results from support vector classifier experiment on the CDC Diabetes Health Indicators Dataset. Mean and standard deviation of out-of-sample risk are computed across 10 training data folds of size 50, after HDP hyperparameter selection performed on a separate validation sample. Note: ERM was performed using numerical optimization from the Python \texttt{scipy} library \citep{2020SciPy-NMeth}.}
\label{tab:diabetes_SVM}
\end{table}

\paragraph{Experiment on Wine Quality Prediction With Support Vector Regression.} In this experiment, we focus on wine quality (numeric) prediction across wine types (red and white) based on 11 features, using the publicly available Wine Quality Dataset \citep{misc_wine_quality_186}\footnote{Downloaded from \url{https://archive.ics.uci.edu/dataset/186/wine+quality}.}. Specifically, the dataset provides a few objective characteristics as well as a 0-10 quality score for 6,497 wines, among which 1,599 are red and 4,898 are white. Here, we focus on predicting quality via support vector regression, by specializing $h(\theta, \xi)$ to be the $\delta$-insensitive loss
\begin{equation*}
    h(\theta, (y,\boldsymbol{x})) = \max\{0, \vert y-\theta^\top\boldsymbol{x}\vert - \delta\}, \qquad \delta = 0.0005,
\end{equation*}
and comparing our HDP-based method to pooled and separate ERM baselines.

Like in the previous experiments, to make the task relevant for our interest in distributionally uncertain problems, we focus on a very small training sample size of $30$ wines for each group, leaving the many remaining observations as test samples. Following the same 10-fold cross validation procedure as for the previous experiments, we perform parameter selection for
\begin{align*}
    \alpha_0, \alpha_1, \alpha_2 \in \{ & 15, 16, 17, 18, 19, 20, 22, 24, 26, 28, 30, 35, 40, 45,\\
    & 50, 55, 60, 65, 70, 75, 80, 85, 90, 95, 100, 110, 120\}.
\end{align*}

After standardizing feature variables, we take $H = N(\boldsymbol{0}, I)$ (prior centering distribution), $\phi(t) = \exp(t)-1$ (i.e., $\beta = 1$), $T_s=T_0=100$ (Multinomial-Dirichlet approximation steps at both levels of the hierarchy), $M=200$ (number of MC samples from the HDP posterior), and set the SGD step size $\eta_t = 3,000/(100 + \sqrt{t})$. We run SGD until visually-inspected convergence, which takes around 1 second per run on our infrastructure (see Appendix \ref{app:computing_resources}).

As in the previous experiments, Table \ref{tab:winequality_SVR} shows that the borrowing of information mechanism of our HDP method, combined with its distributional robustness, leads to improved test performance both on average and in terms of reduced variability.

\begin{table}[ht]
\centering

\begin{tabular}{cccccccc}
\toprule
\textbf{Group ($S=2$)} & \textbf{Statistic} & \textbf{Pooled ERM} & \textbf{Separate ERM} & \textbf{HDP} \\
\midrule
\multirow{2}{*}{Red Wines} & Mean & 0.000687 & 0.000732 & 0.000679 \\
& Standard Deviation & 0.000034 & 0.000093 & 0.000018 \\
\midrule
\multirow{2}{*}{White Wines} & Mean & 0.000743 & 0.000786 & 0.000717 \\
& Standard Deviation & 0.000035 & 0.000040 & 0.000023 \\
\bottomrule
\end{tabular}
\caption{Summary of results from support vector regression experiment on the Wine Quality Dataset. Mean and standard deviation of out-of-sample risk are computed across 10 training data folds of size 30, after HDP hyperparameter selection performed on a separate validation sample. Note: ERM was performed using the \texttt{scikit-learn} Python library \citep{scikit-learn}.}
\label{tab:winequality_SVR}
\end{table}

\section{Experiment on DORO Criterion}\label{app:experiments_DORO}

\paragraph{Setting.} In this experiment, we test the performance of our DORO criterion in the same high-dimensional sparse linear regression task as in Appendix~\ref{app:DP_lin_reg_experiment}, although the samples are contaminated with a small fraction of outliers. Note that the training dynamics reported in Figure~\ref{fig:DRO_DORO_losses} are taken from the first out of 200 independent simulations performed in this experiment (see below for further details).

\paragraph{Data-Generating Process.} The data for the experiment are generated iid across simulations (200) and observations ($n=100$ per simulation) as follows. For each observation $i=1,\dots,100$, the $d$-dimensional ($d=90$) covariate vector follows a multivariate normal distribution with mean 0 and such that (i) each covariate has unitary variance, and (ii) any pair of distinct covariates has covariance 0.3:
\begin{equation*}
    x_i = \begin{bmatrix}
        x_{i1} \\
        \vdots \\
        x_{id}
    \end{bmatrix} \sim \mathcal N(0, \Sigma), \quad \Sigma = \begin{bmatrix}
        1 & 0.3 & \cdots & 0.3 \\
        0.3 & 1 & \cdots & 0.3\\
        \vdots & \vdots & \ddots & \vdots\\
        0.3 & 0.3 & \cdots & 1
    \end{bmatrix} \in \mathbb R^{d\times d}.
\end{equation*}
Then, for the first 95 observations, the response has conditional distribution $y_i\mid x_i \sim \mathcal N(a^\top x_i, \sigma^2)$, with
$$a = (1, 1, 1, 1, 1, 0, \cdots, 0)^\top\in\mathbb R^d$$
and $\sigma = 0.5$. Instead, we generate the last 5 observations from an outlier conditional distribution $y_i\mid x_i \sim \mathcal N(b^\top x_i, \sigma^2)$, where
$$b = -10 \times (1, 1, 1, 1, 1, 1, \cdots, 1)^\top\in\mathbb R^d.$$
Together with these 100 training samples, at each simulation we generate 5000 test samples from the uncontaminated distribution, on which we compute out-of-sample RMSE for the DORO, DRO, ambiguity-neutral (Ridge), and OLS procedures.

\paragraph{Robust Criterion Parameters.} For each simulated sample, we run our DRO and DORO procedures setting the following parameter values: $\phi(t)=\beta\exp(t/\beta)-\beta$, $\beta \in\{1, \infty\}$, $\alpha\in\{1, 2, 5, 10\}$, and $p_0 = \mathcal N(0,I)$, where the $\beta = \infty$ setting corresponds to Ridge regression with regularization parameter $\alpha$. Finally, we run 50 Monte Carlo simulations to approximate the criterion, and truncate the Multinomial-Dirichlet approximation at $T=100$.

\paragraph{Stochastic Gradient Descent Parameters} We initialize the algorithm at $\theta = (0,\dots,0)$ and set the step size at $\eta_t = 200/(100 + \sqrt{t})$. The contamination parameter $\varepsilon$ for the DORO training procedure is set at $0.1 > 0.05 = \varepsilon_\star$. The number of passes over data is set after visual inspection of convergence of the criterion value. The run time per SGD run is less than 1 second on our infrastructure (see Appendix \ref{app:computing_resources}).

\paragraph{Results.} Figure~\ref{fig:DORO_lin_reg_histograms} confirms te expectation that DORO is able to restore the favorable performance of our distributionally robust procedure by filtering out the outliers that make it unstable. In particular, both in terms of out of sample prediction and of parameter estimation accuracy, DORO improves upon plain and regularized ERM and especially DRO (which shows worse performance than ERM in this setting with outliers).

\begin{figure}
    \centering
    \includegraphics[width=0.8\linewidth]{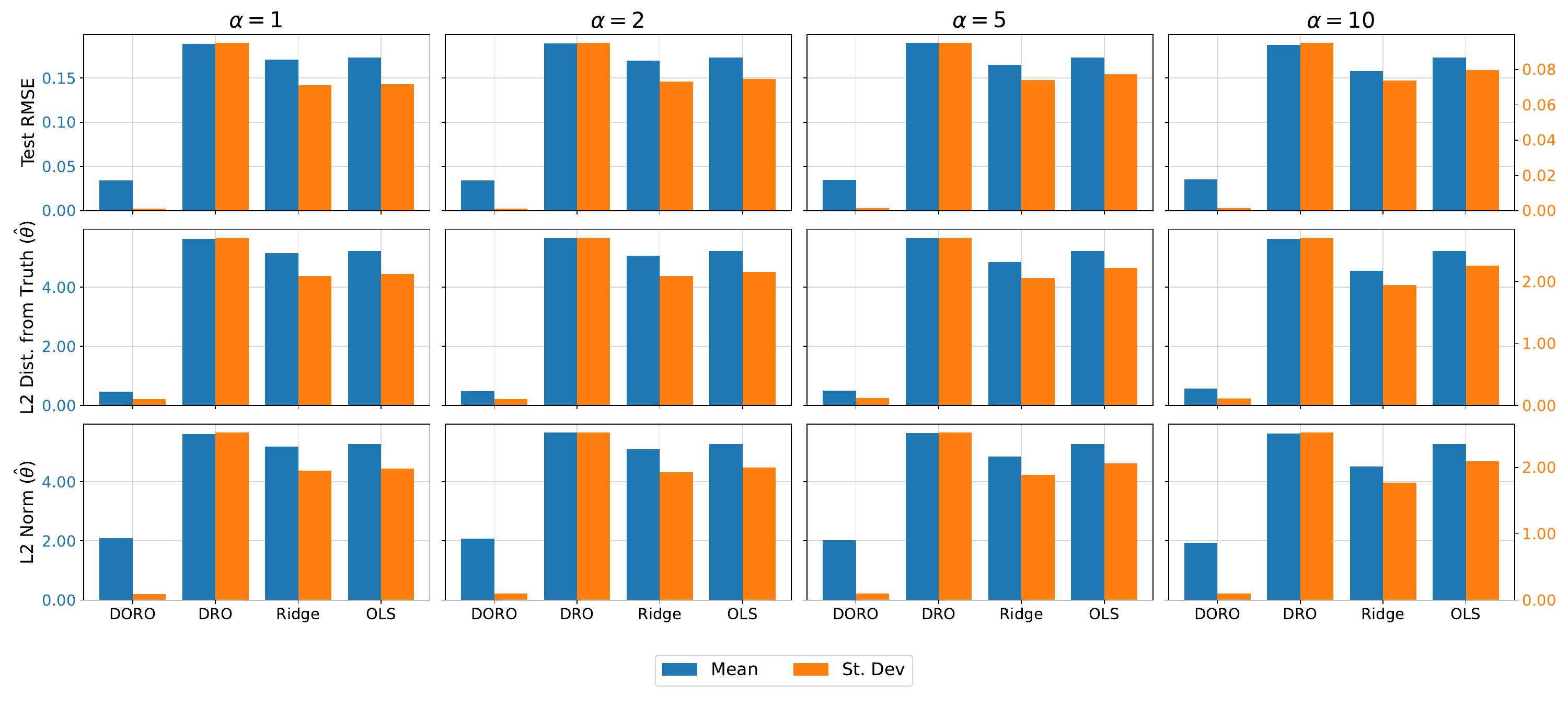}
    \caption{Simulation results for high-dimensional sparse linear regression experiment with 5\% of contaminated sample. Bars report the mean and standard deviation (across 200 sample simulations) of the test RMSE, $L^2$ distance of estimated coefficient vector $\hat\theta$ from the data-generating one, and the $L^2$ norm of $\hat\theta$. Results are shown for the DORO, DRO, ambiguity-neutral, and OLS procedures. Note: The left (blue) axis refers to mean values, the right (orange) axis to standard deviation values.}
    \label{fig:DORO_lin_reg_histograms}
\end{figure}

\section{Computational Infrastructure}\label{app:computing_resources}

All experiments were performed on a desktop with 12th Gen Intel(R) Core(TM) i9-12900H, 2500 Mhz, 14 Core(s), 20 Logical Processor(s) and 32.0 GB RAM.

\vskip 0.2in
\bibliography{sample}

\begin{thebibliography}{63}
\providecommand{\natexlab}[1]{#1}
\providecommand{\url}[1]{\texttt{#1}}
\expandafter\ifx\csname urlstyle\endcsname\relax
  \providecommand{\doi}[1]{doi: #1}\else
  \providecommand{\doi}{doi: \begingroup \urlstyle{rm}\Url}\fi

\bibitem[Anscombe and Aumann(1963)]{anscombe1963definition}
Francis~J Anscombe and Robert~J Aumann.
\newblock A definition of subjective probability.
\newblock \emph{Annals of mathematical statistics}, 34\penalty0 (1):\penalty0 199--205, 1963.

\bibitem[Bariletto and Ho(2024)]{bariletto2024bayesian}
Nicola Bariletto and Nhat Ho.
\newblock {Bayesian Nonparametrics Meets Data-Driven Distributionally Robust Optimization}.
\newblock \emph{Advances in Neural Information Processing Systems}, 37, 2024.

\bibitem[Ben-Tal et~al.(2013)Ben-Tal, den Hertog, Waegenaere, Melenberg, and Rennen]{ben-tal2013robust}
A.~Ben-Tal, D.~den Hertog, A.~De Waegenaere, B.~Melenberg, and G.~Rennen.
\newblock Robust solutions of optimization problems affected by uncertain probabilities.
\newblock \emph{Management Science}, 59\penalty0 (2):\penalty0 341--357, 2013.

\bibitem[Ben-Tal et~al.(2010)Ben-Tal, Bertsimas, and Brown]{ben2010soft}
Aharon Ben-Tal, Dimitris Bertsimas, and David~B Brown.
\newblock A soft robust model for optimization under ambiguity.
\newblock \emph{Operations research}, 58\penalty0 (4-part-2):\penalty0 1220--1234, 2010.

\bibitem[Blackwell and MacQueen(1973)]{blackwell1973ferguson}
David Blackwell and James~B MacQueen.
\newblock Ferguson distributions via p{\'o}lya urn schemes.
\newblock \emph{The annals of statistics}, 1\penalty0 (2):\penalty0 353--355, 1973.

\bibitem[Burrows(2017)]{burrows2017incidence}
Nilka~Rios Burrows.
\newblock {Incidence of end-stage renal disease attributed to diabetes among persons with diagnosed diabetes—United States and Puerto Rico, 2000--2014}.
\newblock \emph{MMWR. Morbidity and mortality weekly report}, 66, 2017.

\bibitem[Cai et~al.(2023)Cai, Li, and Mao]{cai2023dro}
Jun Cai, Jonathan Yu-Meng Li, and Tiantian Mao.
\newblock Distributionally robust optimization under distorted expectations.
\newblock \emph{Operations Research}, 2023.
\newblock \doi{10.1287/opre.2020.0685}.

\bibitem[Camerlenghi et~al.(2019{\natexlab{a}})Camerlenghi, Dunson, Lijoi, Pr{\"u}nster, and Rodr{\'\i}guez]{camerlenghi2019latent}
Federico Camerlenghi, David~B Dunson, Antonio Lijoi, Igor Pr{\"u}nster, and Abel Rodr{\'\i}guez.
\newblock Latent nested nonparametric priors (with discussion).
\newblock \emph{Bayesian Analysis}, 14\penalty0 (4):\penalty0 1303--1356, 2019{\natexlab{a}}.

\bibitem[Camerlenghi et~al.(2019{\natexlab{b}})Camerlenghi, Lijoi, Orbanz, and Pr{\"u}nster]{camerlenghi2019distribution}
Federico Camerlenghi, Antonio Lijoi, Peter Orbanz, and Igor Pr{\"u}nster.
\newblock Distribution theory for hierarchical processes.
\newblock \emph{The Annals of Statistics}, 47\penalty0 (1):\penalty0 67--92, 2019{\natexlab{b}}.

\bibitem[Cerreia-Vioglio et~al.(2011)Cerreia-Vioglio, Maccheroni, Marinacci, and Montrucchio]{cerreia2011uncertainty}
Simone Cerreia-Vioglio, Fabio Maccheroni, Massimo Marinacci, and Luigi Montrucchio.
\newblock Uncertainty averse preferences.
\newblock \emph{Journal of Economic Theory}, 146\penalty0 (4):\penalty0 1275--1330, 2011.

\bibitem[Cerreia-Vioglio et~al.(2013)Cerreia-Vioglio, Maccheroni, Marinacci, and Montrucchio]{cerreia2013ambiguity}
Simone Cerreia-Vioglio, Fabio Maccheroni, Massimo Marinacci, and Luigi Montrucchio.
\newblock Ambiguity and robust statistics.
\newblock \emph{Journal of Economic Theory}, 148\penalty0 (3):\penalty0 974--1049, 2013.

\bibitem[Cortes and Vapnik(1995)]{cortes1995support}
Corinna Cortes and Vladimir Vapnik.
\newblock Support-vector networks.
\newblock \emph{Machine learning}, 20:\penalty0 273--297, 1995.

\bibitem[Cortez et~al.(2009)Cortez, Cerdeira, Almeida, Matos, , and Reis]{misc_wine_quality_186}
Paulo Cortez, A.~Cerdeira, F.~Almeida, T.~Matos, , and J.~Reis.
\newblock {Wine Quality}.
\newblock UCI Machine Learning Repository, 2009.
\newblock {DOI}: https://doi.org/10.24432/C56S3T.

\bibitem[Duchi and Namkoong(2021)]{duchi2021learning}
J.~C. Duchi and H.~Namkoong.
\newblock Learning models with uniform performance via distributionally robust optimization.
\newblock \emph{The Annals of Statistics}, 49\penalty0 (3):\penalty0 1378--1406, 2021.

\bibitem[Ellsberg(1961)]{ellsberg1961risk}
Daniel Ellsberg.
\newblock Risk, ambiguity, and the {S}avage axioms.
\newblock \emph{The Quarterly Journal of Economics}, 75\penalty0 (4):\penalty0 643--669, 1961.

\bibitem[Ferguson(1973)]{ferguson1973bayesian}
Thomas~S Ferguson.
\newblock A {B}ayesian analysis of some nonparametric problems.
\newblock \emph{The annals of statistics}, pages 209--230, 1973.

\bibitem[Forsyth(1990)]{misc_liver_disorders_60}
Richard~S. Forsyth.
\newblock {Liver Disorders}.
\newblock UCI Machine Learning Repository, 1990.
\newblock {DOI}: https://doi.org/10.24432/C54G67.

\bibitem[Gallego and Moon(1993)]{gallego1993newsboy}
G.~Gallego and I.~Moon.
\newblock The distribution free newsboy problem: Review and extensions.
\newblock \emph{The Journal of the Operational Research Society}, 44\penalty0 (8):\penalty0 825--834, 1993.

\bibitem[Garrigos and Gower(2023)]{garrigos2023handbook}
Guillaume Garrigos and Robert~M Gower.
\newblock Handbook of convergence theorems for (stochastic) gradient methods.
\newblock \emph{arXiv preprint arXiv:2301.11235}, 2023.

\bibitem[Ghaoui et~al.(2003)Ghaoui, Oks, and Oustry]{elghaoui2003worstcase}
L.~El Ghaoui, M.~Oks, and F.~Oustry.
\newblock Worst-case value-at-risk and robust portfolio optimization: A conic programming approach.
\newblock \emph{Operations Research}, 51\penalty0 (4):\penalty0 543--556, 2003.

\bibitem[Ghirardato et~al.(2004)Ghirardato, Maccheroni, and Marinacci]{ghirardato2004differentiating}
Paolo Ghirardato, Fabio Maccheroni, and Massimo Marinacci.
\newblock Differentiating ambiguity and ambiguity attitude.
\newblock \emph{Journal of Economic Theory}, 118\penalty0 (2):\penalty0 133--173, 2004.

\bibitem[Ghosal and Van~der Vaart(2017)]{ghosal2017fundamentals}
Subhashis Ghosal and Aad Van~der Vaart.
\newblock \emph{Fundamentals of nonparametric {B}ayesian inference}, volume~44.
\newblock Cambridge University Press, 2017.

\bibitem[Gilboa and Schmeidler(1989)]{gilboa1989maxmin}
Itzhak Gilboa and David Schmeidler.
\newblock Maxmin expected utility with non-unique prior.
\newblock \emph{Journal of mathematical economics}, 18\penalty0 (2):\penalty0 141--153, 1989.

\bibitem[Hansen and Sargent(2001)]{hansen2001robust}
Lars~Peter Hansen and Thomas~J Sargent.
\newblock Robust control and model uncertainty.
\newblock \emph{American Economic Review}, 91\penalty0 (2):\penalty0 60--66, 2001.

\bibitem[Hoerl and Kennard(1970)]{hoerl1970ridge}
Arthur~E Hoerl and Robert~W Kennard.
\newblock Ridge regression: Biased estimation for nonorthogonal problems.
\newblock \emph{Technometrics}, 12\penalty0 (1):\penalty0 55--67, 1970.

\bibitem[Jiang and Guan(2018)]{jiang2018risk}
R.~Jiang and Y.~Guan.
\newblock Risk-averse two-stage stochastic program with distributional ambiguity.
\newblock \emph{Operations Research}, 66\penalty0 (5):\penalty0 1390--1405, 2018.

\bibitem[Klibanoff et~al.(2005)Klibanoff, Marinacci, and Mukerji]{klibanoff2005smooth}
Peter Klibanoff, Massimo Marinacci, and Sujoy Mukerji.
\newblock A smooth model of decision making under ambiguity.
\newblock \emph{Econometrica}, 73\penalty0 (6):\penalty0 1849--1892, 2005.

\bibitem[Koenker(2005)]{koenker2005quantile}
Roger Koenker.
\newblock \emph{Quantile regression}, volume~38.
\newblock Cambridge University Press, 2005.

\bibitem[Kuhn et~al.(2024)Kuhn, Shafiee, and Wiesemann]{kuhn2024distributionallyrobustoptimization}
Daniel Kuhn, Soroosh Shafiee, and Wolfram Wiesemann.
\newblock Distributionally robust optimization, 2024.
\newblock URL \url{https://arxiv.org/abs/2411.02549}.

\bibitem[Lafferty et~al.(2010)Lafferty, Liu, and Wasserman]{wasserman2010concentration}
John Lafferty, Han Liu, and Larry Wasserman.
\newblock Concentration of measure.
\newblock Technical report, Carnegie Mellon University, 2010.
\newblock URL \url{https://www.stat.cmu.edu/~larry/=sml/Concentration.pdf}.

\bibitem[Lasserre and Weisser(2021)]{lasserre2021distributionally}
J.~B. Lasserre and T.~Weisser.
\newblock Distributionally robust polynomial chance-constraints under mixture ambiguity sets.
\newblock \emph{Mathematical Programming}, 185\penalty0 (1--2):\penalty0 409--453, 2021.

\bibitem[Li(2018)]{li2018closedform}
J.~Y.-M. Li.
\newblock Closed-form solutions for worst-case law invariant risk measures with application to robust portfolio optimization.
\newblock \emph{Operations Research}, 66\penalty0 (6):\penalty0 1533--1541, 2018.

\bibitem[Lijoi et~al.(2014{\natexlab{a}})Lijoi, Nipoti, and Pr{\"u}nster]{lijoi2014bayesian}
Antonio Lijoi, Bernardo Nipoti, and Igor Pr{\"u}nster.
\newblock Bayesian inference with dependent normalized completely random measures.
\newblock \emph{Bernoulli}, 20\penalty0 (3):\penalty0 1260--1291, 2014{\natexlab{a}}.

\bibitem[Lijoi et~al.(2014{\natexlab{b}})Lijoi, Nipoti, and Pr{\"u}nster]{lijoi2014dependent}
Antonio Lijoi, Bernardo Nipoti, and Igor Pr{\"u}nster.
\newblock Dependent mixture models: clustering and borrowing information.
\newblock \emph{Computational Statistics \& Data Analysis}, 71:\penalty0 417--433, 2014{\natexlab{b}}.

\bibitem[Lyddon et~al.(2018)Lyddon, Walker, and Holmes]{lyddon2018nonparametric}
Simon Lyddon, Stephen Walker, and Chris~C Holmes.
\newblock Nonparametric learning from {B}ayesian models with randomized objective functions.
\newblock \emph{Advances in Neural Information Processing Systems}, 31, 2018.

\bibitem[Maccheroni et~al.(2006)Maccheroni, Marinacci, and Rustichini]{maccheroni2006ambiguity}
Fabio Maccheroni, Massimo Marinacci, and Aldo Rustichini.
\newblock Ambiguity aversion, robustness, and the variational representation of preferences.
\newblock \emph{Econometrica}, 74\penalty0 (6):\penalty0 1447--1498, 2006.

\bibitem[MacEachern(2000)]{maceachern2000dependent}
Steven~N MacEachern.
\newblock Dependent {D}irichlet processes.
\newblock Department of Statistics, The Ohio State University, 2000.

\bibitem[Mas-Colell et~al.(1995)Mas-Colell, Whinston, and Green]{mas-colell1996microeconomictheory}
Andreu Mas-Colell, Michael~D. Whinston, and Jerry~R. Green.
\newblock \emph{{Microeconomic Theory}}.
\newblock Oxford University Press, 1995.

\bibitem[Mohajerin~Esfahani and Kuhn(2018)]{mohajerin2018data}
Peyman Mohajerin~Esfahani and Daniel Kuhn.
\newblock Data-driven distributionally robust optimization using the {W}asserstein metric: Performance guarantees and tractable reformulations.
\newblock \emph{Mathematical Programming}, 171\penalty0 (1-2):\penalty0 115--166, 2018.

\bibitem[M{\"u}ller et~al.(2004)M{\"u}ller, Quintana, and Rosner]{muller2004method}
Peter M{\"u}ller, Fernando Quintana, and Gary Rosner.
\newblock A method for combining inference across related nonparametric {B}ayesian models.
\newblock \emph{Journal of the Royal Statistical Society: Series B (Statistical Methodology)}, 66\penalty0 (3):\penalty0 735--749, 2004.

\bibitem[Natarajan et~al.(2010)Natarajan, Sim, and Uichanco]{natarajan2010robust}
K.~Natarajan, M.~Sim, and J.~Uichanco.
\newblock Tractable robust expected utility and risk models for portfolio optimization.
\newblock \emph{Mathematical Finance}, 20\penalty0 (4):\penalty0 695--731, 2010.

\bibitem[Pedregosa et~al.(2011)Pedregosa, Varoquaux, Gramfort, Michel, Thirion, Grisel, Blondel, Prettenhofer, Weiss, Dubourg, Vanderplas, Passos, Cournapeau, Brucher, Perrot, and Duchesnay]{scikit-learn}
F.~Pedregosa, G.~Varoquaux, A.~Gramfort, V.~Michel, B.~Thirion, O.~Grisel, M.~Blondel, P.~Prettenhofer, R.~Weiss, V.~Dubourg, J.~Vanderplas, A.~Passos, D.~Cournapeau, M.~Brucher, M.~Perrot, and E.~Duchesnay.
\newblock Scikit-learn: Machine learning in {P}ython.
\newblock \emph{Journal of Machine Learning Research}, 12:\penalty0 2825--2830, 2011.

\bibitem[Pflug and Wozabal(2007)]{pflug2007ambiguity}
G.~C. Pflug and D.~Wozabal.
\newblock Ambiguity in portfolio selection.
\newblock \emph{Quantitative Finance}, 7\penalty0 (4):\penalty0 435--442, 2007.

\bibitem[Pflug et~al.(2012)Pflug, Pichler, and Wozabal]{pflug2012investment}
G.~C. Pflug, A.~Pichler, and D.~Wozabal.
\newblock The 1/\textit{N} investment strategy is optimal under high model ambiguity.
\newblock \emph{Journal of Banking \& Finance}, 36\penalty0 (2):\penalty0 410--417, 2012.

\bibitem[Rahimian and Mehrotra(2022)]{rahimian2022frameworks}
Hamed Rahimian and Sanjay Mehrotra.
\newblock Frameworks and results in distributionally robust optimization.
\newblock \emph{Open Journal of Mathematical Optimization}, 3:\penalty0 1--85, 2022.

\bibitem[Rennie and Srebro(2005)]{rennie2005loss}
Jason~DM Rennie and Nathan Srebro.
\newblock Loss functions for preference levels: {R}egression with discrete ordered labels.
\newblock In \emph{Proceedings of the IJCAI multidisciplinary workshop on advances in preference handling}, volume~1. AAAI Press, Menlo Park, CA, 2005.

\bibitem[Rodriguez et~al.(2008)Rodriguez, Dunson, and Gelfand]{rodriguez2008nested}
Abel Rodriguez, David~B Dunson, and Alan~E Gelfand.
\newblock The nested {D}irichlet process.
\newblock \emph{Journal of the American Statistical Association}, 103\penalty0 (483):\penalty0 1131--1154, 2008.

\bibitem[Savage(1954)]{savage1954foundations}
Leonard~J. Savage.
\newblock \emph{{The Foundations of Statistics}}.
\newblock John Wiley and Sons, New York, 1954.

\bibitem[Scarf(1958)]{scarf1958minmax}
H.~Scarf.
\newblock A min-max solution to an inventory problem.
\newblock In K.~Arrow, S.~Karlin, and H.~Scarf, editors, \emph{Studies in Mathematical Theory of Inventory and Production}, pages 201--209. Stanford University Press, 1958.

\bibitem[Sethuraman(1994)]{sethuraman1994constructive}
Jayaram Sethuraman.
\newblock A constructive definition of {D}irichlet priors.
\newblock \emph{Statistica sinica}, pages 639--650, 1994.

\bibitem[Staib and Jegelka(2019)]{staib2019kernel}
M.~Staib and S.~Jegelka.
\newblock Distributionally robust optimization and generalization in kernel methods.
\newblock In \emph{Advances in Neural Information Processing Systems}, pages 9134--9144, 2019.

\bibitem[Teh et~al.(2004)Teh, Jordan, Beal, and Blei]{teh2004sharing}
Yee Teh, Michael Jordan, Matthew Beal, and David Blei.
\newblock Sharing clusters among related groups: Hierarchical {D}irichlet processes.
\newblock \emph{Advances in neural information processing systems}, 17, 2004.

\bibitem[Teh et~al.(2006)Teh, Jordan, Beal, and Blei]{teh2006hierarchical}
Yee~Whye Teh, Michael~I Jordan, Matthew~J Beal, and David~M Blei.
\newblock Hierarchical {D}irichlet processes.
\newblock \emph{Journal of the American Statistical Association}, 101\penalty0 (476):\penalty0 1566--1581, 2006.

\bibitem[Vershynin(2018)]{vershynin2018high}
Roman Vershynin.
\newblock \emph{{High-Dimensional Probability: An Introduction with Applications in Data Science}}, volume~47.
\newblock Cambridge University Press, 2018.

\bibitem[Virtanen et~al.(2020)Virtanen, Gommers, Oliphant, Haberland, Reddy, Cournapeau, Burovski, Peterson, Weckesser, Bright, {van der Walt}, Brett, Wilson, Millman, Mayorov, Nelson, Jones, Kern, Larson, Carey, Polat, Feng, Moore, {VanderPlas}, Laxalde, Perktold, Cimrman, Henriksen, Quintero, Harris, Archibald, Ribeiro, Pedregosa, {van Mulbregt}, and {SciPy 1.0 Contributors}]{2020SciPy-NMeth}
Pauli Virtanen, Ralf Gommers, Travis~E. Oliphant, Matt Haberland, Tyler Reddy, David Cournapeau, Evgeni Burovski, Pearu Peterson, Warren Weckesser, Jonathan Bright, St{\'e}fan~J. {van der Walt}, Matthew Brett, Joshua Wilson, K.~Jarrod Millman, Nikolay Mayorov, Andrew R.~J. Nelson, Eric Jones, Robert Kern, Eric Larson, C~J Carey, {\.I}lhan Polat, Yu~Feng, Eric~W. Moore, Jake {VanderPlas}, Denis Laxalde, Josef Perktold, Robert Cimrman, Ian Henriksen, E.~A. Quintero, Charles~R. Harris, Anne~M. Archibald, Ant{\^o}nio~H. Ribeiro, Fabian Pedregosa, Paul {van Mulbregt}, and {SciPy 1.0 Contributors}.
\newblock {{SciPy} 1.0: Fundamental Algorithms for Scientific Computing in Python}.
\newblock \emph{Nature Methods}, 17:\penalty0 261--272, 2020.

\bibitem[von Neumann and Morgenstern(1947)]{neumann1947theory}
John von Neumann and Oskar Morgenstern.
\newblock \emph{{Theory of Games and Economic Behavior}}.
\newblock Princeton University Press, Princeton, 1947.

\bibitem[Wainwright(2019)]{wainwright2019high}
Martin~J Wainwright.
\newblock \emph{High-dimensional statistics: A non-asymptotic viewpoint}, volume~48.
\newblock Cambridge university press, 2019.

\bibitem[Wang and Wang(2022)]{wang2022distributional}
Shixiong Wang and Haowei Wang.
\newblock Distributional robustness bounds generalization errors.
\newblock \emph{arXiv preprint arXiv:2212.09962}, 2022.

\bibitem[Wang et~al.(2016)Wang, Glynn, and Ye]{wang2016likelihood}
Z.~Wang, P.~W. Glynn, and Y.~Ye.
\newblock Likelihood robust optimization for data-driven problems.
\newblock \emph{Computational Management Science}, 13:\penalty0 241--261, 2016.

\bibitem[Watson and Holmes(2016)]{holmes2016robuststatistics}
James Watson and Chris Holmes.
\newblock {Approximate Models and Robust Decisions}.
\newblock \emph{Statistical Science}, 31\penalty0 (4):\penalty0 465 -- 489, 2016.
\newblock \doi{10.1214/16-STS592}.
\newblock URL \url{https://doi.org/10.1214/16-STS592}.

\bibitem[Yue et~al.(2006)Yue, Chen, and Wang]{yue2006newsvendor}
J.~Yue, B.~Chen, and M.-C. Wang.
\newblock Expected value of distribution information for the newsvendor problem.
\newblock \emph{Operations Research}, 54\penalty0 (6):\penalty0 1128--1136, 2006.

\bibitem[Zhai et~al.(2021)Zhai, Dan, Kolter, and Ravikumar]{zhai2021doro}
Runtian Zhai, Chen Dan, Zico Kolter, and Pradeep Ravikumar.
\newblock Doro: Distributional and outlier robust optimization.
\newblock In \emph{International Conference on Machine Learning}, pages 12345--12355. PMLR, 2021.

\bibitem[Zhang et~al.(2024)Zhang, Yang, and Gao]{zhang2024short}
Luhao Zhang, Jincheng Yang, and Rui Gao.
\newblock A short and general duality proof for {W}asserstein distributionally robust optimization.
\newblock \emph{Operations Research}, 2024.

\end{thebibliography}

\end{document}